\documentclass{article}

\PassOptionsToPackage{numbers,compress}{natbib}
\usepackage[main,final]{neurips_2026}

\usepackage{microtype}
\usepackage{graphicx}
\usepackage{subcaption}
\usepackage{booktabs} 

\usepackage{hyperref}
\usepackage{comment}
\usepackage{enumitem}

\usepackage[linesnumbered,ruled,vlined]{algorithm2e}

\usepackage{amsmath}
\usepackage{amssymb}
\usepackage{mathtools}
\usepackage{amsthm}

\usepackage{multirow}
\usepackage{float}

\usepackage[capitalize,noabbrev]{cleveref}
\usepackage{minitoc}
\doparttoc
\AtBeginDocument{%
  \setcounter{tocdepth}{2}%
  \setcounter{parttocdepth}{2}%
}
\newcommand{\appendixpart}{%
  \refstepcounter{part}%
  \adjustptc%
  \addcontentsline{toc}{part}{\protect\numberline{\thepart}Appendix}%
}

\theoremstyle{plain}
\newtheorem{theorem}{Theorem}[section]

\newtheorem{lemma}[theorem]{Lemma}

\theoremstyle{definition}

\theoremstyle{remark}

\usepackage[textsize=tiny]{todonotes}

\usepackage{xcolor}

\title{Federated Concept-Based Models: \\ Interpretable models with distributed supervision}

\author{%
  \textbf{Dario Fenoglio\textsuperscript{1,2}}\thanks{Equal contribution. Corresponding authors: \texttt{dario.fenoglio@usi.ch}, \texttt{arianna.casanova@uni.li}}
  \quad
  \textbf{Arianna Casanova\textsuperscript{3}}\footnotemark[1]
  \quad
  \textbf{Francesco De Santis\textsuperscript{4}} \\
  \textbf{Gabriele Dominici\textsuperscript{1}}
  \quad
  \textbf{Johannes Schneider\textsuperscript{3}}
  \quad
  \textbf{Pietro Barbiero\textsuperscript{5}} \\
  \textbf{Giovanni De Felice\textsuperscript{1}}
  \quad
  \textbf{Marc Langheinrich\textsuperscript{1}}
  \quad
  \textbf{Martin Gjoreski\textsuperscript{1}} \\[3pt]
  \textsuperscript{1}Universit\`a della Svizzera italiana, Lugano, Switzerland \\
  \textsuperscript{2}Massachusetts Institute of Technology, Cambridge, MA, USA \\
  \textsuperscript{3}University of Liechtenstein, Vaduz, Liechtenstein \\
  \textsuperscript{4}Politecnico di Torino, Turin, Italy \\
  \textsuperscript{5}IBM Research, Zurich, Switzerland
}

\begin{document}
\maketitle
\faketableofcontents






\begin{abstract}
Concept-based Models (CMs) enhance interpretability in deep learning by grounding predictions in human-understandable concepts.
However, concept annotations are costly and rarely available at scale within a single data source.
Federated Learning (FL) could alleviate this limitation by enabling cross-institutional training over concept annotations distributed across multiple data owners. Yet, FL lacks interpretable modeling paradigms. Integrating CMs with FL is non-trivial: although FL supports heterogeneous and non-stationary client participation, it typically assumes a fixed shared architecture, whereas CMs may require architectural adaptation as the available concept set evolves.
We propose \emph{Federated Concept-based Models} (F-CMs), a new methodology for deploying CMs in evolving FL settings. F-CMs aggregate concept-level information across institutions and efficiently adapt the model architecture to changes in concept supervision while preserving privacy. 
Empirically, F-CMs maintain accuracy and intervention effectiveness comparable to training settings with full concept supervision,
while outperforming on average non-adaptive federated baselines. Notably,
F-CMs enable interpretable inference on concepts unavailable to a given institution, a key novelty over existing approaches.
\end{abstract}

\section{Introduction}
In many real-world applications, the deployment of machine learning systems is subject to constraints beyond predictive accuracy: practitioners need models whose predictions are interpretable and auditable, providing tools for error diagnosis, fairness assessment, and compliance with legal standards
~\cite{kaminski2021, li2022, dhar2023}. Among existing approaches, \emph{Concept-based Models} (CMs) offer an established framework in which predictions are formulated in terms of high-level, human-interpretable variables, referred to as \emph{concepts}~\cite{ganter1999formal, koh2020, espinosa2022}. In practice, however, their applicability is constrained by the need for explicit concept supervision during training~\citep{debole2025if,poeta2023}. Concept annotations are often expensive to obtain, time-consuming, and rarely available at scale within a single data source. 

A natural way to alleviate this limitation is to pool concept knowledge across multiple data owners, such as hospitals, research centers, or devices \cite{litjens1399HE2018, kaissisSecureFL2020}. \emph{Federated Learning} (FL) provides a principled paradigm for collaborative training without sharing raw data, making it particularly appealing in privacy-sensitive settings where concept annotations themselves may encode sensitive information \cite{mcmahanFL2023a, surveyFL2026}. However, combining CMs with FL creates a fundamental mismatch: 
FL is designed for dynamic settings in which clients may join or leave and client data distributions may differ, while still relying on a shared model architecture with a fixed parameterization. CMs challenge this assumption, as their architecture is explicitly tied to the available concept set and, often, to the dependencies among concepts. As a result, changes that FL can typically absorb at the level of client participation or data distribution may require architectural adaptation in CMs. 
This issue becomes especially pronounced in realistic FL deployments, where several sources of variation may exist \cite{jothimurugesanFedDrift2023, jarczewskiLetsF2024, fenoglioFLUX2025, liFLProfile2026}.

To bridge this gap, we introduce \emph{Federated Concept-based Models (F-CMs)}, a new methodology for deploying concept-based architectures in realistic federated settings. 
Crucially, F-CMs handle \emph{statistical heterogeneity} and \emph{temporal non-stationarity}: as new clients join (potentially from new data distributions), they may introduce previously unseen concepts and novel dependencies between them. To accommodate this, F-CMs expand the shared concept space and, when available, update the concept dependency structure. Furthermore, the architecture is adapted by updating only the affected components, thereby avoiding full retraining while preserving previously learned knowledge. To the best of our knowledge, F-CMs are the first to address evolving concept spaces under temporal non-stationarity (dynamic participation and data drift) in FL by adapting the model structure, instead of assuming a fixed shared model. Our key contributions benefit both communities and are as follows: 
\begin{itemize}[noitemsep, topsep=1pt, parsep=1pt, partopsep=1pt, leftmargin=3em]
    \item We introduce \textbf{F-CMs}, a novel methodology for deploying \textbf{concept-based architectures in realistic federated settings}, addressing heterogeneous, partial concept supervision and temporal non-stationarity.
    \begin{itemize}
        \item For evolving FL settings, it provides a systematic way to incorporate architectural \textbf{interpretability} while allowing the shared model structure to adapt during training, moving beyond the fixed-architecture assumption of standard FL.
        \item For interpretability, it \textbf{reduces the cost} for concept annotations at a single institution by enabling privacy-preserving knowledge sharing.
    \end{itemize}
    \item We instantiate F-CMs with four CM architectures and validate them on synthetic benchmarks and two real-world medical imaging datasets. Results show \textbf{comparable predictive and intervention accuracy} w.r.t. full concept supervision, enable \textbf{interpretable inference over concepts missing at individual clients}, and \textbf{reduce retraining cost} under federation growth.
\end{itemize}
With this work, we hope to contribute toward the broader goal of scaling interpretable models while remaining grounded in real-world constraints and deployment scenarios.

\section{Preliminaries}
\subsection{Concept-based Models (CMs)}\label{sec:CMs_background}
CMs predict a task through human-interpretable concepts. In what follows, we consider two main families of CMs, which we refer to as \emph{bipartite CMs} and \emph{graph-based CMs}.

\textbf{Bipartite CMs}, exemplified by ~\citep{koh2020, espinosa2022}, map an input 
$x \in \mathcal{X}$ to a latent representation $z = g(x) \in \mathcal{Z}$, where $z$ is a learned, generally non-interpretable representation of the input. This representation is used to predict a set of $m$ concepts as $\hat{c}_j = h_j(z)$, where each $\hat{c}_j$ may be scalar- or vector-valued. The full concept vector is then given by $\hat{\textbf{c}} = (\hat{c}_1, \dots, \hat{c}_m) \in \mathcal{C}$.
The task prediction is computed from the concepts and, possibly, the latent input representation $\hat{y} = f(\hat{\textbf{c}}, z) \in \mathcal{Y}$. This model induces a \emph{directed acyclic graph} (DAG) $\mathcal{G} = (\mathcal{V}, \mathcal{E})$~\cite{pearl1995causal}  over interpretable variables $\mathcal{V} = \{C_1, \dots, C_m, Y\}$, where edges $\mathcal{E}$ are directed only from concepts to the task, i.e., $C_j \to Y$ for all $j$.

\textbf{Graph-based CMs}, exemplified by~\citep{dominici2024,de2025},  instead allow for dependencies among concepts. In this case, concept and task follow a more general DAG $\mathcal{G}$, where each concept may depend on a set of parent concepts. Concretely, concept and task predictions can be written as
$\hat{c}_j = h_j(\hat{\textbf{c}}_{\mathrm{PA}(C_j)}, z)$ and $\hat{y} = f(\hat{\textbf{c}}_{\mathrm{PA}(Y)}, z)$,
where $\mathrm{PA}$ denotes the parent set induced by $\mathcal{G}$. 
Bipartite CMs can be recovered as a special case of graph-based CMs with only concept-to-task edges. 

\subsection{Federated Learning (FL)} \label{sec:fl_background}
FL trains a shared model across multiple data owners (\emph{clients}) without sharing raw data~\cite{mcmahanFL2023a}. Let $m_{\mathbf{w}}\!:\!\mathcal{X}\to\mathcal{Y}$ denote the shared model with parameters $\mathbf{w}$. The client population may vary over time, denoted by $\mathcal{K}(t)$; at round $t$, the server selects a subset $\mathcal{K}_t\!\subseteq\! \mathcal{K}(t)$ for participation~\cite{jarczewskiLetsF2024, deolivDynamicPart2025}. Each client $k\!\in\!\mathcal{K}(t)$ holds a private dataset $\mathcal{D}^{(k)}\!=\!\{(x_i^{(k)},y_i^{(k)})\}_{i=1}^{n^{(k)}}$. In standard FL, the goal is to minimize a weighted global objective $F(\mathbf{w})\!=\!\sum_{k=1}^K p^{(k)}F^{(k)}(\mathbf{w})$, where $F^{(k)}(\mathbf{w})\!=\!\mathbb{E}_{(x,y)\sim\mathcal{D}^{(k)}}[\ell(m_{\mathbf w}(x),y)]$ is the local objective of client $k$ and $p^{(k)}$ typically weighs clients proportionally to their data size. At round $t$, the server broadcasts the current weights $\mathbf{w}_t$ to clients in $\mathcal{K}_t$, each selected client performs local optimization on $\mathcal{D}^{(k)}$, and the server aggregates the returned updates to obtain $\mathbf{w}_{t+1}^{(k)}$, typically via \emph{FedAvg} \cite{mcmahanFL2023a}, i.e., by weighted averaging proportional to local dataset sizes.

\textbf{Statistical heterogeneity and temporal non-stationarity.}
Practical FL deployments are rarely stationary \cite{xiangNonStationary2024}. First, FL is typically \emph{statistically heterogeneous}: different clients observe different data distributions, i.e., $\exists\,k\neq k' \text{ s.t. } \mathbb{P}^{(k)}(X,Y)\neq \mathbb{P}^{(k')}(X,Y)$, so their local objectives $F^{(k)}$ differ \cite{fenoglioFLUX2025, kairouzAdvancesOpen2021}. Second, FL can be \emph{temporally non-stationary}, meaning that the data and/or the set of participating clients changes over rounds. This commonly arises through (i) \emph{dynamic client participation}, where the available client pool $\mathcal{K}(t)$ and sampled participants $\mathcal{K}_t$ vary over time as clients join or leave \cite{deolivDynamicPart2025}; and (ii) \emph{data drift}, where even for a fixed client $k$, its distribution changes over time, i.e., $\mathbb{P}^{(k)}_t(X,Y)\neq \mathbb{P}^{(k)}_{t'}(X,Y)$ for $t\neq t'$ \cite{liFLProfile2026, panchalFlash2023}. Both effects induce a time-varying optimization landscape and can destabilize global updates.

These conditions violate key assumptions of standard CMs. Usually, standard CMs assume a concept set that determines the architecture. In FL settings, however, concept annotations would be distributed across clients and may evolve with clients or data, so concepts and their dependencies can change over time. Consequently, a federated CM must (i) reconcile heterogeneous, partial concept supervision across clients, and (ii) dynamically update concepts, their dependencies and corresponding architecture during training, rather than relying on a fixed design upfront.

\subsection{Problem Setting}
\label{sec:problem_formulation} 
We study FL of CMs under statistical heterogeneity and temporal non-stationarity. 
Training proceeds over rounds $t\!=\!1,\dots,T$, where a subset of clients $\mathcal{K}_t$ participates. Client $k$ holds a local dataset
$\mathcal D^{(k)}\!=\!
\{
(x_i^{(k)}, \mathbf{v}_i^{(k)})
\}_{i=1}^{n^{(k)}}$,
where $x_i^{(k)}\!\in\!\mathcal X$ is the input  and $\mathbf{v}_i^{(k)}$ contains the available supervision for sample $i$, i.e.,  concept annotations and/or task label. 
We denote by $\mathcal V^{(k)} \neq \emptyset$ the whole set of interpretable variables (concepts and task) supervised by client $k$. Each client \(k\) additionally provides a DAG $\mathcal{G}^{(k)}$, represented by a \emph{weighted adjacency matrix} $A^{(k)}\!\in\![0,1]^{|\mathcal{V}^{(k)}|\times|\mathcal{V}^{(k)}|}$\cite{barabasi2013}. 
Notably, for bipartite CMs, the structure is fixed a priori by allowing only concepts-to-task edges. 

Clients may dynamically join or leave the federation over time; when a client leaves, its raw data is discarded, while its contribution to the learned model is retained. Consequently, the global set of supervised variables evolves as
\(
\mathcal{V}_t \!= \!\bigcup_{\tau=1}^{t} \bigcup_{k \in \mathcal{K}_\tau} \mathcal{V}^{(k)}.
\)


The problem is to learn, at each round $t$, a shared CM $F_{\mathbf{w}_t}$ that minimizes the federated loss induced by the supervision available at each client, while adapting architecture as the global variable set $\mathcal{V}_t$ and its dependencies evolve. Each client contributes only losses for the variables it supervises, so the server must reconcile partial objectives without accessing raw data.

\section{Federated Concept-based Models}\label{sec:method}
We now formalize \textbf{Federated Concept-based Models (F-CMs)} as a federated training framework for bipartite and graph-based CMs under partial, heterogeneous concept supervision and temporal non-stationarity; Algorithm~\ref{alg:fcms} in App.~\ref{app:algo} summarizes the full procedure. As in standard FL, F-CMs maintain a single shared CM optimized through repeated rounds of client-side training and server-side aggregation. Beyond standard FL, F-CMs support an evolving concept space via three key extensions:
\begin{enumerate}[topsep=1pt,itemsep=1pt,parsep=0pt,partopsep=0pt,leftmargin=1.9em]
    \item \textbf{Graph aggregation}: 
    at each round t, the server aggregates client-provided structural information ($A^{(k)}$) into a shared DAG $\mathcal{G}_t$ encoding dependencies and defining the model architecture;
    \item \textbf{Dynamic architecture adaptation}: the shared CM is decomposed into \emph{concept-specific} and \emph{task-specific} modules, where each module denotes the component used to predict a given concept or the task label. This enables the interpretable architecture to expand and rewire modularly as new concepts or dependencies emerge. For instance, when a client introduces a new concept or a new concept dependency, F-CMs add or modify only the modules associated with the involved concepts, while leaving the remaining components unchanged.
    \item \textbf{Module-specific optimization}: each client updates only the modules associated with the interpretable variables it supervises, while modules for unobserved concepts or task labels remain frozen. Server-side aggregation is then performed separately for each module, using only the updates from clients that provided supervision for the corresponding variable.

\end{enumerate}
\begin{figure*}[t]
   \centering \includegraphics[width=\linewidth]{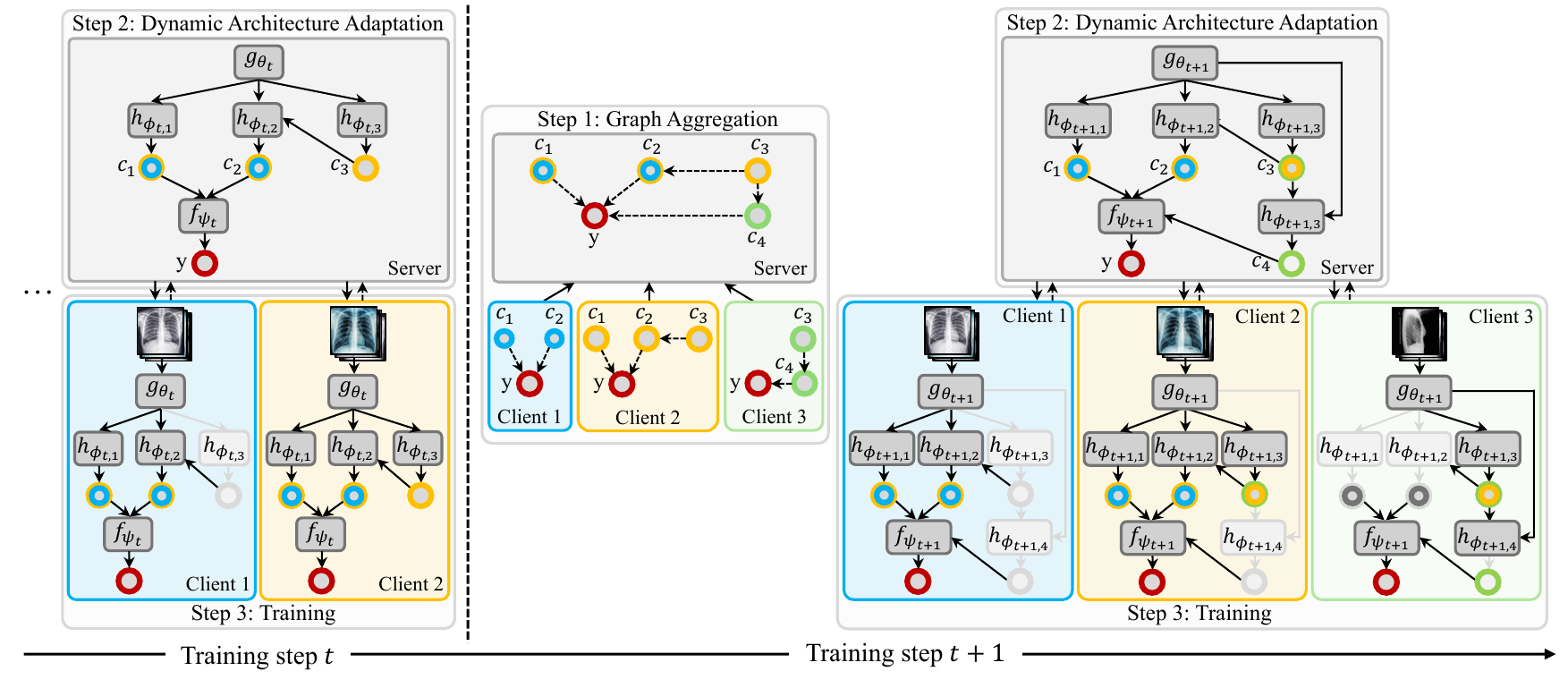}
   \captionsetup{skip=2pt} 
   \caption{\textbf{F-CMs overview.} At each round $t$, F-CMs perform three steps: (1) \emph{Graph aggregation:} the server combines client graphs into a shared graph $\mathcal{G}_t$; (2) \emph{Architecture adaptation:} the shared CM adapts to $\mathcal{G}_t$ and/or newly observed concepts; (3) \emph{Training:} clients update only supervised concept/task modules (dark grey), leaving others frozen (light grey), and the server aggregates the corresponding updates. The figure shows two rounds, where new clients introduce new concepts.}
   \vspace{-10pt}
   \label{fig:overview}
\end{figure*}

\vspace{-2pt}
\subsection{F-CM Shared Model}
\label{sec:architecture} \vspace{-2pt}
At round $t$, the server maintains a shared CM (bipartite or graph-based), with parameters 
$\mathbf{w}_t=(\theta_t,\phi_t,\psi_t)$, which maps inputs $x$ to predictions over variables $\mathcal{V}_t$ through the following modules:
\begin{itemize}[topsep=1pt,itemsep=1pt,parsep=0pt,partopsep=0pt,leftmargin=1.9em]
    \item a \emph{latent encoder} $g_{\theta_t}:\mathcal{X}\to\mathbb{R}^d$ producing $z=g_{\theta_t}(x)$;

    \item a collection of \emph{concept encoders} $\{h_{\phi_t,j}\}_j$
    predicting the value of concepts \(C_j \in \mathcal V_t \setminus \{Y\}\) as
    $\hat c_j = h_{\phi_t,j}\big(\hat{ \textbf{c}}_{\mathrm{PA}_t(C_j)}, z\big)$; 

    \item a \emph{task decoder} $f_{\psi_t}$ producing $\hat y = f_{\psi_t}\big(\hat {\textbf{c}}_{\mathrm{PA}_t(Y)}, z\big)$;
\end{itemize}
Here, $\mathrm{PA}_t$ denotes the parent set in a shared DAG 
$\mathcal{G}_t$ constructed over $\mathcal{V}_t$ by 
aggregating client-specific DAGs $\mathcal{G}^{(k)}$, as described in 
Section~\ref{sec:training}. 
In bipartite instantiations, $\mathcal{G}_t$ allows only concept-to-task edges, hence 
$\mathrm{PA}_t(C_j)=\emptyset$ for every concept $C_j$, and 
$\mathrm{PA}_t(Y)=\mathcal{V}_t \setminus \{Y\}$.  

\textbf{Instantiations. }F-CMs are agnostic to the specific CM architecture used as the shared model. An instantiation specifies the parameterization of modules $\{h_{\phi_t,j}\}_j$ and $f_{\psi_t}$.

\subsection{F-CM Pipeline}
\label{sec:training}
We next detail how F-CMs aggregate structure, adapt the shared architecture, and perform module-specific optimization under partial supervision. Figure~\ref{fig:overview} provides an overview.

\subsubsection{Graph Aggregation}
\label{sec:graph_agg}
A single client’s structural knowledge can be unreliable: local graphs may be noisy due to limited, non-IID data and may also be manipulated by compromised clients (i.e., poisoning)~\cite{sunDataPoisoning2022}. 
We therefore aggregate client-provided structures into a shared DAG $\mathcal{G}_t$ over the current global variable set $\mathcal{V}_t$, by favoring edges and directions \emph{with the strongest support across observed clients}.

Specifically, each client $k$ sends a representation of its local DAG $\mathcal{G}^{(k)}$ to the server only upon joining or when its local structure changes, encoded as a weighted adjacency matrix $A^{(k)}$ over its supervised variables $\mathcal{V}^{(k)}$. 
We interpret $A^{(k)}_{ij}$ as the client’s confidence in the directed edge $V_i\!\to\!V_j$. 
For each unordered pair $\{V_i,V_j\}\subseteq\mathcal{V}_t$ with $i\neq j$, we define a categorical distribution over the three possible outcomes
\(\{V_i \to V_j,\; V_j \to V_i,\; \varnothing\}\), with probabilities
$A^{(k)}_{ij}, A^{(k)}_{ji}, A^{(k)}_{\varnothing,ij}$, respectively. Here, $A^{(k)}_{\varnothing,ij}
\coloneqq
1 - A^{(k)}_{ij} - A^{(k)}_{ji}$
denotes the confidence assigned by client $k$ to the absence of an edge between $V_i$ and $V_j$. The server aggregates client confidences into a \emph{strength} matrix by summing contributions from all clients observed up to round $t$, denoted by $\mathcal{K}_{\le t}$, using the latest graph received from each client:
\[
\setlength{\abovedisplayskip}{3pt} 
\setlength{\belowdisplayskip}{3pt}
 [\bar{A}_t]_{ij} := \sum_{k\in\mathcal{K}_{\le t}:\, V_i,V_j\in \mathcal{V}^{(k)}} \alpha_k\, A^{(k)}_{ij},
\]
where $\alpha_k$ are aggregation weights (e.g., proportional to client dataset sizes, or uniform/reliability-aware; see App.~\ref{app:graph_agg_ablation}), and $ [\bar{A}_t]_{ij}\!=\!0$ if the sum is empty. The same rule is applied to $[\bar {A}_t]_{\varnothing,ij}$. The shared graph $\mathcal{G}_t$ is then constructed by applying a maximum-consensus rule to each pair
\setlength{\abovedisplayskip}{5pt} 
\setlength{\belowdisplayskip}{5pt}
\[
(V_i\!\to\! V_j)\in \mathcal{G}_t
\;\Longleftrightarrow\;
 [\bar{A}_t]_{ij}
=\max\!\big\{ [\bar{A}_t]_{ij},[\bar{A}_t]_{ji},[\bar {A}_t]_{\varnothing,ij}\big\}.
\]
Ties are broken at random. 
Cycles are resolved iteratively (Sec.~\ref{parag:DAG_construction_cycle_res}).
App.~\ref{app:diff_privacy} further shows that graph communication can be equipped with an optional edge-level differentially private mechanism, providing edge-level structural privacy with limited impact on graph quality and predictive performance.

\subsubsection{Dynamic Architecture Adaptation}
\label{sec:arch_update}
The shared DAG \(\mathcal G_t\) from the previous step defines the CM connectivity by
specifying predicted variables and dependencies. Under non-stationarity, these may change as new clients introduce unseen concepts or revised dependencies. 
Accordingly, F-CMs adapt the shared CM only when module interfaces change, while preserving all unaffected ones. At round $t$, this amounts to updating the concept encoders $\{h_{\phi_t,j}\}_{j}$ and $f_{\psi_t}$. 
The latent encoder $g_{\theta_t}$ remains unchanged. 

\textbf{(E1) Adding a new concept.}
For any newly observed concept $C_j$, the server instantiates a new module $h_{\phi_t,j\!}$. The task decoder is updated 
if $C_j\!\in\! \mathrm{PA}_t(Y)$, while other modules remain unchanged.

\textbf{(E2) Updating an edge (add/remove/re-orient).}
In graph-based F-CMs, changing the relation between two variables modifies parent sets and thus module inputs: adding $C_i\!\to\!C_j$ appends $\hat c_i$ to the inputs of $h_{\phi,j}$; removing $C_i\!\to\!C_j$ removes it; re-orienting $C_i\!\to\!C_j$ into $C_j\!\to\!C_i$ is treated as remove+add, so $h_{\phi,j}$ loses $\hat c_i$ while $h_{\phi,i}$ gains $\hat c_j$. The task decoder is updated analogously. In bipartite F-CMs, updates are constrained, disallowing concept-to-concept edges and re-orientations.

More specifically, when the parent set of a node changes, only the
module predicting that node is modified. Conversely, changing a
node's children does not modify that node's own module, but only the
modules predicting the affected children. The concrete resizing
operations depend on the chosen CM instantiation and are detailed in
App.~\ref{app:baselines_details}.

\textbf{Warm-start initialization.}
When the module input expands, only the newly introduced parameters are initialized, while existing parameters are preserved. For example, weights connected to new inputs may be zero-initialized to preserve pre-update behavior, or initialized randomly. Modules whose interfaces are unchanged retain their parameters exactly, avoiding full retraining.

\subsubsection{Module-Specific Training and Module-Wise Aggregation}
\label{sec:fed_training}
Because concept and task supervision are unevenly distributed across clients, vanilla FedAvg \cite{mcmahanFL2023a} would mix updates for modules that some clients cannot supervise. F-CMs therefore perform module-specific training and module-wise aggregation, restricting contributions to supervised components. This induces a masked federated optimization problem, where each client optimizes a client-specific partial objective over a subset of shared parameters. 

\textbf{Module-specific local optimization.}
At round \(t\), each client \(k \in \mathcal K_t\) receives the broadcasted shared CM
\(( \mathbf{w}_t,\mathcal G_t)\), where \(\mathbf{w}_t\) denotes the pre-training parameters obtained after
architecture adaptation, and performs $E$ local optimization steps using only its available supervision. The local objective is
\begin{equation}
\setlength{\abovedisplayskip}{3pt} 
\setlength{\belowdisplayskip}{3pt}
F^{(k)}(\mathbf{w}_t;\mathcal{G}_t)
=
\mathbb{E}_{(x,\mathbf{v})\sim\mathcal{D}^{(k)}}
\bigg[
\gamma
\sum_{C_j \in \mathcal{V}^{(k)} \setminus \{Y\}}
\ell_c\big(\hat{c}_j, c_j\big)
+
(1-\gamma)\,\mathbf{1}_{\{Y \in \mathcal{V}^{(k)}\}}\,
\ell_t\big(\hat{y}, y\big)
\bigg],
\label{eq:loss}
\end{equation}

where $\ell_c$ and $\ell_t$ are standard concept and task losses (e.g., cross-entropy), and $\gamma\in[0,1]$. Only supervised 
modules contribute gradients; unsupervised ones are frozen and 
need not be transmitted. 

\textbf{Module-wise aggregation.}
The server aggregates each module using only the clients that produced updates for that module. We use a FedAvg-like rule, 
although F-CMs can in principle support other aggregation rules and architectures \cite{wang2020tackling, reddi2021adaptive, blanchardkrum2017b, ERISEnhancingPrivacy2026c}. Let $\mathcal{K}^j_t\subseteq\mathcal{K}_t$ denote the clients updating concept module $h_{\phi_t,j}$, yielding local parameters $\phi^{(k)}_{{t+1},j}$. The aggregated parameters are
\[
\setlength{\abovedisplayskip}{3pt} 
\setlength{\belowdisplayskip}{3pt}
\phi_{t+1,j} = \sum_{k \in \mathcal{K}^j_t} \beta^{(k)}_j \, \phi^{(k)}_{t+1,j},
\qquad
\beta^{(k)}_j = \frac{n^{(k)}}{\sum_{r\in \mathcal{K}^j_t} n^{(r)}}.
\]
The aggregation weights are normalized independently for each module,
so the number of clients supervising one concept does not affect the
aggregation weights of other concept modules. 

The task module is aggregated analogously over clients that updated it, while modules not updated by any client remain unchanged.
Since all participating clients update the latent encoder, its parameters are aggregated with the same FedAvg-like rule over $\mathcal{K}_t$.

Importantly, while only clients with supervision for a concept contribute to updating its prediction module, its predictions can still inform the task decoder and, in graph-based models, downstream concept modules, even at clients without local annotations for that concept.

\paragraph{Theoretical analysis.}
F-CM training is non-standard: clients may have partial supervision, aggregation is module-wise, and the architecture may change. Appendix~\ref{app:convergence} provides a theoretical analysis for a static version of F-CM with fixed clients and concept supervision, together with a piecewise-stationary extension covering changes in the client set, supervision patterns, and architecture at regime boundaries. The static analysis allows us to study, in isolation, the effects of partial supervision and module-wise aggregation on training. Specifically, we show that if a global objective exists under standard assumptions,
and the aggregated module-wise updates are sufficiently aligned with its gradient, then the
usual nonconvex stationarity behavior is recovered up to a bounded mismatch. Exact alignment
gives the standard \(O(T^{-1/2})\) rate. We further empirically quantify this alignment mismatch along the optimization trajectory. In the piecewise-stationary setting, the same
average-stationarity guarantee is preserved under eventual stabilization, bounded regime-start objective gaps, and uniformly bounded
regime-specific constants.

\section{Experimental Evaluation} \label{sec:exp_evaluation}
\begin{table*}[t]
\caption{\textbf{Task accuracy} (\%) and \textbf{concept coverage} (\%) with respect to the ground-truth concepts causally relevant for the task, i.e., concepts with a directed path to the task in the ground-truth graph or the available proxy. Methods that improve the average performance are shown in bold. Results are reported as mean $\pm$ standard error over five runs.
}
\centering
\scriptsize
\setlength{\tabcolsep}{1.3pt}
\resizebox{\textwidth}{!}{%
\begin{tabular}{ll|cc|cc|cc|cc|cc|cc|cc}
\toprule
\textbf{Setting} & \textbf{Method} & 
\multicolumn{2}{c|}{\textbf{\textsc{Asia}}} & 
\multicolumn{2}{c|}{\textbf{\textsc{Sachs}}} & 
\multicolumn{2}{c|}{\textbf{\textsc{Alarm}}} & 
\multicolumn{2}{c|}
{\textbf{\textsc{Insurance}}} & 
\multicolumn{2}{c|}{\textbf{\textsc{Hailfinder}}} & 
\multicolumn{2}{c}{\textbf{\textsc{SIIM-Pn.}}} & 
\multicolumn{2}{c}{\textbf{\textsc{CheXpert}}} \\
\cmidrule(lr){3-4} \cmidrule(lr){5-6} \cmidrule(lr){7-8} \cmidrule(lr){9-10} \cmidrule(lr){11-12}
\cmidrule(lr){13-14} \cmidrule(lr){15-16}
& & T & C. Cov. & T & C. Cov. & T & C. Cov. & T & C. Cov. & T & C. Cov. & T & C. Cov. & T & C. Cov. \\
\midrule
\multirow{6}{*}{\shortstack{Cent.\\(Upper\\Bound)}}
 & OpaqNN    
 & 80.7{\tiny$\pm$1.0} & 100
 & 77.7{\tiny$\pm$1.0} & 100
 & 74.0{\tiny$\pm$0.3} & 100
 & 76.3{\tiny$\pm$1.1} & 100
 & 73.0{\tiny$\pm$0.8} & 100
 & 74.1{\tiny$\pm$0.6} & 100
 & 72.0{\tiny$\pm$0.3} & 100\\
 & CBM 
 & 80.0{\tiny$\pm$1.0} & 100
 & 77.3{\tiny$\pm$0.9} & 100
 & 73.9{\tiny$\pm$0.5} & 100
 & 76.3{\tiny$\pm$1.4} & 100
 & 74.2{\tiny$\pm$0.7} & 100
 & 73.9{\tiny$\pm$0.3} & 100
 & 70.7{\tiny$\pm$0.5} & 100\\
 & CEM       
 & 80.4{\tiny$\pm$1.1} & 100
 & 77.3{\tiny$\pm$0.8} & 100
 & 71.6{\tiny$\pm$0.4} & 100
 & 76.0{\tiny$\pm$1.2} & 100
 & 73.3{\tiny$\pm$0.6} & 100
 & 73.1{\tiny$\pm$0.7} & 100
 & 66.6{\tiny$\pm$0.7} & 100\\
 & CGM       
 & 79.8{\tiny$\pm$1.2} & 100
 & 76.2{\tiny$\pm$1.2} & 100
 & 71.5{\tiny$\pm$0.9} & 100
 & 73.4{\tiny$\pm$1.1} & 100
 & 72.2{\tiny$\pm$0.8} & 100
 & 68.8{\tiny$\pm$0.8} & 100
 & 71.6{\tiny$\pm$0.3} & 100\\
 & C$^2$BM   
 & 80.3{\tiny$\pm$1.0} & 100
 & 77.9{\tiny$\pm$1.0} & 100
 & 72.8{\tiny$\pm$0.5} & 100
 & 74.8{\tiny$\pm$0.9} & 100
 & 73.5{\tiny$\pm$0.6} & 100
 & 72.8{\tiny$\pm$0.1} & 100
 & 69.4{\tiny$\pm$0.7} & 100\\
\midrule
\multirow{6}{*}{Loc.}
 & OpaqNN      
 & 56.5{\tiny$\pm$6.7} & 46.7{\tiny$\pm$3.3}
 & 50.2{\tiny$\pm$9.4} & 52.0{\tiny$\pm$4.9}
 & 45.0{\tiny$\pm$11.1} & 67.0{\tiny$\pm$5.1}
 & 19.4{\tiny$\pm$6.2} & 72.6{\tiny$\pm$6.1}
 & 48.1{\tiny$\pm$7.9} & 63.7{\tiny$\pm$3.5}
 & 50.9{\tiny$\pm$8.6} & 59.4{\tiny$\pm$3.1}
 & 57.7{\tiny$\pm$4.7} & 31.4{\tiny$\pm$5.4} \\
 & CBM   
 & 47.7{\tiny$\pm$3.8} & 46.7{\tiny$\pm$3.3}
 & 50.1{\tiny$\pm$7.0} & 52.0{\tiny$\pm$4.9}
 & 51.7{\tiny$\pm$7.0} & 67.0{\tiny$\pm$5.1}
 & 30.9{\tiny$\pm$7.6} & 72.6{\tiny$\pm$6.1}
 & 46.6{\tiny$\pm$7.6} & 63.7{\tiny$\pm$3.5}
 & 44.2{\tiny$\pm$6.2} & 59.4{\tiny$\pm$3.1}
 & 57.7{\tiny$\pm$4.2} & 31.4{\tiny$\pm$5.4} \\
 & CEM         
 & 67.1{\tiny$\pm$7.4} & 46.7{\tiny$\pm$3.3}
 & 52.2{\tiny$\pm$8.8} & 52.0{\tiny$\pm$4.9}
 & 42.1{\tiny$\pm$11.1} & 67.0{\tiny$\pm$5.1}
 & 40.8{\tiny$\pm$5.3} & 72.6{\tiny$\pm$6.1}
 & 44.1{\tiny$\pm$7.1} & 63.7{\tiny$\pm$3.5}
 & 62.4{\tiny$\pm$1.4} & 59.4{\tiny$\pm$3.1}
 & 57.3{\tiny$\pm$4.7} & 31.4{\tiny$\pm$5.4} \\
 & CGM         
 & 64.7{\tiny$\pm$5.1} & 46.7{\tiny$\pm$3.3}
 & 37.0{\tiny$\pm$9.4} & 52.0{\tiny$\pm$4.9}
 & 46.1{\tiny$\pm$8.9} & 50.4{\tiny$\pm$6.5}
 & 31.7{\tiny$\pm$7.2} & 70.5{\tiny$\pm$7.4}
 & 44.1{\tiny$\pm$7.3} & 63.7{\tiny$\pm$3.5}
 & 56.6{\tiny$\pm$6.0} & 59.4{\tiny$\pm$3.1}
 & 60.9{\tiny$\pm$4.9} & 31.4{\tiny$\pm$5.4} \\
 & C$^2$BM        
 & 66.7{\tiny$\pm$7.1} & 46.7{\tiny$\pm$3.3}
 & 37.3{\tiny$\pm$13.8} & 52.0{\tiny$\pm$4.9}
 & 45.0{\tiny$\pm$8.9} & 50.4{\tiny$\pm$6.5}
 & 22.1{\tiny$\pm$12.1} & 70.5{\tiny$\pm$7.4}
 & 43.0{\tiny$\pm$5.6} & 63.7{\tiny$\pm$3.5}
 & 53.8{\tiny$\pm$6.2} & 59.4{\tiny$\pm$3.1}
 & 58.6{\tiny$\pm$4.2} & 31.4{\tiny$\pm$5.4} \\
 \midrule
\multirow{2}{*}{\shortstack{External\\Static FL}}
 & FCL 
 & 63.9{\tiny$\pm$6.4} & N/A
 & 59.5{\tiny$\pm$1.3} & N/A
 & 50.1{\tiny$\pm$5.9} & N/A
 & 62.3{\tiny$\pm$1.8} & N/A
 & 31.7{\tiny$\pm$2.6} & N/A
 & 56.5{\tiny$\pm$5.9} & N/A
 & 50.9{\tiny$\pm$0.9} & N/A \\ 
 & FedCBM
 & 68.0{\tiny$\pm$4.8} & 48.6{\tiny$\pm$5.7}
 & 48.7{\tiny$\pm$9.4} & 54.0{\tiny$\pm$5.1}
 & 54.8{\tiny$\pm$8.8} & 70.0{\tiny$\pm$4.1}
 & 59.7{\tiny$\pm$8.7} & 76.9{\tiny$\pm$2.4}
 & 51.1{\tiny$\pm$7.4} & 84.4{\tiny$\pm$2.9}
 & 54.6{\tiny$\pm$4.6} & 63.6{\tiny$\pm$3.7}
 & 66.1{\tiny$\pm$4.0} & 34.3{\tiny$\pm$3.5}\\ 
 \midrule
\multirow{6}{*}{S-F-CMs}
 & OpaqNN      
 & 74.0{\tiny$\pm$4.9} & 46.7{\tiny$\pm$3.3}
 & 55.2{\tiny$\pm$9.1} & 56.0{\tiny$\pm$4.0}
 & 61.7{\tiny$\pm$7.4} & 71.3{\tiny$\pm$5.4}
 & 64.3{\tiny$\pm$8.5} & 83.2{\tiny$\pm$2.6}
 & 54.0{\tiny$\pm$7.6} & 84.7{\tiny$\pm$2.8}
 & 63.1{\tiny$\pm$1.1} & 68.8{\tiny$\pm$3.6}
 & 56.1{\tiny$\pm$9.0} & 34.3{\tiny$\pm$3.5} \\
 & CBM   
 & 63.5{\tiny$\pm$3.6} & 46.7{\tiny$\pm$3.3}
 & 42.2{\tiny$\pm$12.1} & 56.0{\tiny$\pm$4.0}
 & 57.0{\tiny$\pm$8.0} & 71.3{\tiny$\pm$5.4}
 & 58.4{\tiny$\pm$11.8} & 83.2{\tiny$\pm$2.6}
 & 51.8{\tiny$\pm$10.8} & 84.7{\tiny$\pm$2.8}
 & 55.8{\tiny$\pm$6.2} & 68.8{\tiny$\pm$3.6}
 & 67.0{\tiny$\pm$4.2} & 34.3{\tiny$\pm$3.5} \\
 & CEM     
 & 72.8{\tiny$\pm$8.0} & 46.7{\tiny$\pm$3.3}
 & 39.9{\tiny$\pm$15.2} & 56.0{\tiny$\pm$4.0}
 & 56.2{\tiny$\pm$10.9} & 71.3{\tiny$\pm$5.4}
 & 67.0{\tiny$\pm$3.4} & 83.2{\tiny$\pm$2.6}
 & 55.1{\tiny$\pm$9.7} & 84.7{\tiny$\pm$2.8}
 & 56.8{\tiny$\pm$6.5} & 68.8{\tiny$\pm$3.6}
 & 54.9{\tiny$\pm$9.2} & 34.3{\tiny$\pm$3.5} \\
 & CGM     
 & 70.8{\tiny$\pm$8.3} & 46.7{\tiny$\pm$3.3}
 & 39.3{\tiny$\pm$14.9} & 56.0{\tiny$\pm$4.0}
 & 54.6{\tiny$\pm$10.5} & 51.3{\tiny$\pm$6.7}
 & 58.2{\tiny$\pm$11.8} & 76.8{\tiny$\pm$5.9}
 & 50.6{\tiny$\pm$7.7} & 83.3{\tiny$\pm$2.4}
 & 61.8{\tiny$\pm$0.3} & 68.8{\tiny$\pm$3.6}
 & 65.7{\tiny$\pm$3.8}& 34.3{\tiny$\pm$3.5} \\
 & C$^2$BM 
 & 70.4{\tiny$\pm$7.7} & 46.7{\tiny$\pm$3.3}
 & 46.5{\tiny$\pm$15.5} & 56.0{\tiny$\pm$4.0}
 & 57.1{\tiny$\pm$10.4} & 51.3{\tiny$\pm$6.7}
 & 60.1{\tiny$\pm$13.9} & 78.9{\tiny$\pm$7.1}
 & 59.1{\tiny$\pm$8.2} & 83.3{\tiny$\pm$2.4}
 & 50.0{\tiny$\pm$8.0} & 68.8{\tiny$\pm$3.6}
 & 61.6{\tiny$\pm$4.9} & 34.3{\tiny$\pm$3.5} \\
\midrule
\multirow{6}{*}{\textbf{F-CMs}}
 & OpaqNN      
 & \textbf{80.5}{\tiny$\pm$\textbf{1.1}} & \textbf{100}
 & \textbf{75.2}{\tiny$\pm$\textbf{1.0}} & \textbf{100}
 & \textbf{72.7}{\tiny$\pm$\textbf{0.7}} & \textbf{100}
 & \textbf{71.3}{\tiny$\pm$\textbf{0.6}} & \textbf{100}
 & \textbf{68.0}{\tiny$\pm$\textbf{1.4}} & \textbf{100}
 & \textbf{64.8}{\tiny$\pm$\textbf{1.0}} & \textbf{100}
 & \textbf{70.0} {\tiny$\pm$\textbf{0.2}} & \textbf{100} \\
 & CBM     
 & \textbf{80.1}{\tiny$\pm$\textbf{0.9}} & \textbf{100}
 & \textbf{75.2}{\tiny$\pm$\textbf{0.9}} & \textbf{100}
 & \textbf{71.0}{\tiny$\pm$\textbf{1.9}} & \textbf{100}
 & \textbf{69.5}{\tiny$\pm$\textbf{0.6}} & \textbf{100}
 & \textbf{70.6}{\tiny$\pm$\textbf{1.5}} & \textbf{100}
 & \textbf{67.6}{\tiny$\pm$\textbf{1.1}} & \textbf{100}
 & \textbf{69.2} {\tiny$\pm$\textbf{0.7}}  & \textbf{100} \\
 & CEM     
 & \textbf{80.7}{\tiny$\pm$\textbf{1.0}} & \textbf{100}
 & \textbf{75.7}{\tiny$\pm$\textbf{1.0}} & \textbf{100}
 & \textbf{72.8}{\tiny$\pm$\textbf{0.7}} & \textbf{100}
 & \textbf{71.0}{\tiny$\pm$\textbf{0.5}} & \textbf{100}
 & \textbf{71.6}{\tiny$\pm$\textbf{1.0}} & \textbf{100}
 & \textbf{68.0}{\tiny$\pm$\textbf{0.1}} & \textbf{100}
 & \textbf{67.0} {\tiny$\pm$\textbf{0.1}} & \textbf{100} \\
 & CGM     
 & \textbf{77.8}{\tiny$\pm$\textbf{1.3}} & \textbf{100}
 & \textbf{71.7}{\tiny$\pm$\textbf{1.8}} & \textbf{100}
 & \textbf{68.5}{\tiny$\pm$\textbf{3.3}} & \textbf{100}
 & \textbf{70.0}{\tiny$\pm$\textbf{1.1}} & \textbf{100}
 & \textbf{63.8}{\tiny$\pm$\textbf{3.4}} & \textbf{100}
 & \textbf{63.6}{\tiny$\pm$\textbf{0.2}} & \textbf{100}
 & \textbf{67.5} {\tiny$\pm$\textbf{0.3}} & \textbf{100} \\
 & C$^2$BM 
 & \textbf{80.4}{\tiny$\pm$\textbf{1.0}} & \textbf{100}
 & \textbf{76.2}{\tiny$\pm$\textbf{1.0}} & \textbf{100}
 & \textbf{73.3}{\tiny$\pm$\textbf{0.5}} & \textbf{100}
 & \textbf{74.2}{\tiny$\pm$\textbf{1.1}} & \textbf{100}
 & \textbf{72.0}{\tiny$\pm$\textbf{1.8}} & \textbf{100}
 & \textbf{65.6}{\tiny$\pm$\textbf{2.0}} & \textbf{100}
 & \textbf{66.1} {\tiny$\pm$\textbf{1.3}} & \textbf{100} \\
\bottomrule
\end{tabular}
}
\label{tab:task_accuracy}
\end{table*}
 In our experiments, we instantiate F-CMs with two bipartite CMs, CBM~\cite{koh2020} and CEM~\cite{espinosa2022}, and two graph-based CMs, CGM~\cite{dominici2024} and C$^2$BM~\cite{de2025}. CBM predicts scalar concepts, CEM predicts concept embeddings, while CGM and C$^2$BM condition concept predictions on parent concepts in a DAG (details in App.~\ref{app:baselines_details}). Specifically, C$^2$BM assumes a \emph{causal} DAG, where each directed edge 
represents a direct causal relation between variables. 
By contrast, CGMs can be instantiated with different types of graphs. 
To provide a unified treatment across graph-based models, we use causal DAGs 
also for CGMs. 
This is an implementation choice: F-CMs can also be used with other DAG 
structures.

We empirically evaluate F-CMs across interpretability and FL objectives. Specifically, we assess (i) task accuracy and concept coverage (Section~\ref{sec:taskconcept_acc}), (ii) 
responsiveness to ground-truth concept interventions (\ref{sec:intervention_acc}), and (iii) training efficiency under evolving federations with dynamic architecture updates (\ref{sec:exp_trends}). The appendix complements this evaluation with extended intervention and adaptation analyses across architectures and instantiations, and analyses of the extent and computational overhead of modular adaptation (\ref{app:concept_acc}--\ref{app:overhead}).
We also report privacy--utility trade-offs under sample-level differential privacy for model updates and edge-level differential privacy for client graph communication (\ref{app:diff_privacy}), sensitivity analyses with respect to the number of clients, selective parameter freezing, the concept-loss weight, client arrival timing,
and concept-space expansion (\ref{app:sensitivity}), robustness to inter-client semantic disagreement, missing and noisy concept
annotations, and noisy client graphs (\ref{app:d8} and~\ref{app:robust_graph_perturb}), and an ablation of graph aggregation strategies under graph perturbations (\ref{app:graph_agg_ablation}).
Finally, we study a natural extension of F-CMs to multimodal settings, particularly relevant in real-world FL where clients access heterogeneous modalities while sharing common semantic targets; App.~\ref{app:multimodal} formalizes this extension and reports an  experiment on a realistic multimodal dataset.


\textbf{Baselines and training regimes.}
We compare F-CM instantiations to an opaque neural baseline (\textbf{OpaqNN}) under four training regimes:
(i) \textbf{Centralized (Cent.)}, pooling all client data and concept annotations, serving as an upper bound;
(ii) \textbf{F-CMs} (ours);
(iii) \textbf{S-F-CMs}, a static federated variant of F-CMs that uses the same optimization and structure aggregation, but fixes the concept space and architecture after the first-round client participation (no adaptation);
(iv) \textbf{Localized (Loc.)}, training independent client models without sharing.
We also include two state-of-the-art federated concept-learning baselines, FedCBM \cite{yangFedCBM2024} and FCL \cite{shenFCL2024} reported separately as \textbf{external static FL baselines} since they assume a static federated setting and rely on method-specific concept-learning pipelines.

\textbf{Federated protocol under temporal non-stationarity.}
Federated experiments are conducted under \emph{non-stationarity} induced by
dynamic client participation and evolving concept supervision. 
At each round $t$, the server samples 
$|\mathcal{K}_t|\!=\!10$ clients. 
The client pool $\mathcal{K}(t)$ evolves up to 20 clients during training: after a warm-up phase (at $t\!=\!10$ for \textsc{Asia}, \textsc{Sachs}, and \textsc{SIIM}, and at $t\!=\!20$ otherwise; App.~\ref{app:datapreparation}), new clients join 
with additional concept supervision and, for graph-based F-CMs, potentially new dependencies. Experiments up to $|\mathcal{K}_t|\!=\!100$ are reported in App. \ref{app:scaling_clients}; 
additional ablations on client arrival timing and concept-space expansion are reported
in ~\ref{app:arrival-timing}--\ref{app:concept-expansion}.


\textbf{Datasets and data preparation under heterogeneous partial concept supervision.}
We consider five Bayesian network benchmarks from \texttt{bnlearn}~\citep{scutari2010} (\textsc{Asia}\citep{asia}, \textsc{Sachs}~\citep{sachs2005causal}, \textsc{Insurance}~\citep{binder1997}, \textsc{Alarm}~\citep{alarm}, and \textsc{Hailfinder}~\citep{abramson1996}), each with a ground-truth DAG, and two real-world chest X-ray datasets: 
\textsc{SIIM-Pneumothorax}~\citep{you2023cxr}, with concept annotations generated using a medical 
CLIP model following established practices~\citep{oikarinen2023,de2025}, and \textsc{CheXpert}~\citep{irvinCheXpert2019}, whose 
labels are extracted from radiology reports. For each dataset, one variable is treated as the task and the remaining as concepts.  To model heterogeneous 
settings, clients are assigned different data distributions and observe partial subsets of concepts.
In the centralized setting, graph-based CMs use the ground-truth DAG or a proxy, 
whereas in federated and localized ones clients access only their local subgraphs and concepts. 
To model imperfect knowledge (or data poisoning), client graphs are perturbed by randomly adding, removing, or reversing edges (see Apps.~\ref{app:datasets} and \ref{app:datapreparation}). When computing accuracy, variables that a model cannot predict under a given training regime (e.g., localized) are assigned chance-level accuracy, defined as the expected accuracy of a uniform random predictor over the corresponding space.



\subsection{Task and Concept Coverage}\label{sec:taskconcept_acc}
We evaluate \emph{task accuracy} and \emph{concept coverage}, i.e., the fraction of task-relevant concepts predicted by the model. 
\emph{Concept accuracy} is reported in App.~\ref{app:concept_acc}.
\begin{figure*}
   \centering   \includegraphics[width=0.91\textwidth]{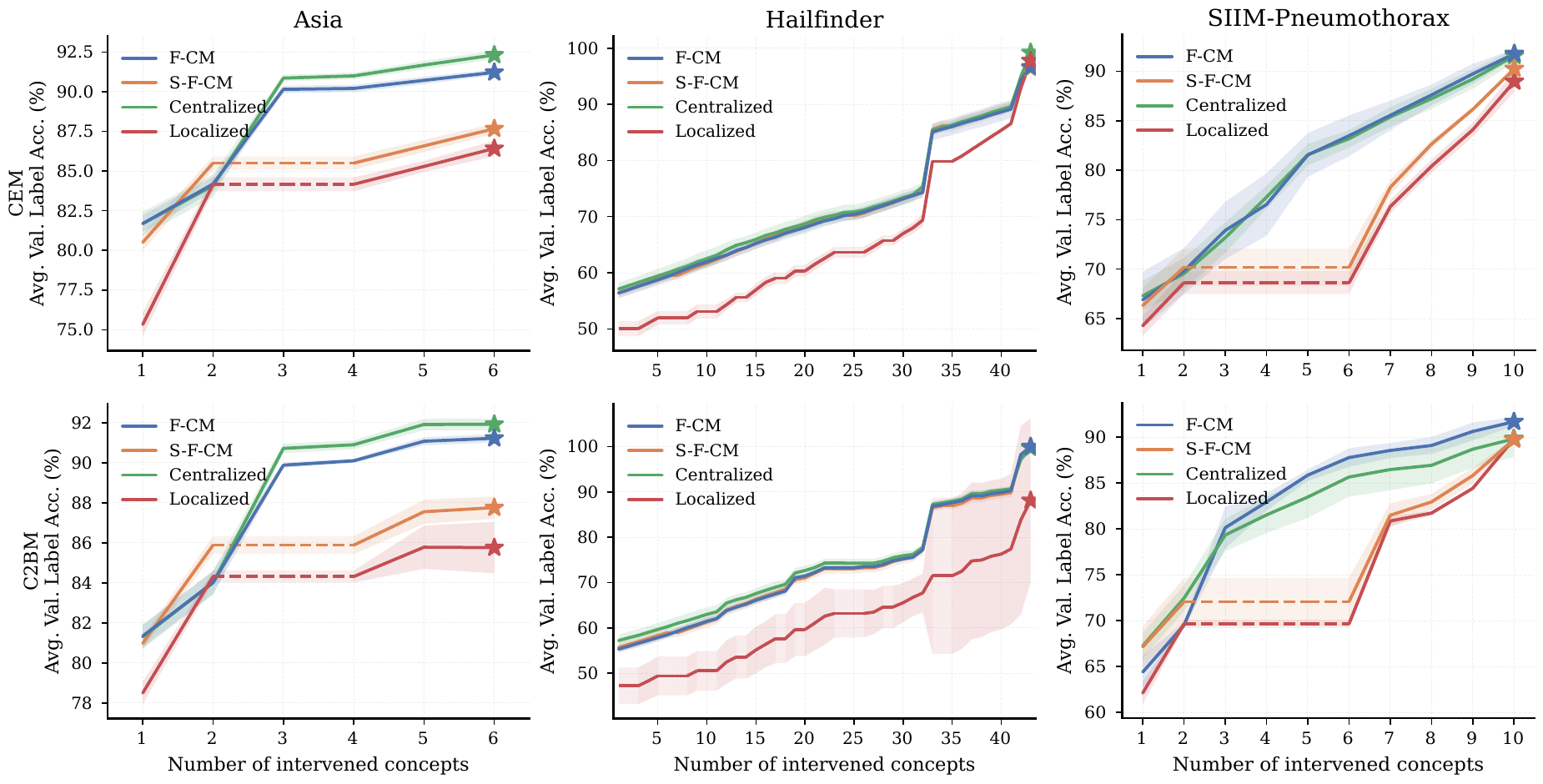}   
   \caption{\textbf{Label accuracy} (\%) on downstream variables (including the task) following interventions on concepts at deeper levels of the ground-truth graph hierarchy. Dashed lines denote interventions that are \textbf{not} possible for a given modality due to the absence of predictions for the 
   concept.}
  \label{fig:int}
\end{figure*}


\textbf{Federated aggregation is necessary under statistical heterogeneity.}
Comparing S-F-CMs to Localized isolates the effect of cross-client aggregation when the concept space is fixed. Across benchmarks, S-F-CMs improve task accuracy and coverage over localized training, showing that aggregating supervision and, for graph-based models, structure is essential to reconcile client heterogeneity. 

\textbf{Aggregation alone is insufficient under temporal non-stationarity; F-CMs close the gap.}
Under temporal non-stationarity, late-arriving clients introduce additional supervision and, for graph-based models, new dependencies, making a fixed architecture increasingly mismatched to the evolving concept space. This is reflected by the gap between S-F-CMs and both Centralized and F-CMs, and is consistent with the weaker task accuracy of the external static FL baselines FedCBM and FCL. By extending the global concept set and updating the dependency graph, F-CMs adapt the shared architecture and remain close to the centralized bound in task accuracy, while attaining near-$100\%$ coverage of task-relevant concepts. 

Overall, these results show that realistic federations require 
(i) graph aggregation to handle heterogeneity and (ii) architecture adaptation to sustain performances as the federation evolves.


\subsection{Concept-Level Interventions}
\label{sec:intervention_acc}
In this section, we evaluate responsiveness to \emph{ground-truth concept interventions} by replacing an increasing number of predicted concepts with their ground-truth values, mimicking human corrections at deployment. To ensure a consistent  policy, interventions are sampled within progressively deeper levels in the ground-truth DAG hierarchy. 
Fig.~\ref{fig:int} reports the average accuracy over the remaining predicted variables (concepts and task) after each intervention, which we refer to as \emph{label accuracy}. 
For clearer visualization we fix the client subgraphs across three seeds, isolating intervention effects from variability in the federated partition. Additional results are reported in App.~\ref{app:add_exps}.

\textbf{F-CMs enable interventions on locally unannotated concepts.}
Dashed segments mark interventions that are impossible because the corresponding concepts are not predicted by a given regime. This is the case for both Localized and S-F-CMs: localized models cannot predict concepts outside a client’s subgraph, while static federation cannot incorporate concepts introduced by late clients because it fixes the concept space early. By combining cross-client aggregation with dynamic expansion and rewiring, F-CMs predict the broadest set of concepts and keep benefiting from further interventions, yielding higher label accuracy. 
The same trend holds for task accuracy alone (App.~\ref{app:int_task_acc}).\footnote{The Localized regime shows higher variance because, with limited concept supervision, the model may encode information about unobserved concepts into the predicted ones; interventions overwrite these predictions, leading to less stable gains.}

\textbf{F-CMs achieve interventional responsiveness comparable to centralized.}
Across datasets and architectures, F-CMs achieve accuracy comparable to centralized setting throughout the intervention trajectory, showing that F-CMs preserve gains from concept corrections as pooled training does.

\subsection{Efficient Adaptation in Non-Stationary Federations}
\label{sec:exp_trends}
\begin{figure*}[t]
   \centering   \includegraphics[width=0.92\linewidth]{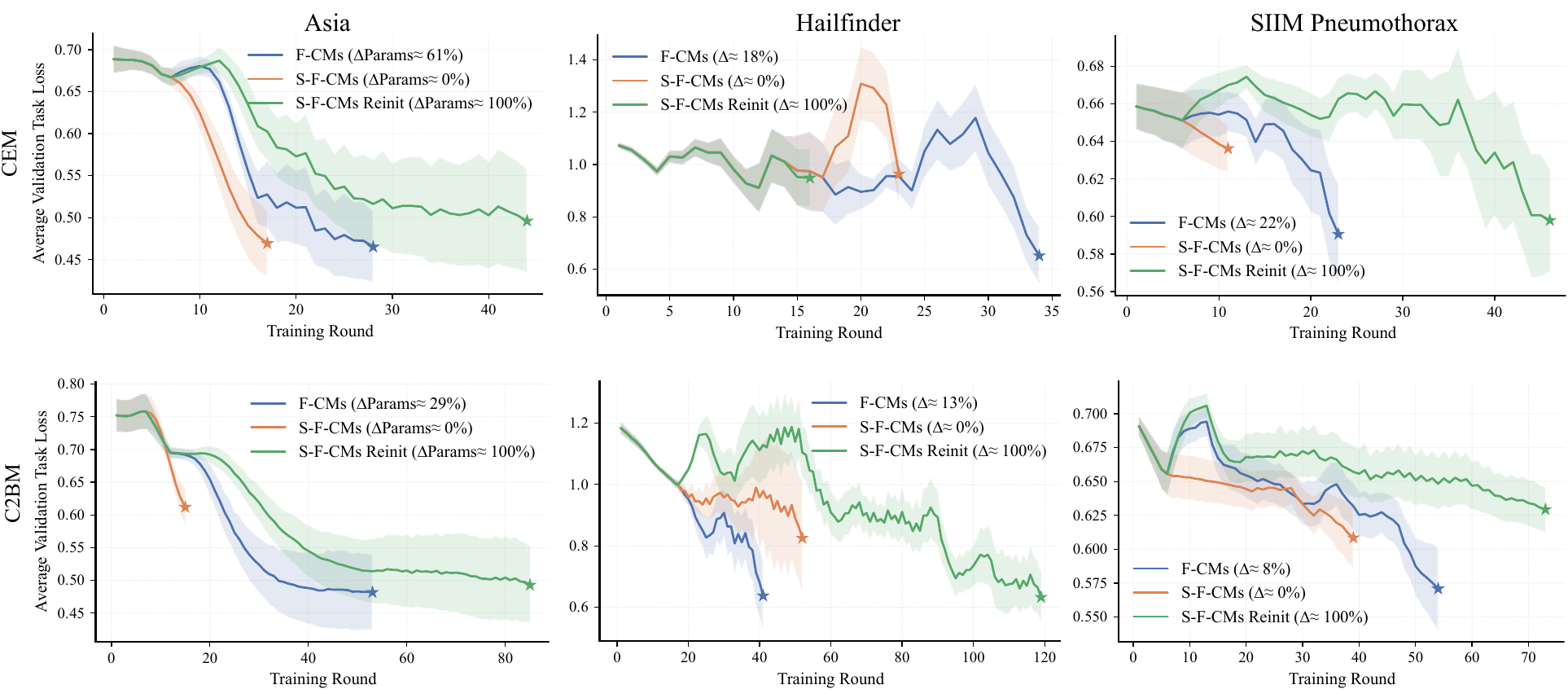}   \caption{\textbf{Convergence under temporal non-stationarity.} Average 
   task loss across rounds for C$^2$BM and CEM,
   where new clients introduce unseen concepts. F-CMs (ours) adapt the architecture 
   and converge faster and to lower loss than S-F-CMs (no adaptation) and S-F-CMs Reinit (
   retraining). Curves are truncated at the best loss-round. The legend reports the fraction of modified parameters.}\label{fig:convergence}
\end{figure*}
We evaluate how F-CMs behave in \emph{non-stationary} federations where new clients join or leave during training and introduce previously unseen concepts, requiring the shared model to expand. 
Fig.\ref{fig:convergence} reports the average validation task loss over communication rounds on the same datasets and instantiations of the previous experiment.
We compare: \textbf{F-CMs} (ours), \textbf{S-F-CMs} (without architecture adaptation), and \textbf{S-F-CMs Reinit} (re-training from scratch whenever new clients join).

\textbf{Across all datasets and both architectures, F-CMs converge faster and to lower task loss.} F-CMs outperform S-F-CMs and avoid the slow recovery of the reinitialization baseline, which erases knowledge accumulated from earlier clients.
This translates into greater \emph{adaptation efficiency} in evolving federations: F-CMs modify only the modules affected by newly introduced concepts or dependencies, yielding sparse parameter changes (e.g., $8\%$--$61\%$ across settings), compared to $0\%$ for S-F-CMs (which cannot incorporate new concepts) and $100\%$ for S-F-CMs Reinit.
Full results on parameter changes and additional plots for the remaining instantiations are reported in App.~\ref{app:params_changed} and~\ref{app:convergence_cbm_cgm}.
These results highlight that modular architecture adaptation yields a more convergence-efficient strategy for continual federation growth, preserving previously learned knowledge while integrating new supervision.




\section{Related works} 
We build on the literature on concept-based models (CMs), which leverage human-interpretable variables for prediction and interpretation. In our experiments, we consider representative architectures from this line of work, including CBM~\cite{koh2020}, CEM~\cite{espinosa2022}, CGM~\cite{dominici2024}, and C$^2$BM~\cite{de2025}. Several other CM variants have also been proposed, e.g.,~\cite{poeta2023,desantis2025,barbiero2023}. Our framework is not tied to the architectures evaluated here and could be extended to accommodate other CM variants.

A related line of work addresses the challenge of obtaining concept annotations. Specifically, several CM extensions address the challenge of obtaining concept annotations by relaxing the need for fully observed concept supervision, e.g., by handling missing or noisy labels and combining supervised with unsupervised concept representations
\cite{oikarinen2023, sawadCBMUnsup2022, penalozaConceptNoise2025, hu2025semi, liu2025hybrid}. However, they typically rely on domain-specific pre-trained models to infer concepts. In contrast, our approach leverages ground-truth concept annotations, even when distributed across multiple entities.


Interpretability in FL remains relatively underexplored. Recent works~\cite{yangFedCBM2024, shenFCL2024, zhangLRXFL2024} take initial steps toward combining concepts with FL, but rely on unsupervised concept extraction~\citep{kim2018interpretability, wang2023learning} rather than concept annotations. Consequently, they cannot fully exploit FL to aggregate existing concept-level supervision across clients. They also assume a \emph{static} federation, and therefore do not address \emph{temporal non-stationarity}, common in real-world FL, where new clients may introduce previously unseen concepts and dependencies, requiring the shared concept space 
and architecture 
to expand over time.

A related direction is \emph{federated causal discovery}, which studies how to infer causal relationships from decentralized data \cite{ng2022towards,huang2023, li24}. These methods typically assume a fixed, or pre-specified union, variable set across clients, which is restrictive in settings with evolving concept sets. Moreover, our goal is a model-agnostic graph aggregation mechanism rather than a method tied specifically to causal discovery. Beyond this, \emph{graph federated learning} \cite{zhang2021, he2021, liu2024,kim2025} focuses on training graph neural networks via federated optimization, rather than learning evolving interpretable concept-based architectures. 

\section{Conclusions} \label{sec:conclusion_limitation}
We introduced F-CMs, a novel methodology for training concept-based models in realistic federated settings under heterogeneous, partial concept supervision and temporal non-stationarity (dynamic participation and data drift). F-CMs aggregate concept supervision and, when available, concept dependencies across clients throughout training, while adapting the shared model modularly as the concept space evolves. Experiments show that F-CMs preserve task performance while improving concept coverage and intervention effectiveness over local and fixed federated architecture baselines. Importantly, F-CMs enable interpretable inference even for concepts that are locally unannotated.
 
\textbf{Limitations and future work.}
While F-CMs can reuse existing CM architectures, these models were not designed for evolving concept spaces. A promising direction is to develop federation-aware CMs with greater modularity, so that concept and dependency updates affect fewer parameters and enable more efficient adaptation under temporal non-stationarity.
Long-running federations
could further benefit from principled support-aware retention mechanisms,
including concept down-weighting, temporal decay or pruning to limit storage and computation when
accumulated concepts remain unsupported for extended periods or become
irrelevant, while preserving still useful concepts.

Another important direction is to strengthen the privacy analysis by systematically characterizing information leakage. Although bringing interpretable models into FL is important, it may also introduce new attack surfaces that remain largely unexplored, motivating defenses tailored to concept-based federated models.
Beyond this, the modular structure of F-CMs could be exploited to assess client agreement at the concept-module level and develop reliability-aware or robust aggregation strategies against semantic inconsistencies and potentially malicious updates.





\begin{ack}
This research was funded by the Swiss National Science Foundation (SNSF), and the European Union’s Horizon Europe programme through the projects SmartCHANGE (No. 101080965), and XAI-PAC (No. Z00P2\_216405). AC and JS
acknowledge support from FFF of the University of Liechtenstein grant lbs\_24\_08. PB acknowledges support from the SNSF Postdoctoral Fellowships IMAGINE (No. 224226). GDF acknowledges support by the SNSF through the grant 205121\_197242 for the project PROSELF 
and the Hasler Foundation under the Project ID: 2024-05-15-70. The authors declare no competing interests.
\end{ack}






\newpage
\bibliographystyle{unsrtnat} 
\bibliography{references} 

\newpage
\appendix
\onecolumn

\appendixpart
\parttoc

\newpage

\section{Dataset details}\label{app:datasets}

\subsection{Bayesian Network datasets (bnlearn)}
We use five benchmark \emph{Bayesian networks} from the \texttt{bnlearn} repository~\citep{scutari2010}. A Bayesian network (BN)~\citep{jensen2007} is a probabilistic graphical model specified by a DAG and a set of conditional probability distributions (CPDs). The DAG encodes the dependency structure among random variables, while each CPD specifies the distribution of a variable conditional on its parents in the graph. The selected BNs span diverse application domains and vary in graph size and structural complexity.
\begin{itemize}[topsep=1pt,itemsep=1pt,parsep=0pt,partopsep=0pt,leftmargin=1.7em]
    \item \textit{\textsc{Asia}:} A small network for lung disease diagnosis with variables representing patient symptoms and conditions.
    \item \textit{\textsc{Sachs}:} A widely used network capturing dependencies among protein and phospholipid expression levels in human cells.
    \item \textit{\textsc{Alarm}:} A medical network designed to provide an alarm message system for patient monitoring.
    \item \textit{\textsc{Insurance}:} A network for car insurance risk assessment with variables covering driver characteristics, vehicle properties, and accident-related factors.
    \item \textit{\textsc{Hailfinder}:} A weather forecasting network with variables for predicting severe weather conditions in northeastern Colorado, including atmospheric measurements and storm indicators.
\end{itemize}

\begin{table}[h]
\centering
\caption{Summary of bnlearn Bayesian networks used in experiments.}
\label{tab:dataset_summary}
\begin{tabular}{lccl}
\toprule
\textbf{Bayesian network} & \textbf{Nodes} & \textbf{Edges} &  \textbf{Domain} \\
\midrule
\textsc{Asia} & 8 & 8 & Medical diagnosis \\
\textsc{Sachs} & 11 & 17 & Protein signaling \\
\textsc{Alarm} & 37 & 46  & Medical monitoring \\
\textsc{Insurance} & 27 & 52  & Insurance risk \\
\textsc{Hailfinder} & 56 & 66 &  Weather prediction \\
\bottomrule
\end{tabular}
\end{table}

For each network, we generate a dataset with a network-specific number of samples, selected via preliminary experiments to ensure stable performance: 15{,}000 for \textsc{Asia}, 15{,}000 for \textsc{Sachs}, 10{,}000 for \textsc{Alarm}, 20{,}000 for \textsc{Insurance}, and 20{,}000 for \textsc{Hailfinder}.
Each dataset is split into training, validation, and test sets using a 70\%–10\%–20\% split.
For each network, we generate a dataset with a network-specific sample size, selected through preliminary experiments to ensure stable performance: 15{,}000 for \textsc{Asia}, 15{,}000 for \textsc{Sachs}, 10{,}000 for \textsc{Alarm}, 20{,}000 for \textsc{Insurance}, and 20{,}000 for \textsc{Hailfinder}.
Each dataset is split into training, validation, and test sets using a 70\%–10\%–20\% split.
We interpret the networks causally and treat their DAGs as ground-truth causal graphs. Each directed edge represents a direct causal relation from a parent to a child variable, and the DAG defines the reference causal structure for evaluation. The nodes in each network are treated either as concepts $\{C_j\}_j$ or as the task variable $Y$. Thus, the generated datasets provide annotations for semantic variables, but no corresponding raw inputs $x_i$. To construct inputs for concept-based models, concept values are flattened and passed through an autoencoder trained with MSE loss, consisting of two encoder and two decoder layers. The latent dimensionality scales with graph size (\textsc{Asia}: 32, \textsc{Sachs}: 32, \textsc{Insurance}: 64, \textsc{Alarm}: 64, \textsc{Hailfinder}: 64). The resulting embeddings are perturbed by mixing 50\% of the encoded signal with 50\% Gaussian noise, and the final inputs are standardized. These transformations yield non-trivial input representations, requiring models to recover the underlying concepts rather than trivially accessing them. The selected tasks are: \emph{dysp} (\textsc{Asia}), \emph{Akt} (\textsc{Sachs}), \emph{PropCost} (\textsc{Insurance}), \emph{BP} (\textsc{Alarm}), and \emph{R5Fcst} (\textsc{Hailfinder}).

\subsection{\textsc{SIIM-Pneumothorax}}
The \textsc{SIIM-Pneumothorax} dataset is derived from the publicly available chest radiograph collection released by the National Institutes of Health (NIH). It consists of chest X-ray images annotated with binary labels indicating the presence or absence of Pneumothorax. In our experiments, we use the training annotations from the SIIM-ACR Pneumothorax Segmentation challenge on Kaggle and the corresponding training images from the SIIM PNG image dataset\footnote{\url{https://www.kaggle.com/competitions/siim-acr-pneumothorax-segmentation}; \url{https://www.kaggle.com/datasets/abhishek/siim-png-images}}. 
We then split the resulting dataset into training, validation, and test sets using standard $70\%$--$10\%$--$20\%$ partition.

As the dataset does not provide concept-level annotations, we generate concept labels automatically. Specifically, we use a pretrained CXR-CLIP model \cite{you2023cxr} following a procedure similar to \citet{oikarinen2023}: candidate concepts are generated once and filtered using the training data, then image-text cosine similarities are computed between each radiograph and the generated concepts to generate annotations. These similarity scores are further binarized independently for each concept using two-cluster k-means, to obtain final binary concept annotations used in our experiments.

As encoder, we use a ResNet-18 followed by a trainable MLP head that maps image representations to the latent space used by the concept and task modules.

\subsection{\textsc{CheXpert}}\label{sec:chexpert}
We use the \textsc{CheXpert} dataset of chest X-ray images annotated with expert-labeled radiological observations \cite{irvinCheXpert2019} hosted on Kaggle\footnote{\url{https://www.kaggle.com/datasets/ashery/chexpert}}. In our setup, the 13 observation labels are used as concept annotations, namely \emph{Enlarged Cardiomediastinum}, \emph{Cardiomegaly}, \emph{Lung Opacity}, \emph{Lung Lesion}, \emph{Edema}, \emph{Consolidation}, \emph{Pneumonia}, \emph{Atelectasis}, \emph{Pneumothorax}, \emph{Pleural Effusion}, \emph{Pleural Other}, \emph{Fracture}, and \emph{Support Devices}. The downstream binary task is derived from the \textit{No Finding} label, which we invert so that the target indicates the presence of any abnormal finding.

For preprocessing, 
we first merge the provided training and validation metadata files, remove duplicate patients by keeping only the last available study for each subject, and replace missing annotations with $0$ and uncertain concept annotations (encoded as $-1$) with $1$. We  further filter the dataset by acquisition view, keeping only frontal images. We then shuffle the resulting dataset and create custom training, validation, and test splits using a $70\%$--$10\%$--$20\%$ partition.
Following an image preprocessing similar to \citet{laguna2024}, each radiograph is center-cropped to a square region, resized to $224 \times 224$, histogram-equalized, and converted to grayscale for compatibility with the image encoder. During training, we additionally apply random affine transformations and random horizontal flipping; for validation and test images, only resizing and tensor conversion are used.

As encoder, we again use a ResNet-18, as in \citet{laguna2024}, followed by a trainable MLP head.

To obtain a dataset size comparable to the other benchmarks, we further subsample an equal number of examples per task class. Importantly, we preserve the federated heterogeneity by assigning different subsets of concept supervision to different clients, as in the other datasets.

\subsection{\textsc{CheXpert Plus}}
\label{sec:chexpert_plus}

We use \textsc{CheXpert Plus}, a multimodal extension of \textsc{CheXpert} that provides radiology reports paired with chest X-ray images~\cite{chexpertPlus}. For the images, we rely on the version of the \textsc{CheXpert} dataset hosted on Kaggle, which we pair with the additional reports from \textsc{CheXpert Plus} (see Section~\ref{sec:chexpert}). To increase the number of samples with both modalities, we do not filter images by view and we retain duplicate studies from the same patients.  However, we perform a patient-level split: all images and reports associated with a given patient are assigned to only one of the training, validation, or test sets. We then apply the same image preprocessing as in the original \textsc{CheXpert} setup with a usual split of $70\%$--$10\%$--$20\%$ for training, validation and test.

For the text modality, each radiology report is tokenized and the resulting sequence is used as input to the text encoder (see App.~\ref{app:multimodal}).

The resulting dataset preserves the same semantic prediction targets as \textsc{CheXpert}, with each instance associated with both an image and a text input.

As encoders, we use a ResNet-18 for images and the pre-trained CXR-BERT specialized model, a chest X-ray domain-specific language model~\cite{CXRBERT}, for text. The outputs of both encoders are then passed through trainable MLP heads.

\section{Baseline details}
\label{app:baselines_details}
All the baselines considered share the same input encoder. They differ only in the parameterization of the concept predictors and the task decoder, and in whether they use a bipartite graph (CBM/CEM) or a more general DAG $\mathcal{G}$ connecting concepts and task. 

\paragraph{Concept Bottleneck Models (CBMs) \cite{koh2020}.}
CBMs use a bipartite graph. Specifically,
each concept is predicted independently from an encoder representation $z$:
\[
\hat{c}_j = h_{\phi_j}(z), \text{ for every } j
\]
where $h_{\phi_j}$ is an MLP head outputting either logits/probabilities (categorical concepts) or real values (continuous concepts).
The task decoder operates only on the predicted concepts $\hat{\textbf{c}} = \{\hat{c}_1,...,\hat{c}_m\}$,
\[
\hat{y} = f_{\psi}(\hat{\textbf{c}}),
\]
implemented as an MLP over concatenation of all concept predictions (no direct dependence on $z$).


\paragraph{Concept Embedding Models (CEMs) \cite{espinosa2022}.}
CEMs share the same bipartite construction as CBMs, but represent each concept with an embedding rather than a scalar. For each concept $C_j$, the model first produces two concept-specific embeddings from the encoder representation $z$: 
\[
\hat{c}^{+}_j = \varphi^{+}_{\phi_j}(z)\in\mathbb{R}^r, 
\qquad
\hat{c}^{-}_j = \varphi^{-}_{\phi_j}(z)\in\mathbb{R}^r,
\]
where $\hat{c}^{+}_j$ encodes the \emph{active} state and $\hat{c}^{-}_j$ the \emph{inactive} state of concept $C_j$. Then, a scoring function maps their joint space to a concept activity probability,
\[
\hat{p}_j \;=\; s_\phi([\hat{c}^{+}_j,\hat{c}^{-}_j]^{T}) \;=\; \sigma\!\big(W_s[\hat{c}^{+}_j,\hat{c}^{-}_j]^{T} + b_s\big)\in[0,1],
\]
and the final concept embedding fed to the task is the corresponding mixture:
\[
\tilde{c}_j \;=\; \hat{p}_j\,\hat{c}^{+}_j \;+\; (1-\hat{p}_j)\,\hat{c}^{-}_j \;\in\mathbb{R}^r.
\]
The task decoder then operates on the concatenation of mixed concept embeddings $\hat{\mathbf{c}}$,
\[
\hat{y} \;=\; f_{\psi}(\hat{\textbf{c}}),
\]
implemented as an (interpretable) linear layer or a small MLP. 

For both CBMs and CEMs, changes to the available concept set are handled by adapting the task decoder to the updated input dimensionality and reinitializing the affected parameters. Concept modules associated with existing concepts are reused, while modules for newly added concepts are created and initialized. 

\paragraph{Concept Graph Models (CGM) \cite{dominici2024}.}
We implement CGMs using the aggregated DAG produced by our methodology. For each concept node $C_j$, we predict
\[
\hat{c}_j = h_{\phi,j}\!\big(\hat{\textbf{c}}_{\mathrm{PA}(C_j)}, z\big),
\]
where $h_{\phi,j}$ is an MLP that takes as input the concatenation of the embeddings of the parent concepts of $C_j$ and the encoder representation $z$, following a formulation similar to CEM. The task prediction is computed analogously,
\[
\hat{y} = f_{\psi}\!\big(\hat{\textbf{c}}_{\mathrm{PA}(Y)}, z\big).
\]
This factorization ensures that changes to upstream concepts propagate to downstream concepts and the task through the parent-conditioned modules.

Under dynamic architectural adaptation, changes in $\mathrm{PA}(C_j)$ or $\mathrm{PA}(Y)$ are handled by adapting $h_{\phi,j}$ or the task decoder, respectively, resizing the input layer when necessary and reinitializing the affected parameters. Concept/task modules not affected by any change are reused while modules for newly added concepts are created
and initialized.

\paragraph{Causally Reliable Concept Bottleneck Models (C$^2$BMs) \cite{de2025}.}
C$^2$BMs also use our aggregated DAG, but explicitly separate \emph{exogenous} information extracted from the input from the \emph{structural} mechanism mapping parent variables to each node.

Concretely, each node $V$  in the DAG is equipped with a CEM-style module that outputs an embedding $u_V$ representing the node-specific exogenous signal extracted from the input. This exogenous representation is used (a) to predict \emph{root} nodes directly via a lightweight output head, and (b) to \emph{parameterize} the parent-to-node mechanism for \emph{non-root} nodes through a hypernetwork. In particular, for non-root nodes we first generate node-specific coefficients from $u_V$,
\[
\theta_V = r_{\phi_V}(u_V),
\]
and then predict $V$ from its parents with a linear-in-parents map followed by the appropriate output nonlinearity:
\[
\hat{v}
= \sigma_V\!\Big(\langle \theta_V,\; \hat{\textbf{v}}_{\mathrm{PA}(V)} \rangle \Big)
\;\equiv\;
f_{\theta_V}\!\big(\hat{\textbf{v}}_{\mathrm{PA}(V)}\big),
\]
where $r_{\phi_V}$ is an MLP (hypernetwork) and $\sigma_V$ is chosen according to the variable type (e.g., sigmoid/softmax/identity). The task node $Y$ is treated identically as a node in the graph.

In C$^2$BM, a change in $\mathrm{PA}(V)$ is handled by resizing the output layer of the corresponding hypernetwork $r_{\phi_V}$ to generate the updated set of parent-specific coefficients and reinitializing its parameters accordingly. Concept encoders for newly added concepts are also created and initialized.

\paragraph{Federated Concept Learning (FCL) \cite{shenFCL2024}.}
FCL is a \emph{static} federated concept-learning method built on top of a Bottleneck Concept Learner (BotCL) \cite{wangBotCL2023}. On each client, BotCL extracts a set of unsupervised concept activations from the input and uses them for classification through a concept--class score matrix. In the federated setting, clients do not exchange concept predictors or structural relations; instead, each client uploads a matrix of \emph{concept--class co-occurrence scores}, encoding how strongly each learned concept contributes to each class. The server aggregates these scores across clients with a FedAvg-style procedure, after first trimming the highest and lowest values for each entry to mitigate the effect of malicious clients, and sends the aggregated score matrix back to the clients for the next round. Classification is then performed only from the learned concept activations and the aggregated concept--class scores. Since FCL assumes a single shared static concept space and communicates only concept--class scores, it is reported in our experiments as an external \emph{static FL} concept-learning baseline.

\paragraph{Federated Concept Bottleneck Models (FedCBM) \cite{yangFedCBM2024}.}
FedCBM is a \emph{static} federated concept-learning method based on a concept bottleneck architecture with \emph{pre-defined} human concepts. The method follows a two-stage pipeline. First, clients collaboratively train a shared \emph{concept bank} from local concept datasets, where each concept is represented by a \emph{Concept Activation Vector} (CAV), i.e., a linear hyperplane normal learned with a federated SVM-like optimization procedure. These shared CAVs define a common concept space across clients. Second, for a standard input $x$, each client extracts a latent representation $h(x)$ with a backbone encoder, projects it onto each CAV to obtain a vector of concept scores, and feeds this concept vector to a local linear classifier for prediction. Thus, FedCBM performs classification exclusively through a bottleneck of concept scores, while concepts are aligned across clients through the shared concept bank. Since FedCBM assumes a fixed shared concept vocabulary and does not model evolving concept dependencies or graph structure, we report it as an external \emph{static FL} concept-learning baseline.

\section{Experimental details}\label{app:exp_details}

\subsection{Details on data preparation}\label{app:datapreparation}
\paragraph{\textsc{SIIM-Pneumothorax} ground-truth graph.} 
The ground-truth graph among concepts and the task for \textsc{SIIM-Pneumothorax} is constructed following a procedure similar to the one provided in \citet{de2025}. First, the Greedy Equivalence Search (GES) algorithm~\cite{chickering2002}, a causal discovery method that infers causal structure from observational data, is applied to the concept and task annotations to produce an initial graph. This graph however may still contain unoriented edges. To resolve these ambiguities, we employ a large language model (LLM), specifically \texttt{GPT-4o}~\cite{hurst2024gpt}, enhanced with retrieval-augmented generation (RAG)~\cite{lewis2020}, which uses domain-specific medical literature to provide context. This process enriches the initial graph with background knowledge, yielding the final DAG.

\paragraph{Client local data.}
In our setting, each client $k$ holds a local dataset of inputs and semantic variable annotations,
\[
\mathcal{D}^{(k)} \coloneqq \{(x_i^{(k)}, \mathbf{v}_{i}^{(k)})\}_{i=1}^{n_k}.
\]
A client may also provide a local DAG $\mathcal{G}^{(k)}$ encoding relationships among its observed concepts and the task. 

To generate client data, we proceed as follows. Given a dataset $\mathcal{D}$ with full inputs, concepts, and task annotations (see App.~\ref{app:datasets}), we first construct a ground-truth causal DAG $\mathcal{G}$ over the available concepts and the task variable (see \emph{Proxy DAG construction and cycle resolution} in App.\ref{parag:DAG_construction_cycle_res}). 
Next, we generate $\mathcal{K}$ subgraphs $\{\mathcal{G}^{(k)}\}_{k=1}^{\mathcal{K}}$ of $\mathcal{G}$, where $\mathcal{K}$ denotes the total number of clients considered throughout training. 
Each $\mathcal{G}^{(k)}$ is defined as an induced subgraph over a subset of nodes, corresponding to a subset of concepts (and possibly the task variable). 
Specifically, each subgraph is built as follows: starting from the task node $Y$ (a leaf in the graph $\mathcal{G}$), we repeatedly move \emph{upstream} by following incoming edges and collect multiple ancestor paths up to the roots; we then take the union of these paths (equivalently, the induced ancestor tree of $Y$). To increase heterogeneity across clients, we additionally sample a set of random concept nodes that are not contained in the set of roots of $\mathcal{G}$ and, for each of them, collect one or more ancestor paths up to the roots; these paths are finally merged with the subgraph built from $Y$.
The task node $Y$ is then randomly excluded from a subset of clients to model missing task supervision. We ensure full concept coverage at $T$ for F-CMs, i.e., the union of the client subgraphs in the final federation considered in our experiments covers all concepts and the task variable of the original dataset.
To simulate imperfect client knowledge of concept relations, we also perturb each $\mathcal{G}^{(k)}$ by modifying a fraction $p$ of its edges (configuration parameter) for a percentage $r$ of clients. Specifically, for these clients, we randomly flip, remove or add edges to their DAGs while preserving acyclicity. The resulting aggregated $\mathcal{G}^{(k)}$ is the graph provided to the server.
Local datasets are then obtained by partitioning $\mathcal{D}$ into $\mathcal{K}$ disjoint subsets of approximately equal size. For client $k$, we retain only the supervision available in its subgraph, i.e., for the concepts and/or the task contained in the set of nodes of the subgraph. All other concept and task are treated as unobserved for that client.

\subsection{Training details}
\label{app:training_details}

\paragraph{Optimization and initialization.}
All models are trained using the Adam optimizer with LeakyReLU activations throughout. Experiments are repeated over different random seeds, affecting data partitioning, client assignment, and model initialization. Newly instantiated layers, including those added during architecture adaptation, use PyTorch's default random initialization, whose weights and biases are zero-centred in expectation and reproducible for each seed. The batch size is fixed to 512 across all datasets and learning regimes. All models use the loss defined in Eq.~\eqref{eq:loss}, with the concept--task trade-off parameter set to $\gamma=0.8$. Key hyperparameters—including learning rate, MLP hidden dimension, and dropout rate—are tuned via grid search on validation sets independently for each dataset and model. Complete hyperparameter configurations for all methods and training regimes are provided in the YAML configuration files released with the code. During training, we apply random concept interventions with probability $0.8$, following the intervention protocol of \citet{espinosa2022}, which encourages robustness to concept-level corrections.

\paragraph{Centralized and localized training.}
For centralized and localized baselines, models are trained for up to 200 epochs with early stopping based on the validation loss, using a patience of 10 epochs without improvement. To ensure a fair comparison with federated methods, we control for the total amount of training data. In the centralized setting, we train on the union of all data available across the federation over time, i.e., pooling the datasets of all clients that appear in $\mathcal{K}(t)$ throughout training. In the localized setting, each model is trained using data from a single client; the number of localized models equals the size of the largest client pool.

\paragraph{Federated training.}
Federated methods are trained for up to 200 communication rounds. At each round, clients perform 2 local epochs before sending updates to the server. Early stopping is applied based on the global validation loss, with a patience of 10 rounds without improvement. Unless otherwise specified (e.g., when varying the number of clients in Section~\ref{app:scaling_clients}), the number of participating clients per round is fixed to $|\mathcal{K}_t|=10$ across all experiments. All federated experiments are conducted under \emph{statistical heterogeneity} and \emph{temporal non-stationarity} (thus resulting in data drift). The client pool $\mathcal{K}(t)$ evolves over time, reaching a total of 20 clients. After an initial warm-up phase—at round $t=10$ for \textsc{Asia}, \textsc{Sachs}, and \textsc{SIIM}, and at $t=20$ for the remaining datasets—new clients join the federation. These late-arriving clients introduce additional concept supervision and, for graph-based instantiations, potentially novel dependency structures. Following client introduction, the client pool is balanced such that approximately 50\% of clients belong to the initial data distribution and 50\% to the new distribution. 

\paragraph{Parameter freezing under partial supervision.}
To enforce \emph{module-specific training} under partial concept/task supervision (Section~\ref{sec:fed_training}), each client updates only the modules it can directly supervise.
Concretely, concept (and task) labels may be missing at a client, which we represent via a missing-label mask (e.g., value $-1$). At the start of local training for client $k$, we identify the supervised variable set $\mathcal{V}^{(k)}$. We then \emph{freeze} all concept modules $h_{\phi,j}$ and task module $f_\psi$ of the variables not contained in $\mathcal{V}^{(k)}$.
This mechanism applies uniformly across instantiations (CBM/CEM/C$^2$BM/CGM): regardless of parameterization, any module without local supervision is held fixed during the client update, while only supervised modules (and the shared encoder) are optimized. As a result, local updates match Eq.~\eqref{eq:loss} by preventing gradients from modifying unsupervised components, and naturally supports module-wise aggregation since clients return updates only for the modules they trained.

\paragraph{Proxy DAG construction and cycle resolution.} \label{parag:DAG_construction_cycle_res}
When a ground-truth DAG is available, we use it directly as the reference global structure. Otherwise, following the general strategy of \citet{de2025}, we construct a \emph{proxy} causal graph from the available concepts and task variable by combining data-driven structure discovery with background knowledge. In particular, we first apply causal discovery to obtain an initial directed graph from the observational data, and then refine it with a retrieval-augmented LLM using domain knowledge sources to improve edge plausibility and orientation. This produces a single global proxy DAG, which is then used as the reference structure for centralized training and as the source from which client-specific partial views are derived. In the federated setting, each client is assumed to observe only an imperfect subgraph of this global structure. The server aggregates the client-provided edges into a single graph considering weights proportional to client dataset sizes and then applies a lightweight post-processing step to ensure acyclicity. Concretely, we run a standard depth-first search to detect back-edges and remove them; this procedure is repeated iteratively until no directed cycle remains. Since graph updates occur only when the shared concept set or structural proposals change, this cycle-resolution step is performed infrequently and adds negligible overhead.

\subsection{Licenses and hardware}
\label{app:code_licenses_hardware}
All experiments were implemented in Python using open-source libraries. The core implementation relies on PyTorch 2.12 \citep{Paszke2019PyTorchAI} (BSD-3-Clause) for model training, Flower 1.29 \citep{beutel2020flower} (Apache 2.0) for the federated learning pipeline, and PyTorch Lightning 2.6 (Apache 2.0) for training orchestration. We further used NumPy 2.4 (BSD-3-Clause), SciPy 1.16 (BSD-3-Clause), Pandas 2.3 \citep{pandas} (BSD-3-Clause), and scikit-learn 1.8 \citep{sklearn_api} (BSD-3-Clause) for data processing and evaluation, and Matplotlib 3.10 \citep{Hunter:2007} (Matplotlib license, BSD-compatible) for visualization. To support the concept and graph-based components of our methodology, we used bnlearn 0.12 (MIT), pgmpy 0.1 (MIT), and causal-learn 0.1.4 (MIT). For the medical imaging experiments, we additionally relied on TorchXRayVision 1.4 (Apache 2.0). The complete codebase and instructions for reproducing all experiments are available on GitHub\footnote{\href{https://github.com/francescoTheSantis/F-CMs}{https://github.com/francescoTheSantis/F-CMs}} under the MIT license.

We use publicly available datasets and pretrained models under their respective licenses and terms of use. Specifically, we use the \textsc{Asia}, \textsc{Sachs}, \textsc{Alarm}, \textsc{Insurance}, and \textsc{Hailfinder} networks from the bnlearn repository \cite{scutari2010}, \textsc{SIIM-Pneumothorax} under the applicable Kaggle competition terms, and \textsc{CheXpert} \cite{irvinCheXpert2019} and \textsc{CheXpert Plus} \cite{chexpertPlus} under their respective data-access and research-use agreements. Where an explicit open-source license is unavailable, we follow the dataset provider's access and usage conditions. The pretrained CXR-CLIP \cite{you2023cxr} and CXR-BERT \cite{CXRBERT} models are used according to their respective licenses. All datasets and pretrained models are credited through their original references in the paper.

Experiments were run on a server with four NVIDIA RTX A6000 GPUs (48 GB each), dual AMD EPYC 7513 32-core CPUs, and 512 GB RAM.


\section{Additional experiments}\label{app:add_exps}

\subsection{Concept accuracy}\label{app:concept_acc}
In Table~\ref{tab:concept_accuracy}, we report concept accuracy as the average accuracy across all concepts (no downstream task), evaluated in the same experimental setting as Section~\ref{sec:taskconcept_acc}. However, since concept coverage has already been reported in Section~\ref{sec:taskconcept_acc}, we report here, for each benchmark, model, and training regime, the concept accuracy computed only over the concepts predicted in each configuration.

Across benchmarks, models, and training regimes, concept accuracy is generally comparable, with minor differences observed in some localized and static baselines. This is expected, since for each configuration the metric is computed only over the concepts predicted by that configuration. Thus, differences in concept coverage and downstream task performance are not reflected in this measure.
\vspace{-10pt}
\begin{table}[H]
\centering
\scriptsize
\caption{Concept Accuracy.}
\label{tab:concept_accuracy}
\resizebox{0.8\textwidth}{!}{
\begin{tabular}{llccccccc}
\toprule
\textbf{Setting} & \textbf{Model} & \textbf{\textsc{Asia}} & \textbf{\textsc{Sachs}} & \textbf{\textsc{Alarm}} & \textbf{\textsc{Insurance}} & \textbf{\textsc{Hailfinder}} & \textbf{\textsc{SIIM Pn.}} &
\textbf{\textsc{CheXpert}}\\
\midrule
\multirow{5}{*}{Cent. (Upper Bound)}
 & OpaqNN   & 91.6 $\pm$ 0.4 & 75.3 $\pm$0.5 & 91.3 $\pm$ 0.1 & 79.5 $\pm$ 0.2 & 61.4 $\pm$ 0.3 &74.1 $\pm$ 0.2 & 89.4$\pm$0.2\\
 & CBM    & 91.5 $\pm$ 0.4 & 74.8  $\pm$ 0.7 & 91.0  $\pm$ 0.1 & 79.7  $\pm$ 0.2 & 69.7 $\pm$ 0.1 & 74.1  $\pm$ 0.1 & 89.1$\pm$0.3\\
 & CEM    & 91.3 $\pm$ 0.4 & 73.9 $\pm$ 0.5  &89.9 $\pm$ 0.3 & 77.1 $\pm$ 0.3 & 66.3 $\pm$ 0.3 & 72.8 $\pm$ 0.1 &
 89.0$\pm$0.1\\
 & CGM    & 91.0 $\pm$ 0.5 & 72.3  $\pm$ 0.5  &89.5  $\pm$ 0.2  & 72.7  $\pm$ 0.3 & 63.8  $\pm$ 0.3 & 66.8  $\pm$ 0.8 &
 89.2$\pm$0.1\\
 & C$^2$BM   & 91.3  $\pm$ 0.4 & 74.4  $\pm$ 0.5 &  90.1 $\pm$ 0.4 & 77.5  $\pm$ 0.2 & 65.4  $\pm$ 0.6 & 72.8  $\pm$ 0.2  & 89.1$\pm$0.1\\
\midrule
\multirow{5}{*}{Loc.}
 & OpaqNN & 87.0 $\pm$ 3.4 & 67.9 $\pm$ 0.5 & 85.3 $\pm$ 1.3 & 73.3 $\pm$ 1.2 & 53.0 $\pm$ 1.8  &   67.2 $\pm$ 0.6 &
 87.5$\pm$1.4\\
 & CBM    & 86.8  $\pm$  3.8 &  68.7  $\pm$  0.6 & 87.3  $\pm$  0.9 & 74.4  $\pm$  1.1 & 60.3  $\pm$  1.0 &  66.1 $\pm$ 0.5  &
 88.5$\pm$1.4\\
 & CEM    & 84.5 $\pm$  4.8 & 64.4 $\pm$  1.6 & 85.0 $\pm$  1.3 & 73.6 $\pm$  1.0 & 52.3 $\pm$  3.7   &   67.1 $\pm$  1.4 &
 87.5$\pm$1.4\\
 & CGM    & 85.3  $\pm$ 4.4 & 62.0  $\pm$ 2.0 & 81.7  $\pm$ 1.6 & 64.0  $\pm$ 0.8 & 50.2  $\pm$ 1.7  &  61.7  $\pm$ 0.8  &
 88.6$\pm$1.3\\
 & C$^2$BM   & 86.5 $\pm$ 3.8  &  63.2 $\pm$ 1.7  &  81.9 $\pm$ 1.2  &   70.6 $\pm$ 0.9  &  48.7 $\pm$ 3.2    &    64.5 $\pm$ 1.7 &
 87.5$\pm$1.4\\
\midrule
\multirow{2}{*}{\shortstack{External\\Static FL}}
 & FCL 
& N/A
& N/A
& N/A
& N/A
 & N/A
& N/A
  & N/A \\ 
 & FedCBM  & 65.9  $\pm$ 4.2 &  63.0  $\pm$ 0.7 &  62.7  $\pm$ 0.2 & 66.9  $\pm$ 0.6 & 60.2  $\pm$ 0.2 &  63.6  $\pm$ 0.3 & 86.0$\pm$5.0 \\ 
 \midrule
\multirow{5}{*}{S-F-CMs}
 & OpaqNN & 89.5 $\pm$ 2.8 & 73.8 $\pm$ 0.7 &  89.9 $\pm$ 0.4 & 76.8 $\pm$ 0.9 &  56.8 $\pm$ 0.7   &   62.4 $\pm$ 0.6 &
 86.2$\pm$1.1\\
 & CBM    & 89.8 $\pm$ 2.8 & 74.5 $\pm$ 0.5 &  90.8 $\pm$ 0.3 & 78.3 $\pm$ 1.0 & 67.0 $\pm$ 0.9   &   64.9 $\pm$ 1.4 &
 86.5$\pm$1.1\\
 & CEM    & 89.2 $\pm$ 2.9 & 74.0 $\pm$ 0.4 & 89.6 $\pm$ 0.8 & 76.2 $\pm$ 1.1 & 60.7 $\pm$ 1.2   &  64.8 $\pm$ 0.9 &
 86.0$\pm$1.2\\
 & CGM    & 88.7 $\pm$ 2.9 & 72.3 $\pm$ 0.5 & 88.6 $\pm$ 1.4 & 71.3 $\pm$ 1.0 &  56.6 $\pm$ 0.6   & 62.4 $\pm$ 0.9 &
 86.8$\pm$1.5\\
 & C$^2$BM   & 89.5 $\pm$ 2.8  & 73.2 $\pm$ 0.8  & 88.3 $\pm$ 1.3  & 75.2 $\pm$ 1.0  & 60.0 $\pm$ 0.7  & 64.4 $\pm$ 1.0 &
 85.7$\pm$1.2\\
\midrule
 \multirow{5}{*}{F-CMs}
 & OpaqNN & 91.0 $\pm$ 0.5 & 72.8 $\pm$ 0.6 &  90.5 $\pm$ 0.2 & 77.0 $\pm$ 0.4 & 56.2 $\pm$ 0.3 &  64.8 $\pm$ 0.5 &
 89.1$\pm$0.1\\
 & CBM    & 91.2 $\pm$ 0.4 & 74.2 $\pm$0.6 & 90.9 $\pm$ 0.3 & 78.9 $\pm$ 0.5 & 66.3 $\pm$ 0.3 &  68.7 $\pm$ 0.2 &
 89.2$\pm$0.1\\
 & CEM    & 90.2 $\pm$ 1.0  &  72.9 $\pm$ 0.8  &  90.2 $\pm$ 0.3  &  77.2 $\pm$ 0.4  &  60.4 $\pm$ 0.3  &  66.4 $\pm$ 0.8 &
 89.1$\pm$0.1\\
 & CGM    & 88.7 $\pm$ 0.7 & 68.8 $\pm$ 0.5 & 88.9 $\pm$ 0.6 & 72.3 $\pm$ 0.7 & 55.5 $\pm$ 0.4  &   61.9 $\pm$ 1.2 &
 89.2$\pm$0.2\\
 & C$^2$BM   & 90.7 $\pm$ 0.5 &  71.3 $\pm$ 0.8 & 90.0 $\pm$ 0.2 & 75.8 $\pm$ 0.9 & 59.6 $\pm$ 0.5 & 66.3 $\pm$ 1.6 &
 89.1$\pm$0.1\\
\bottomrule
\end{tabular}
}
\end{table}

\subsection{Concept-level interventions - further experiments}
In this section, we complement the concept-level intervention analysis presented in Sec.~\ref{sec:intervention_acc}. Specifically, under the same experimental setting, we report: 
(i) the \emph{task accuracy} under interventions, instead of label accuracy, of the same models, training regimes, and benchmarks evaluated in Sec.~\ref{sec:intervention_acc} (App.~\ref{app:int_task_acc}); 
(ii) the \emph{task} and \emph{label accuracy} under interventions across models and benchmarks different from the ones evaluated in Sec.~\ref{sec:intervention_acc}  (App.~\ref{app:int_diff_models_and_b}); and 
(iii) a \emph{comparison}, in terms of task and label accuracy under interventions, of the different \emph{F-CM instantiations} (App.~\ref{app:int_diffinstant}).

\vspace{-2pt}
\begin{figure}[H]
   \centering   \includegraphics[width=0.89\linewidth]{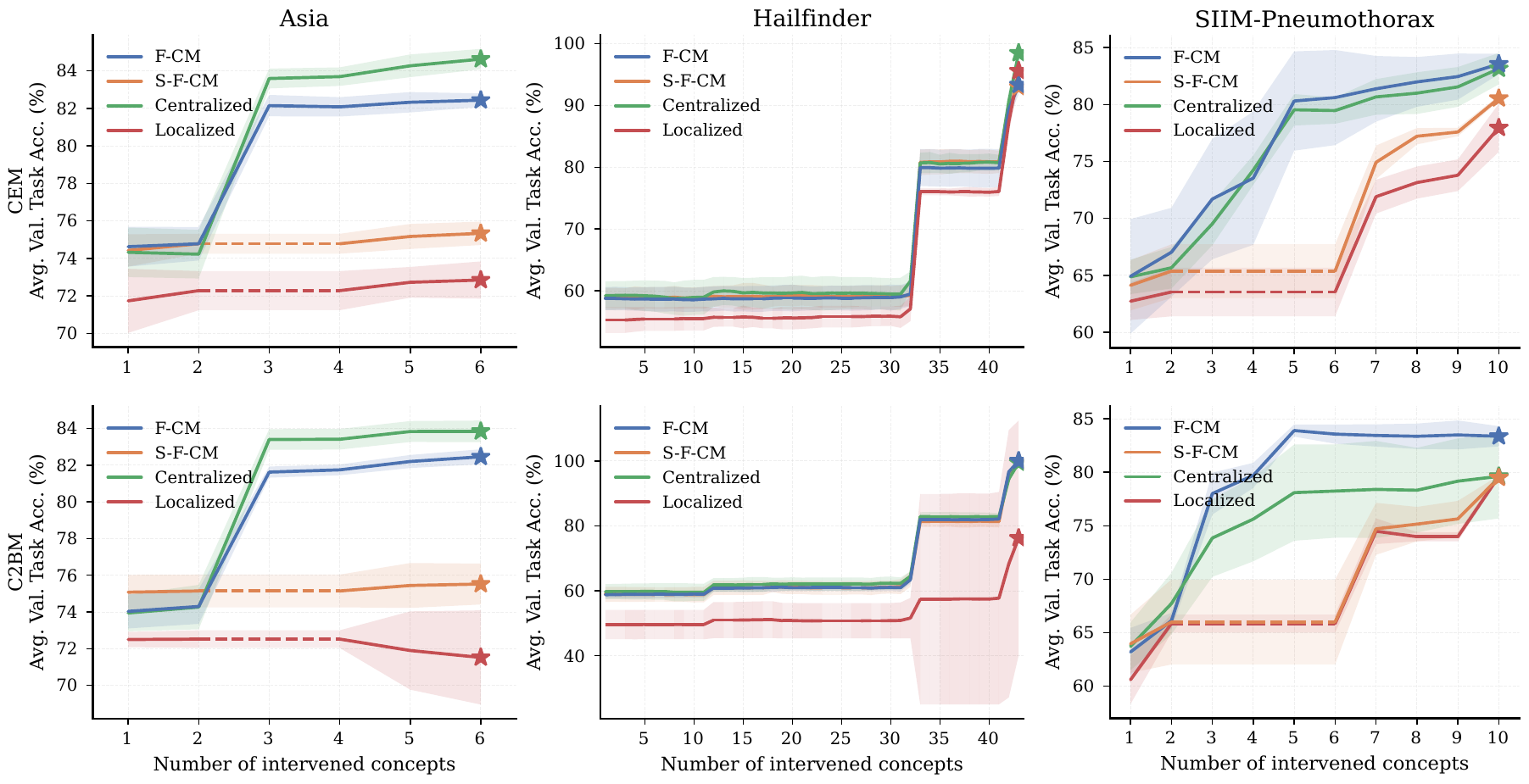}   \caption{ \textbf{Task accuracy} (\%) following interventions on concepts at increasing depths
of the ground-truth graph hierarchy. Dashed lines denote interventions that are not possible for a given training setting due to the absence of
predictions for the corresponding concept. 
}
\label{fig:task_acc}
\end{figure}

\subsubsection{Task accuracy}\label{app:int_task_acc}
In Fig.~\ref{fig:task_acc}, we report the \emph{task accuracy} under interventions for the same models, benchmarks, and training regimes considered in Sec.~\ref{sec:intervention_acc}. The results are consistent with the findings discussed in that section for label accuracy.

\subsubsection{Different Models and Benchmarks}\label{app:int_diff_models_and_b}
In Fig.~\ref{fig:labels_othermodels} and Fig.~\ref{fig:task_acc_othermodels}, we repeated the experiments of Sec.~\ref{sec:intervention_acc} and App.~\ref{app:int_task_acc} using different model architectures. Specifically, we report both label and task accuracy for the F-CM instantiations not previously shown (CBM and CGM), evaluated under the same experimental setup as in Sec. \ref{sec:intervention_acc}. The results are consistent with the findings discussed in that section.

For further reference, we reported the same results for the other real-world dataset considered, \textsc{CheXpert}, in Fig.~\ref{fig:labels_chexpert} and Fig.~\ref{fig:task_chexpert}. Again, results are consistent with the findings discussed in Sec.~\ref{sec:intervention_acc}.

\begin{figure}[H]
   \centering   \includegraphics[width=0.9\linewidth]{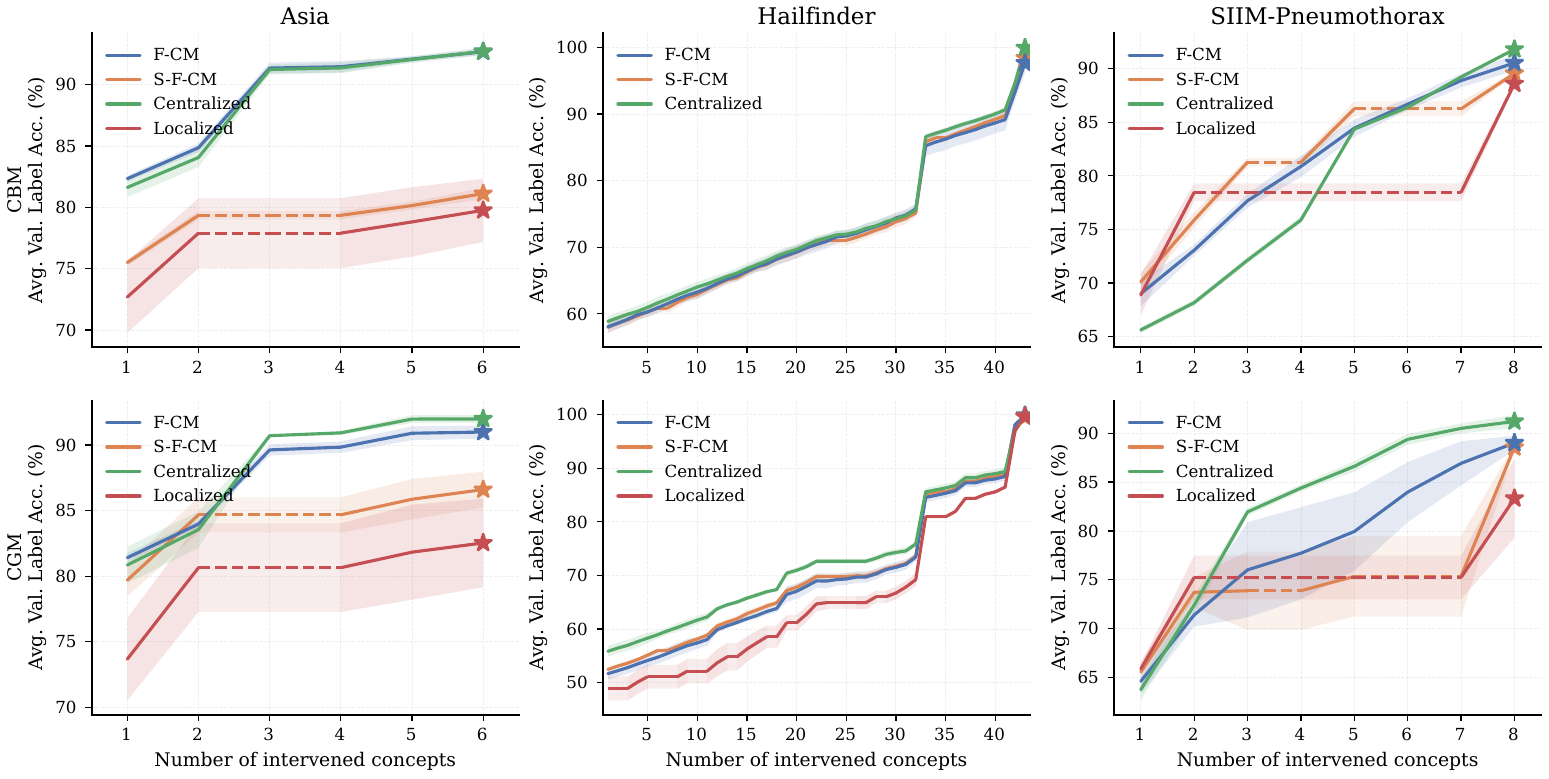}   \caption{\textbf{Label accuracy} (\%) on downstream variables (including the task) following interventions on concepts at increasing depths
of the ground-truth graph hierarchy. Dashed lines denote interventions that are not possible for a given training setting due to the absence of
predictions for the corresponding concept. 
}
\label{fig:labels_othermodels}
\end{figure}

\begin{figure}[H]
   \centering   \includegraphics[width=0.9\linewidth]{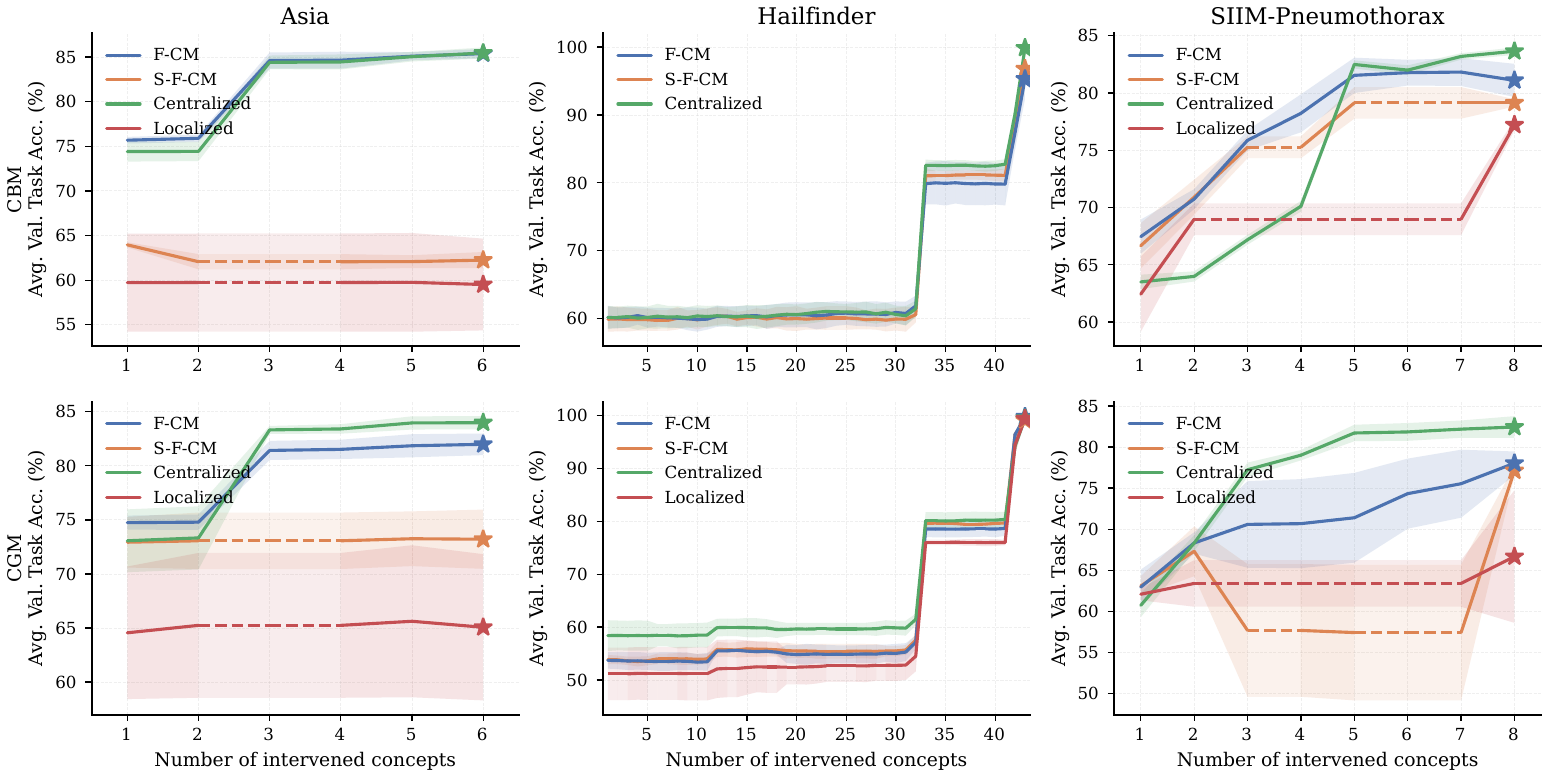}   \caption{\textbf{Task accuracy} (\%) following interventions on concepts at increasing depths of the ground-truth graph hierarchy. Dashed
lines denote interventions that are not possible for a given training setting due to the absence of predictions for the corresponding concept.
}
\label{fig:task_acc_othermodels}
\end{figure}

\begin{figure}
   \centering   \includegraphics[width=0.75\linewidth]{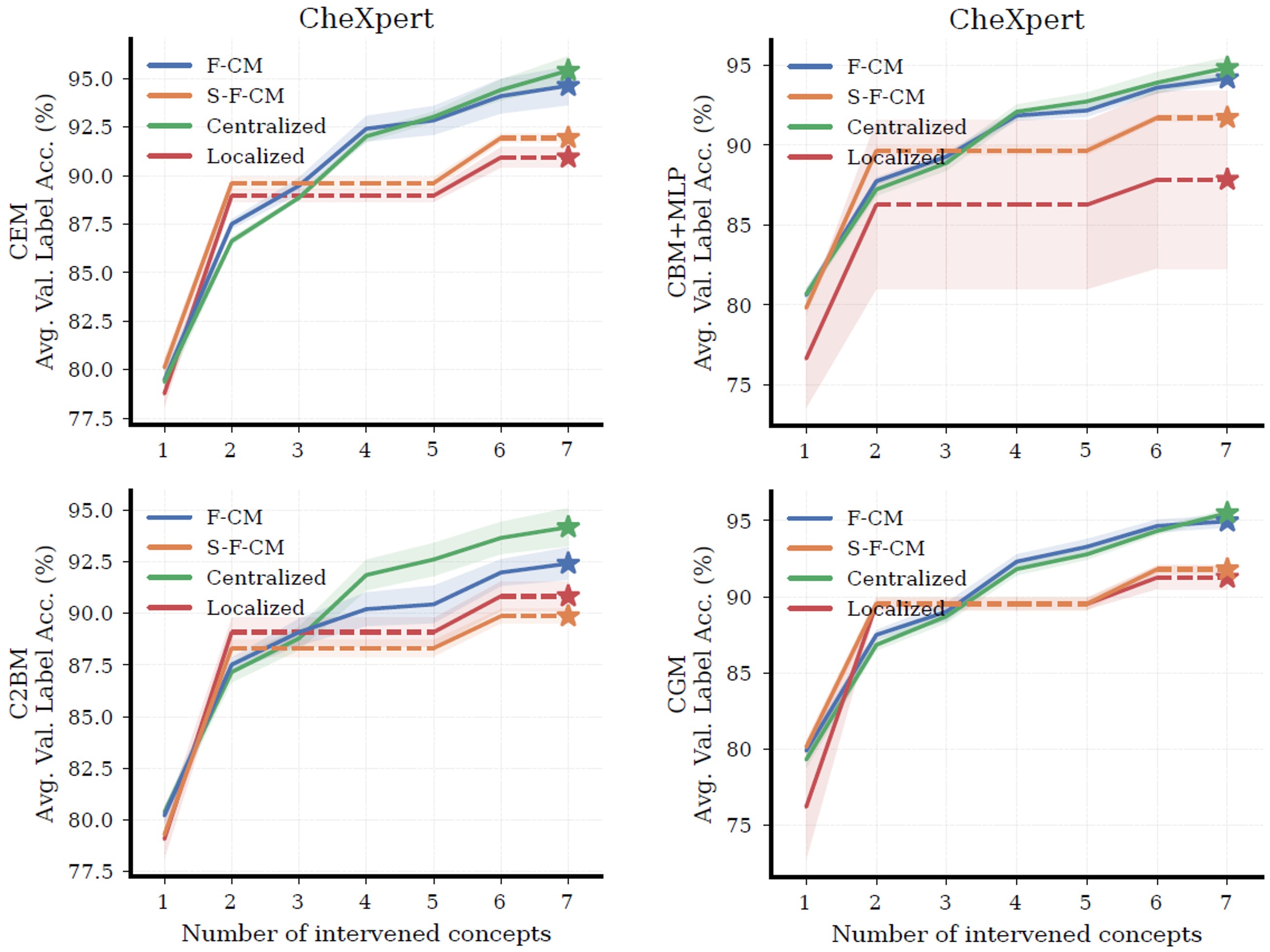}   \caption{\textbf{Label accuracy} (\%) on downstream variables (including the task) following interventions on concepts at increasing depths
of the ground-truth graph hierarchy. Dashed lines denote interventions that are not possible for a given training setting due to the absence of
predictions for the corresponding concept. 
}
\label{fig:labels_chexpert}
\end{figure}

\begin{figure}
   \centering   \includegraphics[width=0.75\linewidth]{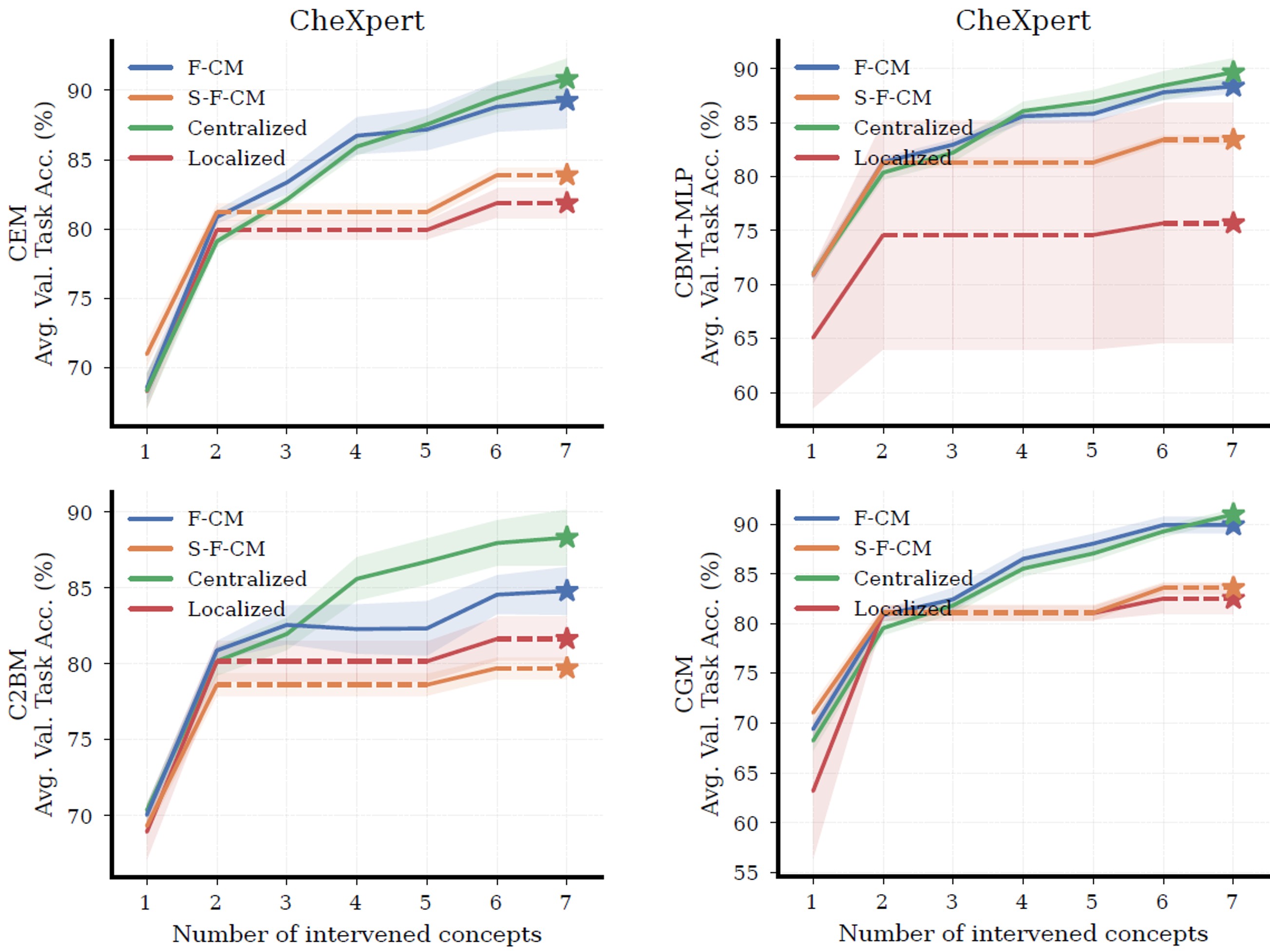}   \caption{\textbf{\textbf{Task accuracy}} (\%) following interventions on concepts at increasing depths of the ground-truth graph hierarchy. Dashed
lines denote interventions that are not possible for a given training setting due to the absence of predictions for the corresponding concept. 
}
\label{fig:task_chexpert}
\end{figure}

\subsubsection{Comparison of F-CMs Instantiations}\label{app:int_diffinstant}
We further compare how different CM instantiations benefit from the proposed federated protocol by evaluating performance under concept-level interventions.
Figures~\ref{fig:comparison_CMs_label}, ~\ref{fig:comparison_CMs_task} and~\ref{fig:comparison_chexpert} report the \emph{label accuracy} and \emph{task accuracy} obtained after intervening on concepts sampled at increasing depths of the ground-truth hierarchy (Sec.~\ref{sec:intervention_acc}). Overall, both metrics exhibit consistent trends: improvements in corrected concept predictions translate into accuracy gains, with larger gains typically obtained when intervening on upstream concepts.

On \textsc{Asia}, the bipartite CBM variant often shows the strongest response to interventions. We attribute this to the small scale and limited structural complexity of \textsc{Asia}: when the concept set is compact and most concepts have short causal paths to the task, explicitly modeling concept dependencies offers limited additional benefit, while the simpler bipartite parameterization can be easier to optimize and less sensitive to graph misspecification. On the larger and more structurally rich datasets (\textsc{Hailfinder}, \textsc{SIIM-Pneumothorax} and \textsc{CheXpert}), graph-based instantiations tend to yield stronger gains under interventions, followed closely by CEMs, which have higher representational flexibility than CBMs due to their ability to encode richer concept representations. 

\begin{figure}[H]
   \centering   \includegraphics[width=0.99\linewidth]{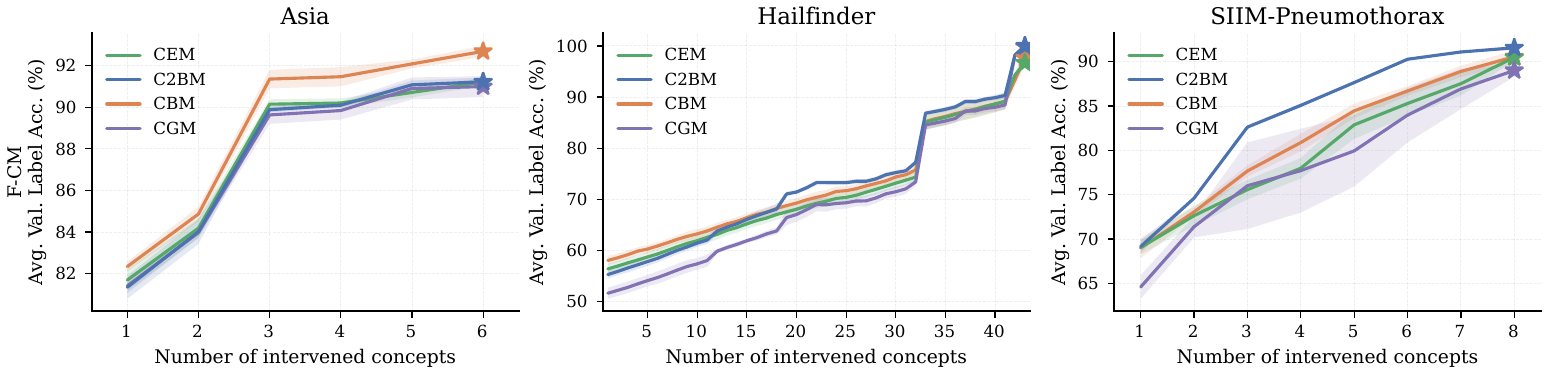}   \caption{\textbf{Label accuracy} (\%) following interventions on concepts at increasing depths of the ground-truth graph hierarchy across CMs instantiations.}
\label{fig:comparison_CMs_label}
\end{figure}

\begin{figure}[H]
   \centering   \includegraphics[width=0.99\linewidth]{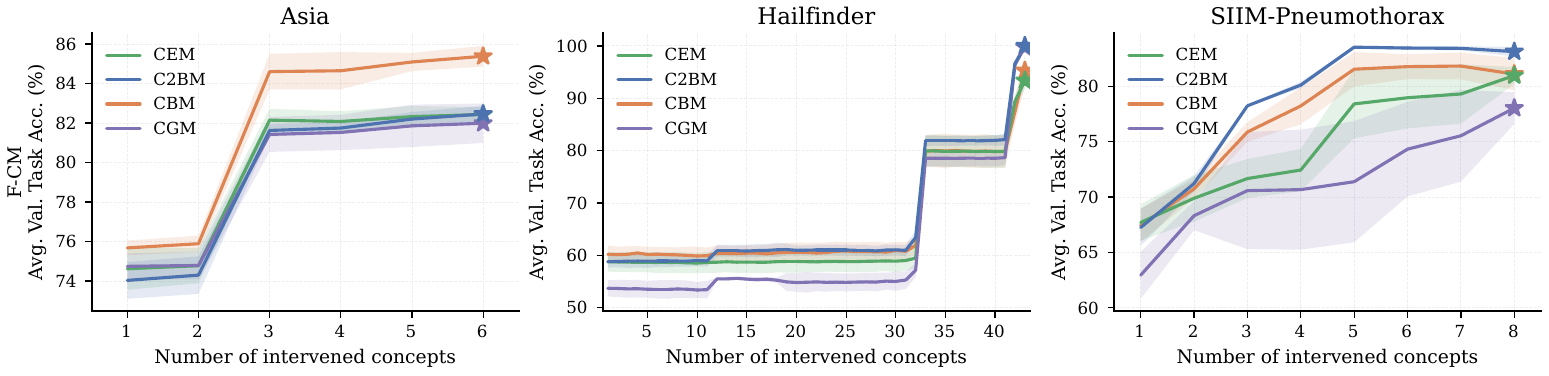}   \caption{\textbf{Task accuracy} (\%) following interventions on concepts at increasing depths of the ground-truth graph hierarchy across CMs instantiations.}
\label{fig:comparison_CMs_task}
\end{figure}

\begin{figure}[H]
   \centering   \includegraphics[width=0.7\linewidth]{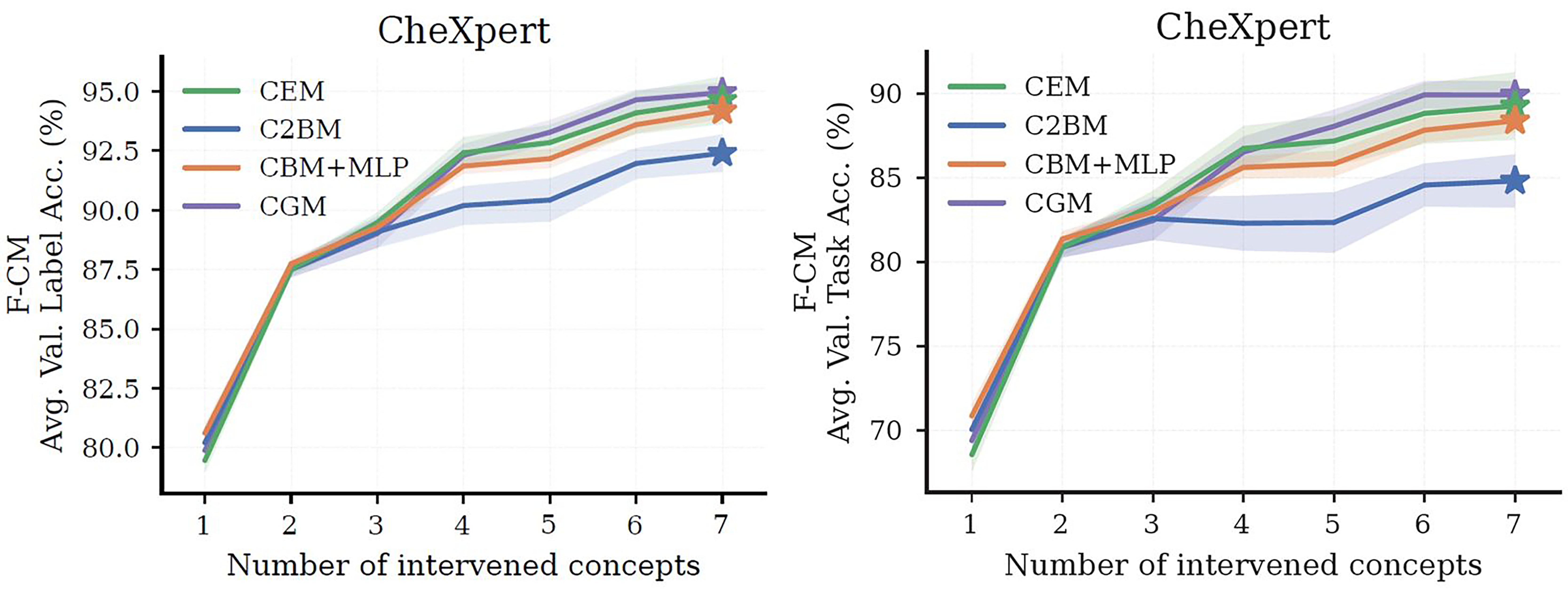}   \caption{\textbf{Label} and \textbf{task accuracy} (\%) on \textsc{CheXpert} following interventions on concepts at increasing depths of the ground-truth graph hierarchy across CMs instantiations.}
\label{fig:comparison_chexpert}
\end{figure}

\subsection{Additional model instances for efficient adaptation in non-stationary federations}
\label{app:convergence_cbm_cgm}
To complement the analysis in Sec.~\ref{sec:exp_trends}, we report the same non-stationary federation experiment for two additional model instantiations, \textsc{CBM} and \textsc{CGM}. As in the main paper, we compare \textbf{F-CMs} (ours), \textbf{S-F-CMs} (ours without architecture adaptation), and \textbf{S-F-CMs Reinit} (full retraining after federation growth), and measure the average validation task loss across communication rounds when newly joined clients introduce previously unseen concepts.

Fig.~\ref{fig:convergence_cbm_cgm} shows trends that are qualitatively consistent with those reported in the main paper for C$^2$BM and CEM. In both additional instantiations, \textbf{F-CMs again adapts more efficiently to federation growth}, reaching lower loss on average than the static baseline and avoiding the slow recovery induced by re-initialization from scratch. Overall, these additional results further support the conclusion that selectively expanding only the modules affected by newly introduced concepts or dependencies yields a substantially more efficient adaptation strategy than either freezing the architecture or fully re-training the model after each structural change.

\begin{figure}[H]
   \centering
   \includegraphics[width=0.99\linewidth]{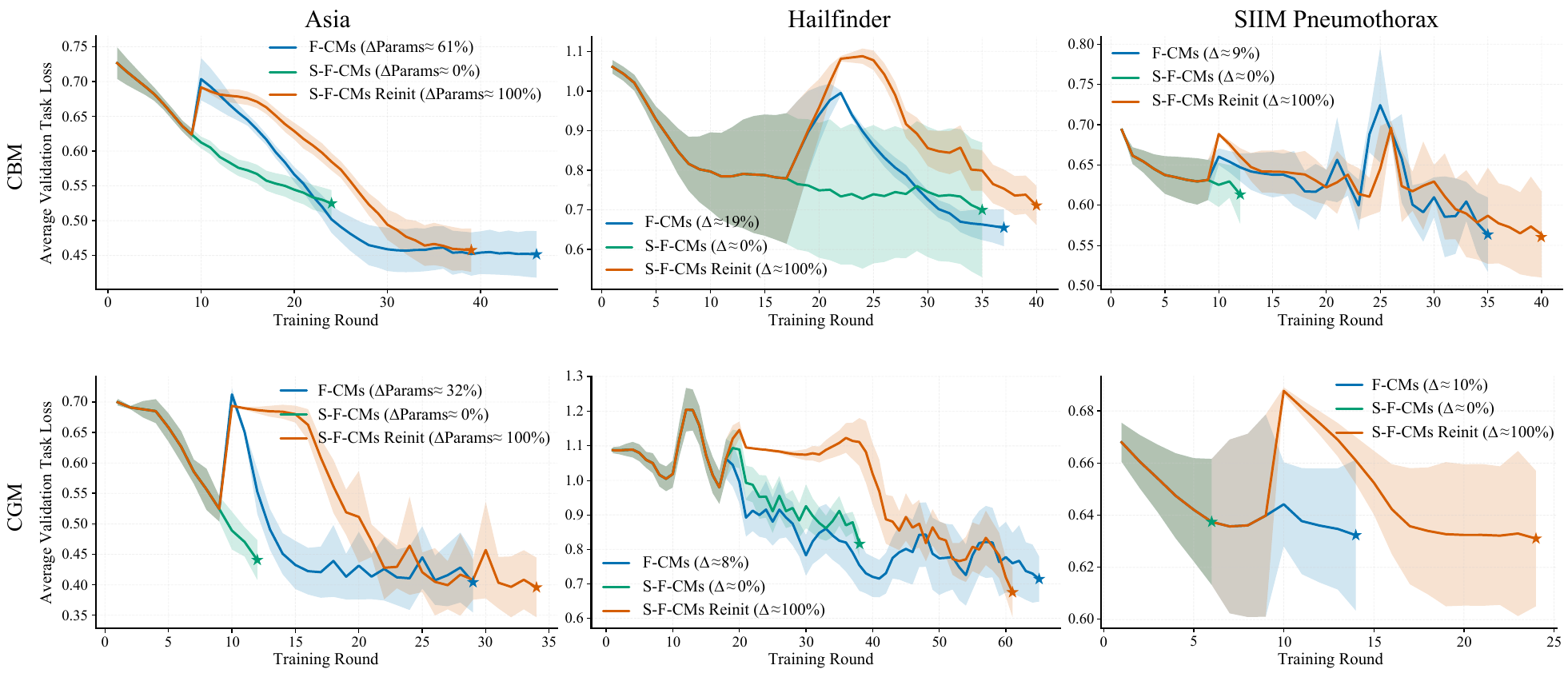}
   \caption{\textbf{Convergence under temporal non-stationarity for additional model instantiations.}
   Average validation task loss across rounds for \textsc{CBM} and \textsc{CGM}, where new clients introduce previously unseen concepts. F-CMs (ours) adapts the architecture modularly and converges faster and to lower loss than S-F-CMs (no adaptation) and S-F-CMs Reinit (full retraining). Curves are truncated at the best loss-round; the omitted tail corresponds to no-improvement rounds. The legend reports the fraction of modified parameters.}
   \label{fig:convergence_cbm_cgm}
\end{figure}

\subsection{Extent of architecture updates under temporal non-stationarity}
\label{app:params_changed}
To quantify the efficiency of \emph{dynamic architecture adaptation} under temporal non-stationarity, we measure the fraction of shared model parameters whose values change after new clients introduce previously unseen concepts (and, for graph-based instantiations, updated dependencies). Concretely, we report the percentage of parameters that differ between the shared model before the arrival of new clients and the final shared model at the end of training, averaged across seeds. We use this metric as a proxy for the \emph{scope of adaptation} and \emph{training compute-efficiency}: smaller values indicate that updates remain localized to a limited set of concept/task modules rather than inducing widespread changes.

Table~\ref{tab:params_changed} reports results across datasets and model instantiations. Overall, the fraction of changed parameters remains far below full retraining (i.e., $\ll 100\%$), supporting our main claim that F-CMs incorporate new supervision by updating only the affected modules while preserving the rest of the model. The magnitude of the update varies across architectures and datasets, reflecting differences in concept dimensionality and the degree to which newly introduced supervision interacts with existing modules (e.g., CEM typically changes a larger fraction than C$^2$BM). These results complement the convergence trends in Fig.~\ref{fig:convergence}: sparser updates are associated with faster recovery after client arrivals, avoiding the slow restart behaviour of retraining from scratch.



\begin{table}[H]
\centering
\scriptsize
\caption{\textbf{Fraction of parameters changed (\%)} due to \emph{temporal non-stationarity}.} 
\label{tab:params_changed}
\resizebox{0.85\textwidth}{!}{
\begin{tabular}{lcccccc}
\toprule
\scriptsize
\textbf{Model} & \textbf{\textsc{Asia}} & \textbf{\textsc{Sachs}} & \textbf{\textsc{Alarm}} & \textbf{\textsc{Insurance}} & \textbf{\textsc{Hailfinder}} & \textbf{\textsc{SIIM Pn.}} \\
\midrule
OpaqNN & 18.28$\pm$2.56 & 38.32$\pm$6.40 & 42.76$\pm$5.82 & 28.90$\pm$4.24 & 17.37$\pm$2.93 & 5.46$\pm$0.67 \\
CBM    & 62.55$\pm$12.93 & 70.00$\pm$15.48 & 48.44$\pm$7.88 & 26.90$\pm$3.58 & 19.24$\pm$4.21 & 9.27$\pm$1.26 \\
CEM    & 61.30$\pm$11.83 & 81.11$\pm$18.54 & 57.30$\pm$8.84 & 30.44$\pm$4.48 & 17.80$\pm$3.13 & 21.61$\pm$3.01 \\
CGM    & 32.19$\pm$3.17 & 43.09$\pm$7.05 & 45.87$\pm$10.04 & 19.68$\pm$3.39 & 8.37$\pm$1.45  & 10.29$\pm$1.33 \\
C$^2$BM & 28.56$\pm$4.54 & 44.82$\pm$8.49 & 31.06$\pm$4.62  & 19.24$\pm$2.78 & 12.73$\pm$2.05 & 8.49$\pm$1.13 \\
\bottomrule
\end{tabular}
}
\end{table}

\subsection{Computational overhead of graph aggregation and modular adaptation}
\label{app:overhead}
We analyze the computational implications of F-CMs from two complementary perspectives: the additional overhead introduced by graph aggregation and modular architecture adaptation, and the scaling of the shared model as new concepts accumulate. The former comes from two main sources: \emph{(i)} server-side graph aggregation and cycle resolution, and \emph{(ii)} sparse architecture adaptation when newly arrived clients introduce unseen concepts or dependencies.

\paragraph{Complexity analysis.}
Consider a client $k$ supervising $m_k$ variables. If the client communicates a dense local adjacency/confidence matrix, the structural communication cost is $\mathcal{O}(m_k^2)$, and aggregating the client-provided structures at the server costs $\mathcal{O}\!\left(\sum_{k\in\mathcal{K}_t} m_k^2\right)$. Importantly, this cost is independent of the number of local training samples, local epochs, and SGD steps, and is therefore small relative to the dominant FL cost, namely repeated forward/backward passes over model parameters. For intuition, even with $m_k=50$, a dense graph contains only $2500$ entries; with float32 values this corresponds to roughly 10~KB of transmitted data, which is negligible compared to communicating model updates for networks with millions of parameters, typically requiring several MBs. Moreover, graph communication is \emph{event-triggered}: a client only needs to transmit its local graph when it first joins the federation, or when its graph changes. Hence, this cost is amortized over rounds, and in most rounds the communication pattern is unchanged from standard FL. After combining the client-provided graphs, the server applies the cycle-resolution step described in App. \ref{parag:DAG_construction_cycle_res}, which depends only on the number of currently known variables $m_t$ and remains a lightweight server-side operation.

Architecture adaptation is similarly sparse. When the shared structure changes, F-CMs keep the encoder fixed and expands or rewires only the concept/task modules whose interfaces are affected, while warm-starting all unchanged parameters. Therefore, the adaptation cost scales as $\mathcal{O}(\Delta P_t)$, where $\Delta P_t$ is the number of modified parameters, rather than $\mathcal{O}(P_t)$ for full retraining. This is consistent with the results in Table~\ref{tab:params_changed}: in our experiments, the fraction $\Delta P_t/P_t$ remains well below full reinitialization, explaining the faster recovery observed in Sec.~\ref{sec:exp_trends} and App.~\ref{app:convergence_cbm_cgm}.

For ordinary FL communication, F-CMs already performs module-specific local training and module-wise aggregation, so only supervised modules produce updates. As a result, the uplink communication is proportional to the updated subset rather than the full model. The downlink can remain identical to standard FL when the whole model is broadcast. Overall, compared to static FL, F-CMs introduce only a modest and mostly event-triggered structural overhead while gaining the ability to absorb late-arriving concepts; compared to retraining from scratch, it is substantially cheaper because it avoids re-optimizing the full model.

\paragraph{Runtime overhead.}
Table~\ref{tab:overhead} reports the average runtime of the main additional operations introduced by F-CMs. The results confirm that these costs are negligible in practice compared to the overall training time. In particular, graph aggregation requires only a few hundredths of a second when present, while post-drift architecture instantiation remains below one second across all reported settings. Parameter transfer from the pre-drift to the post-drift model is even smaller. By contrast, the total end-to-end training time is in the order of hundreds of seconds, showing that the computational overhead of structure handling and modular adaptation is minimal relative to the dominant cost of repeated local training and communication rounds.

\begin{table}[t]
\centering
\scriptsize
\setlength{\tabcolsep}{3pt}
\caption{\textbf{Average runtime (seconds) of the additional operations introduced by F-CMs}, reported as mean $\pm$ 95\% confidence interval over three runs. Here, $d\_inst.$ denotes the time spent instantiating the architecture after client drift, $d\_par.$ the time spent initializing the post-drift architecture with previously learned parameters, $agg.$ the time spent aggregating client graphs, $n\_rounds$ the number of training rounds, and $tot.$ the total training time. The results show that graph aggregation and modular adaptation contribute only a very small fraction of the total runtime.}
\begin{tabular}{llccccc}
\toprule
\textbf{Data} & \textbf{Model} & \textbf{$d\_inst.$} & \textbf{$d\_par.$} & \textbf{$agg.$} & \textbf{$n\_rounds$} & \textbf{$tot.$} \\
\midrule
\textsc{Asia} & F-CMs (CEM) 
& 0.227 $\pm$ 0.001 & 0.002 $\pm$ 0.000 & -- 
& 42.583 $\pm$ 2.256 & 229.726 $\pm$ 13.609 \\
\textsc{SIIM-Pn.} & F-CMs (CEM) 
& 0.691 $\pm$ 0.093 & 0.003 $\pm$ 0.000 & -- 
& 35.350 $\pm$ 5.188 & 506.230 $\pm$ 75.622 \\
\midrule
\textsc{Asia} & F-CMs (C$^2$BM) 
& 0.244 $\pm$ 0.001 & 0.003 $\pm$ 0.000 & 0.063 $\pm$ 0.051
& 61.500 $\pm$ 3.679 & 370.800 $\pm$ 24.924 \\
\textsc{SIIM-Pn.} & F-CMs (C$^2$BM) 
& 0.672 $\pm$ 0.102 & 0.004 $\pm$ 0.000 & 0.037 $\pm$ 0.003
& 56.550 $\pm$ 7.278 & 922.543 $\pm$ 120.707 \\
\bottomrule
\end{tabular}
\label{tab:overhead}
\end{table}

\paragraph{Scaling under concept accumulation.}
As defined in Sec.~\ref{sec:problem_formulation}, F-CMs retain concepts already incorporated into the
shared model when clients leave, since client absence does not necessarily imply that the corresponding concepts are obsolete. This conservative choice avoids prematurely discarding potentially relevant knowledge.\footnote{The modular adaptation mechanism is nevertheless compatible with concept removal: the corresponding module could be removed, the shared DAG updated as for other structural changes, and only the affected interfaces adapted.  Developing principled criteria for concept pruning while preserving relevant knowledge is discussed as a future research direction.} We therefore additionally quantify how
the shared CM scales as new concepts accumulate. Specifically, we measure the
parameters introduced by each new concept module, the total parameter growth
including affected interfaces, the resulting FP32 memory footprint, and the
increase in forward-pass multiply--accumulate operations (MACs). We report
results for CEM and C$^2$BM on \textsc{Asia} and \textsc{Alarm}. All quantities isolate the
concept-dependent component of the model; the fixed input encoder is excluded,
as its size and computation do not depend on the number of concepts.

Table~\ref{tab:concept-scaling} shows that model growth remains modest in
absolute memory in the evaluated settings: the largest post-drift
concept-dependent component occupies $0.745$ MiB in FP32. The parameter cost
of an individual new concept module ranges from $610$ to $3{,}843$ parameters,
with additional growth arising when existing concept or task interfaces are
affected. Forward computation increases more visibly, by $20.7$--$42.0\%$
for the concept-dependent component, although the corresponding full-model
increase is lower because the fixed input encoder is unchanged. Overall,
the scaling cost is architecture-dependent, with C$^2$BM exhibiting lower
parameter and forward-computation growth than CEM in both evaluated settings.

Importantly, concept accumulation does not induce a proportional increase in
client-side backward computation or uplink cost, since unsupervised modules
remain frozen and are not transmitted; the measured forward-cost growth instead
reflects that the CM instantiations considered here construct the complete
concept representation for task prediction.

\begin{table}[t]
    \centering
    \caption{
        \textbf{Scaling under concept accumulation.}
        Growth of the concept-dependent component for CEM and C$^2$BM as
        new concepts are incorporated. ``New params.'' denotes the parameters
        of a newly instantiated concept module, while ``Added params.'' also
        includes affected interfaces. Memory is the post-drift FP32 parameter
        footprint. The fixed input encoder is excluded from all measurements.
    }
    \label{tab:concept-scaling}
    \setlength{\tabcolsep}{3pt}
    \resizebox{\linewidth}{!}{%
    \begin{tabular}{lc|cc|cc|cc|cc}
        \toprule
        & & \multicolumn{2}{c|}{New params./concept}
          & \multicolumn{2}{c|}{Added params.}
          & \multicolumn{2}{c|}{Memory (MiB)}
          & \multicolumn{2}{c}{Forward MAC increase} \\
        Dataset & Concepts added
        & CEM & C$^2$BM
        & CEM & C$^2$BM
        & CEM & C$^2$BM
        & CEM & C$^2$BM \\
        \midrule
        \textsc{Asia}
        & $3$
        & $1{,}346$ & $610$
        & $5{,}130$ & $1{,}982$
        & $0.066$ & $0.044$
        & $42.0\%$ & $20.7\%$ \\
        \textsc{Alarm}
        & $8$
        & $3{,}843$ & $1{,}763$
        & $48{,}513$ & $15{,}516$
        & $0.745$ & $0.335$
        & $33.1\%$ & $21.4\%$ \\
        \bottomrule
    \end{tabular}%
    }
\end{table}

\subsubsection{Ablation of Interface-Initialization Strategies}
\label{app:interface-initialization}

Dynamic architecture adaptation changes the interface
of the affected modules. Although F-CMs localize each structural update
and preserve all unaffected modules, reinitializing a previously learned
interface may temporarily perturb optimization. We therefore investigate
how different initialization strategies affect the short-term adaptation
dynamics and the final validation loss after concept-space expansion.
We compare three strategies that differ only in how the modified
interface is initialized:

\begin{enumerate}
    \item \textbf{Current strategy (Current).}
    The resized input or output interface of the affected module is
    reinitialized, while all unaffected modules retain their parameters. See more details in Section~\ref{sec:arch_update} and App.~\ref{app:baselines_details}.

    \item \textbf{Partial reuse with random initialization
    (Partial-random).}
    Parameters associated with existing inputs or outputs are copied
    from the pre-update module, while only those associated with newly
    introduced inputs or outputs are randomly initialized. For example, suppose the task decoder 
 is initially trained to predict $Y$ 
 from $C_1$
. When a new concept $C_2$ 
 is added as an input, we expand the first layer of the task decoder to receive 
$(C_1,C_2)$. The block of weights associated with 
$C_1$ is copied from the previously learned decoder, while the new block associated with $C_2$ 
 is randomly initialized. Training then proceeds as in standard F-CMs.

    \item \textbf{Partial reuse with zero initialization
    (Partial-zero).}
    Parameters associated with existing inputs or outputs are
    preserved, as in the previous strategy, while the newly introduced parameters are initialized
    to zero. When applicable, this initialization preserves previously learned dependencies immediately after the architectural change, while allowing new dependencies to be learned during subsequent training.
\end{enumerate}

We evaluate the three strategies using C\(^{2}\)BM and CGM instantiations on the \textsc{Asia} and \textsc{Alarm} datasets. For each dataset, we consider two concept-space expansion scenarios, from initial coverage ranges of \(10\)--\(50\%\) and \(51\)--\(75\%\), respectively, to full coverage. We quantify coverage using \emph{structural concept coverage}, a model-independent measure of the fraction of nodes in the ground-truth graph covered by the union of the selected client subgraphs (see App.~\ref{app:concept-expansion} for details).

Figures~\ref{fig:init-ablation-asia} and
\ref{fig:init-ablation-alarm} show how the validation task loss evolves
before and after the structural update under the three
interface-initialization strategies. Each figure reports results for
C\(^{2}\)BM and CGM under the \(10\)--\(50\%\rightarrow100\%\) and
\(51\)--\(75\%\rightarrow100\%\) coverage expansions. At each federated
round, the loss is averaged across the validation sets of the clients in
the training federation. Curves and shaded regions report the mean
\(\pm\) standard deviation over three seeds, while the vertical dashed
line marks the structural update. To illustrate longer-term training dynamics, we allow the models to train for additional rounds. However,
model selection is based on the lowest validation loss
achieved during training, rather than the loss at
the final round.

\begin{figure*}[t]
    \centering
    \includegraphics[width=0.85\linewidth]{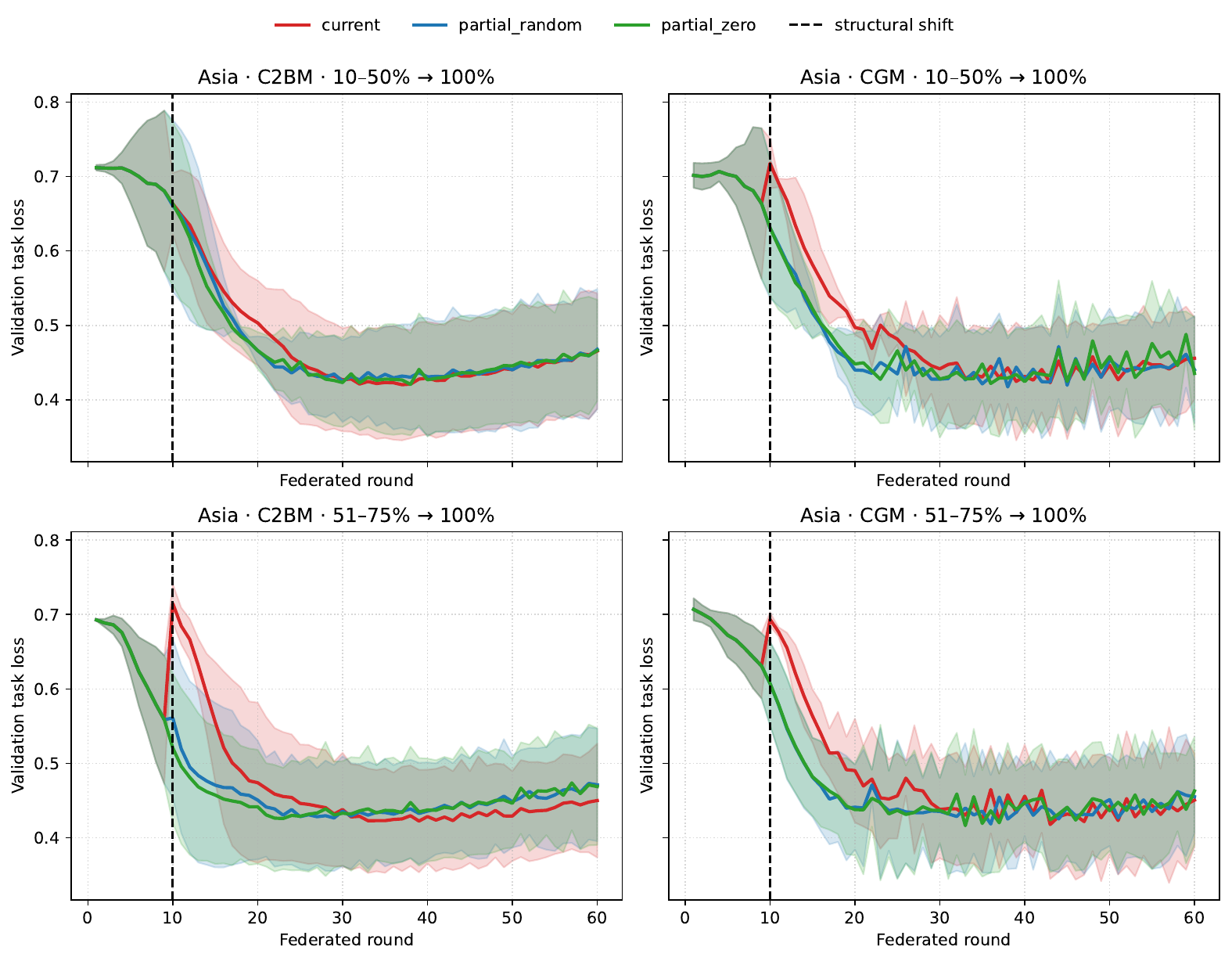}
    \caption{\textbf{Interface-initialization ablation on \textsc{Asia}.}
    Average validation task loss across federated rounds for
    C\(^{2}\)BM (left) and CGM (right), under concept-space expansions
    from \(10\)--\(50\%\) to \(100\%\) structural concept coverage
    (top) and from \(51\)--\(75\%\) to \(100\%\) coverage (bottom).
    We compare the current initialization strategy with partial
    parameter reuse, using either random or zero initialization for
    newly introduced parameters. The vertical dashed line marks the
    structural update. Curves and shaded regions report mean
    \(\pm\) standard deviation over three seeds.}
    \label{fig:init-ablation-asia}
\end{figure*}

\begin{figure*}[t]
    \centering
    \includegraphics[width=0.85\linewidth]{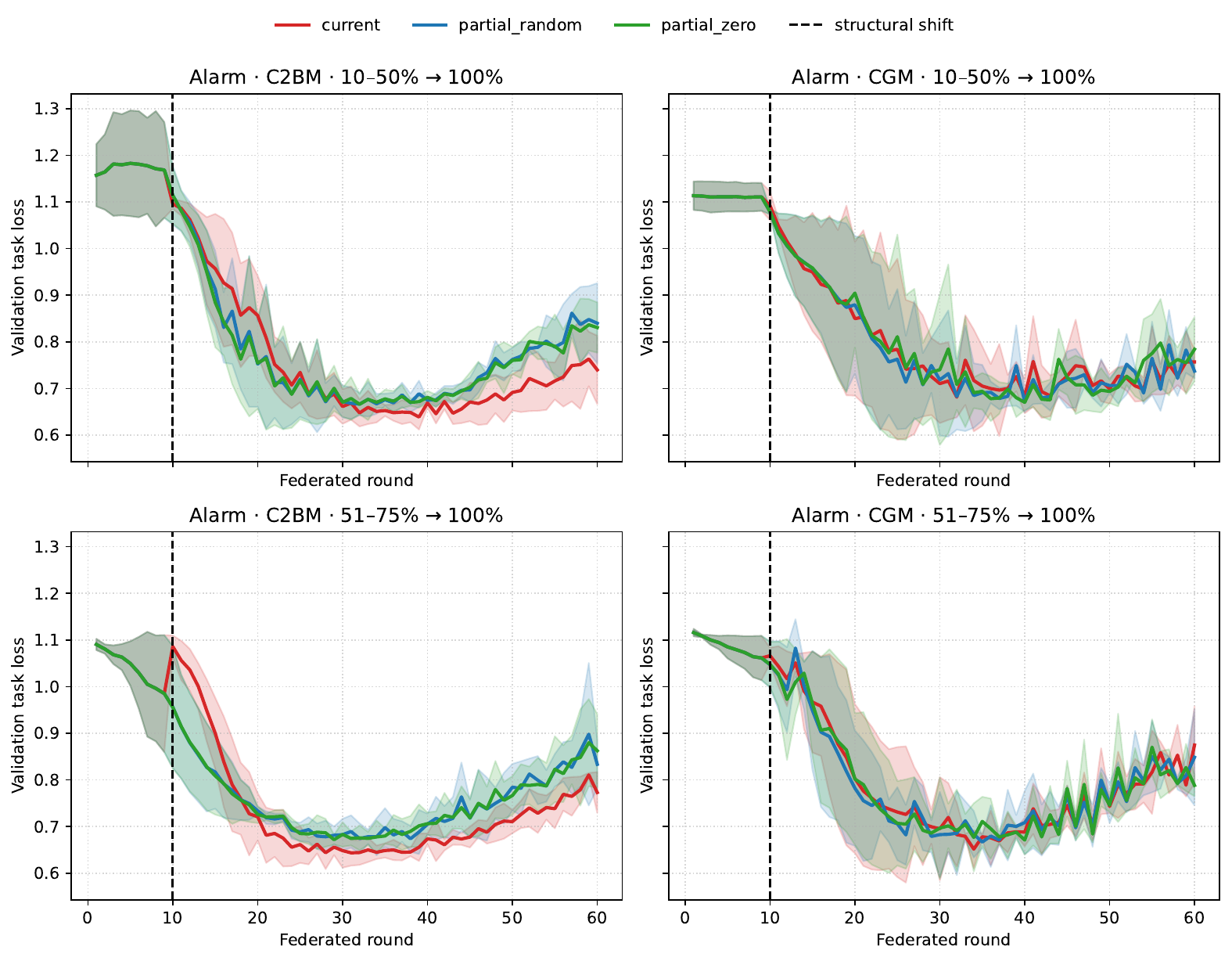}
    \caption{\textbf{Interface-initialization ablation on \textsc{Alarm}.}
    Average validation task loss across federated rounds for
    C\(^{2}\)BM (left) and CGM (right), under concept-space expansions
    from \(10\)--\(50\%\) to \(100\%\) structural concept coverage
    (top) and from \(51\)--\(75\%\) to \(100\%\) coverage (bottom).
    We compare the current initialization strategy with partial
    parameter reuse, using either random or zero initialization for
    newly introduced parameters. The vertical dashed line marks the
    structural update. Curves and shaded regions report mean
    \(\pm\) standard deviation over three seeds.}
    \label{fig:init-ablation-alarm}
\end{figure*}

On \textsc{Asia} (Fig.~\ref{fig:init-ablation-asia}), both concept-space expansion scenarios exhibit broadly similar overall trends. Immediately after the structural update, the two partial-reuse strategies generally achieve an earlier reduction in validation loss, while the current strategy exhibits a transient increase, particularly for CGM and C\(^{2}\)BM in the second expansion scenario. These differences diminish during subsequent training, with all three strategies reaching comparable validation losses. In the second expansion scenario for C\(^{2}\)BM, the current strategy attains a slightly lower late-stage loss, although the standard-deviation bands overlap. The trajectories of \emph{partial-random} and \emph{partial-zero} remain close, suggesting that the main short-term benefit comes from preserving existing parameters rather than from the initialization of newly introduced ones.

On \textsc{Alarm} (Fig.~\ref{fig:init-ablation-alarm}), the two concept-space expansion scenarios exhibit broadly similar trends. For C\(^{2}\)BM, the partial-reuse strategies achieve a faster initial reduction in validation loss after the structural update, while the current strategy adapts more gradually. This difference is particularly evident in the second expansion scenario, where the current strategy exhibits a transient increase in loss. For CGM, the three strategies show largely comparable trajectories, with only minor differences during post-update adaptation.

In later training rounds, the current strategy reaches a numerically lower validation loss for C\(^{2}\)BM in both expansion scenarios, although the standard-deviation bands overlap. For CGM, the strategies achieve comparable losses without a consistent advantage for any initialization approach.

Taken together, the results on \textsc{Asia} and \textsc{Alarm} suggest that preserving parameters associated with the existing interface generally facilitates the initial adaptation following a structural update. However, this short-term benefit does not consistently translate into lower validation loss later in training. The current strategy typically catches up with the partial-reuse variants and can achieve comparable or, in some settings, numerically lower losses. Moreover, the relative differences between strategies do not follow a consistent pattern across the two expansion scenarios. Overall, no strategy consistently outperforms the others across datasets, architectures, and concept-space expansions.

\subsection{Differential privacy}\label{app:diff_privacy}
We investigate two complementary privacy mechanisms for F-CMs. First, we
consider sample-level differential privacy (DP) for model updates, assessing its
impact on predictive and concept performance (App.~\ref{app:sampleDP}). We then study an optional
edge-level differentially private mechanism for client graph communication,
evaluating the resulting trade-off between structural privacy, graph quality,
and predictive performance (App.~\ref{app:edgeDP}).

\subsubsection{Sample-level privacy of model updates} \label{app:sampleDP}
We first consider sample-level DP for communicated model updates, limiting the information that can be inferred about individual training samples.
DP provides a formal guarantee that the inclusion or exclusion of any single data sample has a bounded impact on the released information, thereby protecting individual participants against inference attacks. Formally, a randomized mechanism $\mathcal{M} : \mathcal{D} \rightarrow
\mathcal{R}$ satisfies $(\varepsilon,\delta)$-differential privacy if, for any pair of neighboring datasets $D$ and $D'$ that differ in one sample and any measurable set $\mathcal{S} \subseteq \mathcal{R}$,
\begin{equation}
\Pr[\mathcal{M}(D) \in \mathcal{S}]
\;\le\;
e^{\varepsilon}\Pr[\mathcal{M}(D') \in \mathcal{S}] + \delta .
\end{equation}
Here, $\varepsilon$ controls the strength of the privacy guarantee, with smaller values corresponding to stronger privacy protection, while $\delta$ represents a negligible probability of privacy violation.

In practice, we enforce DP at the client update level using a standard DP-SGD pipeline implemented with Opacus~\cite{opacus2022}. At each local optimization step, client gradients are computed at the \emph{per-sample} level and then \emph{clipped} to a maximum $\ell_2$ norm $C$:
\begin{equation}
g_i \leftarrow g_i \cdot \min\Bigl(1,\frac{C}{\|g_i\|_2}\Bigr),
\qquad C=1.0.
\end{equation}
Clipping bounds the contribution of any single sample to the update, which effectively bounds the sensitivity of the released (aggregated) gradient signal. After clipping, Gaussian noise is added to the (mini-batch) aggregated gradient before taking the optimizer step. The noise scale is chosen by the privacy accountant to meet a target privacy budget $(\varepsilon,\delta)$ given the training configuration (sampling rate, number of steps/epochs, etc.). In our experiments we fix $\delta=10^{-5}$ and report results for $\varepsilon\in\{10,5,1\}$, where smaller $\varepsilon$ corresponds to stronger privacy (and typically lower utility).

\paragraph{Empirical Validation.}
Table \ref{tab:dp} reports task, and concept accuracy under DP for different privacy budgets on \textsc{Asia}. As expected, tightening privacy (decreasing $\varepsilon$) leads to a moderate degradation in performance. Notably, across most instantiations the drop from $\varepsilon=10$ to $\varepsilon=5$ is small, and even at $\varepsilon=1$ several models remain competitive (e.g., OpaqNN and CEM exhibit limited degradation), indicating that F-CM-style training remains effective under strong privacy constraints. Larger drops for some architectures (e.g., CBM and CGM at $\varepsilon=1$) are consistent with the stronger noise required at low $\varepsilon$ and the known sensitivity of some training dynamics to DP perturbations.

\begin{table}[htbp]
\centering
\caption{
\textbf{Impact of differential privacy on federated concept-based models under
different privacy budgets $\varepsilon$.}
We report task accuracy and concept accuracy (mean $\pm$ std.) for
representative model instantiations.
}
\scriptsize
\renewcommand{\arraystretch}{1.2}
\setlength{\tabcolsep}{4pt}
\begin{tabular}{lcccccc}
\toprule
\textbf{Privacy Budget} &
\multicolumn{2}{c}{\textbf{$\varepsilon=10$}} &
\multicolumn{2}{c}{\textbf{$\varepsilon=5$}} &
\multicolumn{2}{c}{\textbf{$\varepsilon=1$}} \\
\cmidrule(lr){2-3} \cmidrule(lr){4-5} \cmidrule(lr){6-7}
\textbf{} &
\textbf{Task Acc.} & \textbf{Concept Acc.} &
\textbf{Task Acc.} & \textbf{Concept Acc.} &
\textbf{Task Acc.} & \textbf{Concept Acc.} \\
\midrule
OpaqNN  &
80.03 $\pm$ 0.68 & 90.37 $\pm$ 0.70 &
79.19 $\pm$ 0.97 & 89.43 $\pm$ 1.32 &
77.31 $\pm$ 0.96 & 89.15 $\pm$ 1.07
\\
C$^2$BM &
79.37 $\pm$ 1.01 & 86.89 $\pm$ 2.78 &
79.20 $\pm$ 1.16 & 86.89 $\pm$ 2.63 &
76.22 $\pm$ 2.24 & 85.81 $\pm$ 2.06
\\
CBM     &
79.11 $\pm$ 1.79 & 89.32 $\pm$ 1.46 &
79.03 $\pm$ 1.87 & 89.13 $\pm$ 1.49 &
72.73 $\pm$ 6.52 & 88.07 $\pm$ 2.16
\\
CEM     &
80.19 $\pm$ 1.08 & 86.26 $\pm$ 2.72 &
80.17 $\pm$ 1.63 & 87.26 $\pm$ 2.74 &
78.12 $\pm$ 2.68 & 85.67 $\pm$ 2.47
\\
CGM     &
78.67 $\pm$ 2.39 & 88.59 $\pm$ 1.34 &
77.68 $\pm$ 0.94 & 86.47 $\pm$ 1.90 &
72.00 $\pm$ 2.23 & 85.83 $\pm$ 1.83
\\
\bottomrule
\end{tabular}
\label{tab:dp}
\end{table}

\subsubsection{Edge-level structural privacy} \label{app:edgeDP}
Complementary to sample-level differential privacy for model updates, we
consider structural privacy for client graph communication. In sensitive
applications, such structural information may itself reveal properties of a
client's data or domain-specific practices. We therefore equip the communicated
client structures $A^{(k)}$ with an optional edge-level differentially private
mechanism that perturbs pairwise structural relations before transmission.

For each unordered pair $\{V_i,V_j\}$, the client graph assigns one of three
possible relations,
\(
\mathcal{R}_{ij}
=
\{V_i \rightarrow V_j,\;
  V_j \rightarrow V_i,\;
  \emptyset\}.
\)
Before transmission, the client independently perturbs the relation associated
with each pair using randomized response. Given structural privacy budget
$\epsilon_s$, the true relation is retained with probability
\(
p
=
\frac{e^{\epsilon_s}}{e^{\epsilon_s}+2},
\)
while each of the other two relations is reported with probability
\(
q
=
\frac{1}{e^{\epsilon_s}+2}.
\)
For any two possible relations $r,s\in\mathcal{R}_{ij}$ and any reported
relation, the ratio between the corresponding reporting probabilities is at
most $e^{\epsilon_s}$, providing $\epsilon_s$-edge-level differential privacy
for each pair.

This perturbation is compatible with the consensus-based aggregation used by
F-CMs. In particular, if $f_r$ denotes the true frequency of relation $r$
across clients, its perturbed frequency satisfies
\(
\mathbb{E}[\widetilde{f}_r]
=
q+(p-q)f_r.
\)
Consequently, for any two relations $r$ and $s$,
\(
\mathbb{E}[\widetilde{f}_r]
-
\mathbb{E}[\widetilde{f}_s]
=
(p-q)(f_r-f_s).
\)
Since $p-q>0$, the ordering of relation frequencies is preserved in expectation.

\paragraph{Empirical evaluation.}
We evaluate the effect of structural privacy on \textsc{Asia} and \textsc{Alarm} using two
representative graph-based instantiations, C$^2$BM and CGM. In addition to
task accuracy, we report graph disagreement (GD), measured relative to the
non-private aggregated graph, where lower values indicate greater structural
agreement. The non-private setting is denoted by $\epsilon_s=\infty$.

\begin{table}[t]
    \centering
    \caption{
        \textbf{Edge-level structural privacy.}
        Task accuracy (\%) for C$^2$BM and CGM and graph disagreement
        (GD, \%) with respect to the non-private aggregate under different
        structural privacy budgets $\epsilon_s$.
    }
    \label{tab:structural-privacy}
    \setlength{\tabcolsep}{5pt}
    \begin{tabular}{lcccc}
        \toprule
        Dataset & $\epsilon_s$ & C$^2$BM & CGM & GD (\%) \\
        \midrule
        \multirow{5}{*}{\textsc{Asia}}
        & $\infty$ & $80.4 \pm 2.3$ & $77.0 \pm 3.9$ & $0$ \\
        & $8$      & $80.4 \pm 2.3$ & $77.0 \pm 3.9$ & $0$ \\
        & $4$      & $80.4 \pm 2.3$ & $77.0 \pm 3.9$ & $0$ \\
        & $2$      & $80.2 \pm 1.9$ & $77.9 \pm 1.7$ & $3.6 \pm 3.6$ \\
        & $1$      & $79.8 \pm 2.4$ & $77.8 \pm 1.5$ & $21.4 \pm 11.3$ \\
        \midrule
        \multirow{5}{*}{\textsc{Alarm}}
        & $\infty$ & $72.6 \pm 0.9$ & $71.1 \pm 0.6$ & $0$ \\
        & $8$      & $72.6 \pm 0.9$ & $71.0 \pm 0.7$ & $0$ \\
        & $4$      & $72.5 \pm 2.7$ & $70.9 \pm 0.5$ & $1.0 \pm 0.5$ \\
        & $2$      & $72.8 \pm 2.2$ & $65.5 \pm 8.6$ & $6.6 \pm 1.7$ \\
        & $1$      & $70.8 \pm 0.4$ & $65.6 \pm 10.4$ & $15.8 \pm 2.2$ \\
        \bottomrule
    \end{tabular}
\end{table}

Table~\ref{tab:structural-privacy} shows that moderate structural privacy
introduces little degradation. For $\epsilon_s \geq 4$, the aggregated graph
remains unchanged on \textsc{Asia} and differs on average by at most $1.0\%$ on \textsc{Alarm}, while task
accuracy is essentially preserved for both architectures. Under stronger
perturbation, the effect becomes dataset- and architecture-dependent. At
$\epsilon_s=1$, GD increases to $21.4\%$ on \textsc{Asia} and $15.8\%$ on \textsc{Alarm},
while C$^2$BM remains comparatively stable, decreasing by $0.6$ and $1.8$
percentage points, respectively, relative to the non-private setting. CGM is
also stable on \textsc{Asia}, but on \textsc{Alarm} its accuracy decreases by approximately
$5.5$ percentage points for $\epsilon_s\in\{2,1\}$. Overall, moderate
edge-level privacy preserves both the aggregated structure and predictive
performance, whereas stronger structural perturbations may reduce utility.


\subsection{Sensitivity analysis}\label{app:sensitivity}
We further assess the stability of F-CMs under key design and deployment choices. In particular, we analyze the effect of (i) varying the number of
clients (App.~\ref{app:scaling_clients}), (ii) selective parameter freezing for locally
unsupervised modules (App.~\ref{app:freeze_effect}), (iii) the concept-loss weight $\gamma$
(App.~\ref{app:gamma_sensitivity}), (iv) the timing of client arrivals (App.~\ref{app:arrival-timing}), and
(v) the magnitude of concept-space expansion under federation growth
(App.~\ref{app:concept-expansion}).

\subsubsection{Number of Clients} \label{app:scaling_clients}

We study the effect of scaling the number of clients by varying the total client population $K \in \{10, 20, 50, 100\}$, while keeping the overall amount of data fixed. As $K$ increases, each client holds fewer local samples on average, resulting in smaller and more heterogeneous local datasets. This setting reflects realistic large-scale federated deployments, where data are fragmented across many participants.

Tables~\ref{tab:scaling_clients} reports task and concept accuracy for different federation sizes. Across all model instantiations, performance remains stable as the number of clients increases, with only moderate degradation in accuracy for the largest federation sizes. In particular, concept accuracy exhibits limited sensitivity to increasing $K$, indicating that aggregating concept supervision from a larger number of smaller clients remains effective.

Models relying on more structured dependencies (e.g., C$^2$BM and CGM) experience slightly larger variance as $K$ grows. This behavior is consistent with the increased statistical heterogeneity and reduced local sample sizes associated with larger client populations. Overall, these results suggest that F-CMs scale gracefully with the number of clients, preserving both concept quality and
downstream task performance even in highly fragmented federations.

\begin{table}[htbp]
\centering
\caption{
\textbf{Scaling with the number of clients.}
Task accuracy and concept accuracy (mean $\pm$ std.) as the total number of
clients $K$ increases, while keeping the overall amount of data fixed.
}
\scriptsize
\renewcommand{\arraystretch}{1.2}
\setlength{\tabcolsep}{3pt}
\begin{tabular}{lcccccccc}
\toprule
\textbf{\# Clients} &
\multicolumn{2}{c}{\textbf{$K=10$}} &
\multicolumn{2}{c}{\textbf{$K=20$}} &
\multicolumn{2}{c}{\textbf{$K=50$}} &
\multicolumn{2}{c}{\textbf{$K=100$}} \\
 \cmidrule(lr){2-3} \cmidrule(lr){4-5} \cmidrule(lr){6-7} \cmidrule(lr){8-9}
 \textbf{Models}
& \textbf{Task Acc.} & \textbf{Concept Acc.}
& \textbf{Task Acc.} & \textbf{Concept Acc.}
& \textbf{Task Acc.} & \textbf{Concept Acc.}
& \textbf{Task Acc.} & \textbf{Concept Acc.} \\
\midrule
OpaqNN  &
80.61 $\pm$ 1.98 & 91.63 $\pm$ 0.70 &
79.83 $\pm$ 0.99 & 90.80 $\pm$ 0.40 &
80.37 $\pm$ 1.06 & 90.03 $\pm$ 0.99 &
79.94 $\pm$ 1.15 & 90.30 $\pm$ 0.82
\\
C$^2$BM &
80.14 $\pm$ 2.02 & 91.23 $\pm$ 0.81 &
78.16 $\pm$ 3.11 & 87.81 $\pm$ 2.32 &
80.35 $\pm$ 0.72 & 86.44 $\pm$ 2.17 &
75.38 $\pm$ 3.98 & 86.10 $\pm$ 2.42
\\
CBM &
79.91 $\pm$ 1.98 & 91.52 $\pm$ 0.83 &
79.25 $\pm$ 1.12 & 91.09 $\pm$ 0.61 &
79.54 $\pm$ 1.02 & 90.94 $\pm$ 0.50 &
78.32 $\pm$ 1.33 & 90.54 $\pm$ 0.55
\\
CEM     &
80.58 $\pm$ 2.04 & 91.40 $\pm$ 0.72 &
80.58 $\pm$ 1.11 & 88.59 $\pm$ 1.63 &
79.92 $\pm$ 1.31 & 87.92 $\pm$ 2.01 &
78.80 $\pm$ 1.80 & 86.19 $\pm$ 1.91
\\
CGM     &
79.33 $\pm$ 2.20 & 91.05 $\pm$ 0.89 &
79.30 $\pm$ 1.74 & 87.76 $\pm$ 1.83 &
78.96 $\pm$ 1.77 & 85.86 $\pm$ 2.10 &
78.84 $\pm$ 1.85 & 86.41 $\pm$ 2.14
\\
\bottomrule
\end{tabular}
\label{tab:scaling_clients}
\end{table}

\subsubsection{Effect of Selective Parameter Freezing}
\label{app:freeze_effect}
We evaluate the empirical impact of \emph{selective parameter freezing}, i.e., freezing concept/task modules that lack local supervision at a client (Sec.~\ref{sec:fed_training}). This mechanism enforces module-specific training by ensuring that only supervised components receive gradients, while unsupervised modules remain unchanged. Tables~\ref{tab:freeze_1}--\ref{tab:freeze_2} show task and concept accuracy with and without freezing for two representative instantiations (C$^2$BM and CEM). Overall, freezing is largely \emph{performance-neutral}: on smaller graphs such as \textsc{Asia}, the effect is minimal since only a few modules are unsupervised at each client. However, on larger benchmarks (e.g., \textsc{Alarm}, \textsc{Hailfinder}, \textsc{SIIM}), where the model contains more concept modules and parent-conditioned dependencies, freezing can slightly improve stability by preventing updates to poorly-supported components and reducing update noise. Beyond accuracy, freezing provides a practical efficiency benefit at the client side, as it reduces the number of trainable parameters during local optimization (lower gradient/optimizer-state footprint and compute), while remaining fully consistent with our module-wise aggregation protocol.
\begin{table*}[t]
\centering
\caption{
\textbf{Effect of selective parameter freezing (\textsc{Asia/Sachs/Alarm}).}
Task and concept accuracy (mean $\pm$ std.) across 5 seeds, comparing standard training (No) vs.\ selective freezing of unsupervised modules (Yes).
}
\scriptsize
\renewcommand{\arraystretch}{1.15}
\setlength{\tabcolsep}{3.5pt}
\begin{tabular}{llcccccc}
\toprule
\textbf{Model} & \textbf{Freezing} &
\multicolumn{2}{c}{\textbf{\textsc{Asia}}} &
\multicolumn{2}{c}{\textbf{\textsc{Sachs}}} &
\multicolumn{2}{c}{\textbf{\textsc{Alarm}}} \\
\cmidrule(lr){3-4}\cmidrule(lr){5-6}\cmidrule(lr){7-8}
& &
\textbf{Task Acc.} & \textbf{Concept Acc.} &
\textbf{Task Acc.} & \textbf{Concept Acc.} &
\textbf{Task Acc.} & \textbf{Concept Acc.} \\
\midrule
\multirow{2}{*}{C$^2$BM} & Yes  &
80.35 $\pm$ 2.46 & 90.55 $\pm$ 1.08 &
76.20 $\pm$ 2.38 & 72.22 $\pm$ 1.91 &
73.14 $\pm$ 0.80 & 90.23 $\pm$ 0.31 \\
& No &
80.52 $\pm$ 2.24 & 90.47 $\pm$ 1.29 &
76.55 $\pm$ 2.20 & 72.73 $\pm$ 1.59 &
72.68 $\pm$ 1.28 & 89.12 $\pm$ 0.84 \\
\midrule
\multirow{2}{*}{CEM} & Yes  &
80.50 $\pm$ 1.90 & 89.77 $\pm$ 2.33 &
75.97 $\pm$ 2.92 & 73.03 $\pm$ 1.55 &
73.42 $\pm$ 0.50 & 90.10 $\pm$ 0.63 \\
& No &
80.27 $\pm$ 1.99 & 90.07 $\pm$ 1.92 &
76.34 $\pm$ 2.58 & 73.33 $\pm$ 1.38 &
72.79 $\pm$ 0.50 & 89.87 $\pm$ 0.71 \\
\bottomrule
\end{tabular}
\label{tab:freeze_1}
\end{table*}

\begin{table}
\centering
\caption{
\textbf{Effect of selective parameter freezing (\textsc{Insurance/Hailfinder/SIIM}).}
Task and concept accuracy (mean $\pm$ std.) across 5 seeds, comparing standard training (No) vs.\ selective freezing of unsupervised modules (Yes).
}
\scriptsize
\renewcommand{\arraystretch}{1.15}
\setlength{\tabcolsep}{3.5pt}
\begin{tabular}{llcccccc}
\toprule
\textbf{Model} & \textbf{Freezing} &
\multicolumn{2}{c}{\textbf{\textsc{Insurance}}} &
\multicolumn{2}{c}{\textbf{\textsc{Hailfinder}}} &
\multicolumn{2}{c}{\textbf{\textsc{SIIM Pn.}}} \\
\cmidrule(lr){3-4}\cmidrule(lr){5-6}\cmidrule(lr){7-8}
& &
\textbf{Task Acc.} & \textbf{Concept Acc.} &
\textbf{Task Acc.} & \textbf{Concept Acc.} &
\textbf{Task Acc.} & \textbf{Concept Acc.} \\
\midrule
\multirow{2}{*}{C$^2$BM} & Yes  &
74.27 $\pm$ 2.67 & 76.08 $\pm$ 1.37 &
72.83 $\pm$ 2.35 & 57.81 $\pm$ 1.61 &
68.80 $\pm$ 2.87 & 67.68 $\pm$ 1.51 \\
& No &
74.47 $\pm$ 3.01 & 76.23 $\pm$ 1.44 &
71.70 $\pm$ 1.81 & 56.61 $\pm$ 1.45 &
65.90 $\pm$ 2.99 & 64.91 $\pm$ 3.07 \\
\midrule
\multirow{2}{*}{CEM} & Yes  &
70.97 $\pm$ 0.86 & 77.16 $\pm$ 0.80 &
70.15 $\pm$ 4.36 & 60.13 $\pm$ 1.43 &
66.33 $\pm$ 3.06 & 64.90 $\pm$ 1.80 \\
& No &
72.62 $\pm$ 2.27 & 76.92 $\pm$ 1.39 &
70.30 $\pm$ 4.24 & 59.34 $\pm$ 1.51 &
66.91 $\pm$ 2.90 & 65.88 $\pm$ 3.24 \\
\bottomrule
\end{tabular}
\label{tab:freeze_2}
\end{table}

\subsubsection{Sensitivity to the concept loss weight $\gamma$} \label{app:gamma_sensitivity}
We evaluate the sensitivity of our framework to the concept-loss weight $\gamma$, which controls the trade-off between downstream task supervision and concept supervision during local optimization. We vary $\gamma \in \{0.2, 0.5, 0.8\}$ and report both task accuracy and concept accuracy for two representative instantiations, C$^2$BM and CEM, on three representative datasets with different structural and concept-space complexity: \textsc{Asia}, \textsc{Alarm}, and \textsc{SIIM-Pneumothorax} (Figure~\ref{fig:gamma_ablation}).

Overall, the results show that performance is largely robust to the choice of $\gamma$. On \textsc{Asia} and \textsc{Alarm}, both task accuracy and concept accuracy remain nearly unchanged across the tested values for both model instantiations, indicating that the framework is not particularly sensitive to the exact balance between task and concept objectives in these settings. This suggests that the optimization remains stable even when concept supervision is given substantially different weights. A more visible effect appears on \textsc{SIIM-Pneumothorax}, which has a more challenging concept space. There, increasing $\gamma$ leads to a mild improvement in concept accuracy, particularly for C$^2$BM, while task accuracy exhibits a corresponding small decrease. This trend is expected: assigning more weight to the concept loss encourages the model to better align with concept supervision, at the cost of slightly reduced emphasis on the end task. Importantly, even in this more demanding setting, the variation remains moderate, confirming that our method does not rely on careful tuning of $\gamma$ to achieve strong performance.
\begin{figure}[t]
   \centering
   \includegraphics[width=0.9\linewidth]{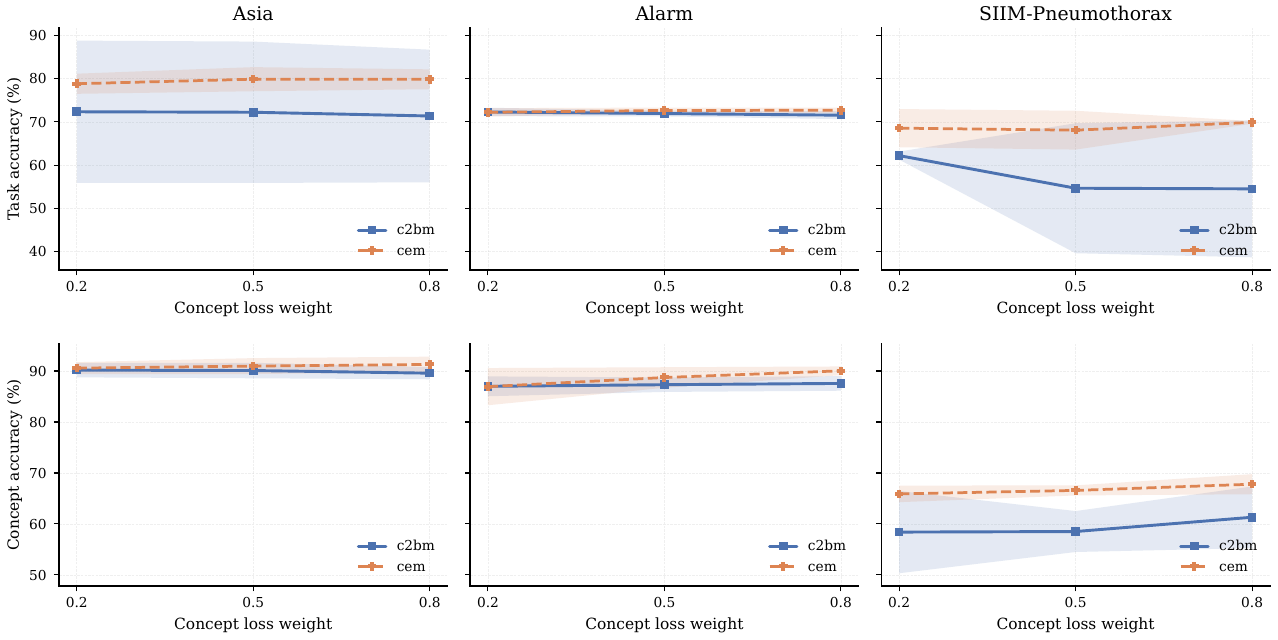}
   \caption{\textbf{Sensitivity analysis on the concept loss weight $\gamma$.}
   \emph{Task accuracy} (\%, top row) and \emph{concept accuracy} (\%, bottom row) as a function of $\gamma$ for the C$^2$BM and CEM instantiations on \textsc{Asia}, \textsc{Alarm}, and \textsc{SIIM-Pneumothorax}, averaged over three runs. Performance is largely stable across values of $\gamma$. On \textsc{SIIM-Pneumothorax}, increasing $\gamma$ yields a slight improvement in concept accuracy, accompanied by a small decrease in task accuracy, reflecting the expected trade-off between concept fidelity and downstream prediction.}
   \label{fig:gamma_ablation}
\end{figure}

\subsubsection{Sensitivity to Client Arrival Timing}
\label{app:arrival-timing}
To assess whether F-CM performance depends on the specific arrival schedule
used in the main experiments, we first vary the round at which new clients enter
the federation. Specifically, we consider client arrivals at $10\%$, $30\%$, $60\%$, and $90\%$ of the total training duration, as well as a setting with no new client arrivals. We report task accuracy across three random seeds for two representative instantiations, CEM and C$^2$BM, on the \textsc{Asia} and \textsc{Alarm} datasets.

\begin{table}[t]
    \centering
    \caption{
        \textbf{Sensitivity to client arrival timing.}
        Task accuracy (mean $\pm$ std.) for F-CMs instantiated with CEM and C$^2$BM when
        new clients arrive after different communication rounds.
        ``No arr.'' denotes the setting without client arrival.
        Results are reported over three random seeds.
    }
    \label{tab:arrival-timing}
    {
    \setlength{\tabcolsep}{3pt}
    \resizebox{\linewidth}{!}{%
    \begin{tabular}{@{}lccccc|ccccc@{}}
        \toprule
        & \multicolumn{5}{c|}{\textsc{Asia}}
        & \multicolumn{5}{c}{\textsc{Alarm}} \\
        \cmidrule(lr){2-6}
        \cmidrule(lr){7-11}
        Model
        & $10\%$ & $30\%$ & $60\%$ & $90\%$ & No arr.
        & $10\%$ & $30\%$ & $60\%$ & $90\%$ & No arr. \\
        \midrule
        CEM
        & $79.3 \pm 1.3$
        & $79.9 \pm 1.2$
        & $79.8 \pm 0.9$
        & $77.8 \pm 2.5$
        & $67.7 \pm 12.9$
        & $73.6 \pm 0.2$
        & $72.6 \pm 0.7$
        & $70.4 \pm 1.4$
        & $70.3 \pm 1.5$
        & $55.1 \pm 17.3$ \\
        C$^2$BM
        & $79.6 \pm 1.3$
        & $79.8 \pm 1.0$
        & $79.8 \pm 0.9$
        & $73.8 \pm 6.0$
        & $63.5 \pm 11.5$
        & $73.7 \pm 0.4$
        & $72.8 \pm 0.1$
        & $72.5 \pm 1.0$
        & $71.0 \pm 2.3$
        & $54.7 \pm 17.6$ \\
        \bottomrule
    \end{tabular}%
    }}
\end{table}

Table~\ref{tab:arrival-timing} highlights the dependence of performance on the available adaptation horizon. On \textsc{Asia}, early and intermediate arrivals ($\leq 60\%$ of the total training duration) leave task accuracy largely unchanged, whereas delaying client arrival to $90\%$ reduces accuracy to $77.8\%$ and $73.8\%$ for CEM and C$^2$BM, respectively, while still outperforming the no-arrival setting on average.
On \textsc{Alarm}, all arrival configurations improve over no arrival, with gains of
up to $19$ percentage points for C$^2$BM and $18.5$
points for CEM. Overall, these results indicate
that F-CMs remain effective across different arrival schedules, while very
late arrivals naturally leave less time to incorporate the newly available
supervision.

\subsubsection{Sensitivity to Concept-Space Expansion}
\label{app:concept-expansion}

We next study how performance changes as the federation acquires
supervision for additional concepts. We consider two expansion
transitions: (i) from low to medium federation coverage, corresponding
to \(< 51\%\) and \(51\)--\(75\%\) of the concepts in the
ground-truth graph, respectively; and (ii) from medium to high coverage,
corresponding to \(51\)--\(75\%\) and \(>75\%\), respectively.
We use relative coverage ranges rather than fixed absolute numbers of
concepts to make the expansion levels comparable across datasets with
different graph sizes.

Table~\ref{tab:concept-expansion} reports the resulting task accuracy
for CEM and C\(^{2}\)BM instantiations on \textsc{Asia} and \textsc{Alarm}, together with the
\emph{structural concept coverage}. Unlike the standard concept coverage, this measure is model-independent
and quantifies the fraction of nodes in the ground-truth graph covered
by the union of the selected client subgraphs. While the coverage ranges define the expansion scenarios, structural concept coverage reports the actual coverage achieved within each range.

\begin{table}[t]
    \centering
    \caption{
        \textbf{Sensitivity to concept-space expansion.}
        Structural concept coverage (S.C.~Cov., \%) and task accuracy (T., \%) as
        additional concepts become available during federation growth.
        The coverage-range column denotes the target fraction of concepts
        from the ground-truth graph available to the federation.
        Results are reported over three random seeds.
    }
    \label{tab:concept-expansion}
    \small
    \setlength{\tabcolsep}{4pt}
    \begin{tabular}{clcc|cc}
        \toprule
        Coverage range & Method
        & \multicolumn{2}{c|}{\textsc{Asia}}
        & \multicolumn{2}{c}{\textsc{Alarm}} \\
        \cmidrule(lr){3-4}
        \cmidrule(lr){5-6}
        & & S.C.~Cov. & T. & S.C.~Cov. & T. \\
        \midrule

        \multirow{2}{*}{$<51\%$}
        & S-F-CMs (CEM)
        & $45.8 \pm 4.2$ & $71.3 \pm 7.8$
        & $37.0 \pm 3.3$ & $36.5 \pm 1.6$ \\
        & S-F-CMs (C$^2$BM)
        & $45.8 \pm 4.2$ & $66.7 \pm 12.3$
        & $37.0 \pm 3.3$ & $38.5 \pm 4.5$ \\

        \midrule

        \multirow{2}{*}{$51$--$75\%$}
        & F-CMs (CEM)
        & $62.5$ & $76.1 \pm 3.1$
        & $65.8 \pm 3.9$ & $49.7 \pm 12.3$ \\
        & F-CMs (C$^2$BM)
        & $62.5$ & $80.2 \pm 1.1$
        & $65.8 \pm 3.9$ & $47.6 \pm 13.4$ \\

        \midrule

        \multirow{2}{*}{$>75\%$}
        & F-CMs (CEM)
        & $100.0$ & $80.3 \pm 1.1$
        & $100.0$ & $73.7 \pm 0.7$ \\
        & F-CMs (C$^2$BM)
        & $100.0$ & $79.9 \pm 0.9$
        & $100.0$ & $73.0 \pm 0.4$ \\

        \bottomrule
    \end{tabular}
\end{table}

Results show an overall positive trend as concept
availability increases. On \textsc{Asia}, increasing concept coverage from
$45.8\%$ to $62.5\%$ produces moderate changes in task accuracy,
consistent with its smaller causal graph, for which this expansion
corresponds to only a few additional concepts. At full coverage, task
accuracy increases to $79.9$--$80.3\%$. On \textsc{Alarm}, where each
increase in coverage introduces more concepts in absolute terms, the trend
is more pronounced: task accuracy increases on average from  $36.5$--$38.5\%$ at
$37.0\%$ coverage to $47.6$--$49.7\%$ at $65.8\%$, albeit with higher variance, and to
$73.0$--$73.7\%$ at full coverage. Overall, these results show that F-CMs
can progressively exploit additional concept supervision as the federation
expands, with larger gains when growth introduces a larger amount of new
semantic information.


\subsection{Robustness to imperfect concept supervision} \label{app:d8}
We evaluate F-CMs under two complementary forms of imperfect concept
supervision: inter-client semantic disagreement, where the same concept may
have inconsistent meanings across clients (App.~\ref{app:semantic_disagreement}), and sample-level missing or noisy
concept annotations (App.~\ref{app:concept-annotation-noise}).

\subsubsection{Inter-client semantic disagreement} \label{app:semantic_disagreement}
To assess whether F-CMs remain reliable when different clients attach partially inconsistent meanings to the same concepts, we conducted an additional ablation simulating \emph{inter-client semantic disagreement}. The goal is to model scenarios in which the same concept label does not correspond to exactly the same underlying semantics across clients, e.g., due to annotation inconsistencies, protocol differences, or distributional shifts. Concretely, for each experiment we first select a subset of concepts and, for each selected concept, a subset of clients. Within those selected client--concept pairs, we then swap a fraction of the corresponding concept labels. In this way, the same concept may acquire conflicting label semantics across clients, while leaving the rest of the supervision unchanged. We consider two disagreement regimes: a \emph{mild} setting, where 30\% of concepts and 30\% of clients are selected, and a \emph{severe} setting, where 60\% of concepts and 60\% of clients are selected. For each regime, we vary the concept-swap probability in $\{30\%,60\%,90\%,100\%\}$.

Figure~\ref{fig:swap_concepts} reports label accuracy for two representative F-CM instantiations, the bipartite \textsc{F-CMs (CEM)} and the graph-based \textsc{F-CMs (C$^2$BM)}, on \textsc{Asia} and \textsc{SIIM-Pneumothorax}. Under mild disagreement, both variants remain highly stable: performance changes only marginally even as the swap probability increases, indicating that aggregation across the remaining consistent clients is sufficient to absorb limited semantic mismatch. Under severe disagreement, performance degrades more visibly as the swap probability grows, with the decline being particularly pronounced on the smaller \textsc{Asia} benchmark. On \textsc{SIIM-Pneumothorax}, the degradation is smoother, suggesting that larger or more redundant datasets provide some resilience even when semantic conflicts become widespread.

Overall, these results indicate that F-CMs are robust to \emph{moderate} inter-client semantic disagreement, but that widespread and systematic concept mismatch eventually harms downstream prediction, as expected. This highlights an important practical direction for future work: detecting client-level semantic inconsistencies before aggregation. Since F-CMs operate modularly, they naturally expose module-level update patterns that could be used to identify semantic or protocol mismatches and trigger targeted diagnostics or filtering.

\begin{figure}[t]
   \centering
   \includegraphics[width=0.8\linewidth]{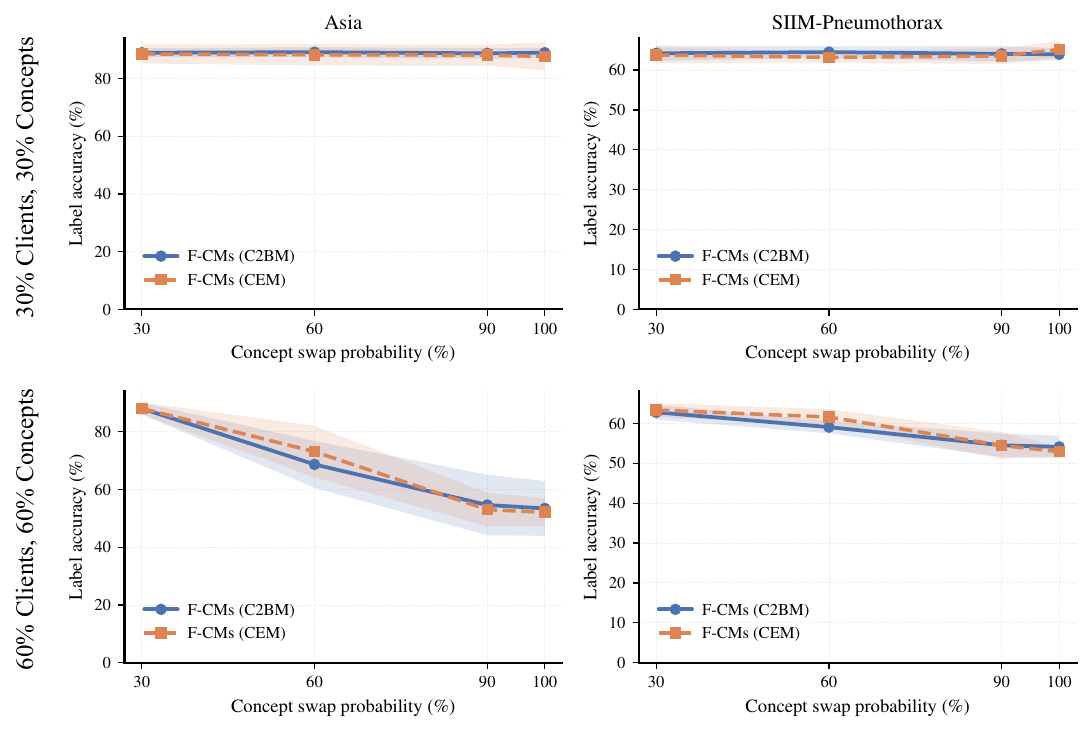}
   \caption{\textbf{Label accuracy (\%) under inter-client semantic disagreement.} Results for two representative F-CM instantiations, the bipartite \textsc{F-CMs (CEM)} and the graph-based \textsc{F-CMs (C$^2$BM)}, under increasing fractions of within-client concept-label swaps. Each row corresponds to a disagreement regime: 30\% of concepts and 30\% of clients selected (top), or 60\% of concepts and 60\% of clients selected (bottom). For each selected concept, the corresponding fraction of clients is chosen, and within them a fraction of the concept labels (x-axis: 30\%, 60\%, 90\%, 100\%) is swapped, simulating conflicting concept semantics across clients. Each column corresponds to a different dataset.}
   \label{fig:swap_concepts}
\end{figure}

\subsubsection{Missing and noisy concept annotations}
\label{app:concept-annotation-noise}

F-CMs already operate under heterogeneous and partial concept supervision,
with different clients supervising different subsets of concepts. We further
assess robustness to sample-level unreliability in the available concept
annotations by perturbing only the training annotations, while keeping all
other experimental information unchanged. We consider two complementary
settings: \textit{missing annotations}, where each available concept annotation
is independently masked with probability
$\mu \in \{0.3, 0.6, 0.9\}$, and \textit{noisy annotations}, where each
available annotation is independently corrupted with probability
$\eta \in \{0.3, 0.6, 0.9\}$ (e.g., by label swapping for binary concepts).

We report task and concept accuracy for CEM and C$^2$BM on \textsc{Asia} and \textsc{Alarm},
averaged over three random seeds. Since these perturbations modify the
sample-level annotations while leaving the client concept sets unchanged,
concept coverage is unaffected by construction. 

\begin{table}[t]
    \centering
    \caption{
        \textbf{Robustness to missing concept annotations.}
        Task/concept accuracy (\%, mean $\pm$ std.) under increasing
        probabilities $\mu$ of masking available concept annotations.
        Results are averaged over three random seeds.
    }
    \label{tab:missing-concept-annotations}
    \setlength{\tabcolsep}{3pt}
    \resizebox{\linewidth}{!}{%
    \begin{tabular}{llcccc}
        \toprule
        Dataset & Model & Clean ref. & $\mu=0.3$ & $\mu=0.6$ & $\mu=0.9$ \\
        \midrule
        \multirow{2}{*}{\textsc{Alarm}}
        & C$^2$BM
        & $71.8\pm0.9 / 89.9\pm0.5$
        & $72.4\pm0.7 / 89.7\pm0.2$
        & $72.6\pm0.3 / 89.1\pm0.1$
        & $70.1\pm1.1 / 85.1\pm0.9$ \\
        & CEM
        & $69.9\pm2.8 / 90.2\pm0.3$
        & $62.3\pm9.6 / 90.0\pm0.4$
        & $58.1\pm9.5 / 89.4\pm0.2$
        & $56.8\pm9.2 / 86.8\pm0.3$ \\
        \midrule
        \multirow{2}{*}{\textsc{Asia}}
        & C$^2$BM
        & $81.0\pm1.6 / 91.2\pm0.6$
        & $81.1\pm1.7 / 91.0\pm0.8$
        & $81.1\pm1.9 / 90.7\pm0.6$
        & $81.2\pm1.8 / 89.4\pm0.7$ \\
        & CEM
        & $81.3\pm1.4 / 91.0\pm0.7$
        & $81.2\pm1.5 / 90.9\pm0.6$
        & $81.3\pm1.6 / 90.7\pm0.8$
        & $81.0\pm1.5 / 89.2\pm1.0$ \\
        \bottomrule
    \end{tabular}%
    }
\end{table}

Table~\ref{tab:missing-concept-annotations} shows substantial robustness to
missing supervision. On \textsc{Asia}, even masking $90\%$ of the available concept
annotations leaves task accuracy essentially unchanged for both architectures,
while concept accuracy decreases by only $1.8$ percentage points. On \textsc{Alarm},
C$^2$BM remains similarly robust, with task accuracy decreasing by only
$1.7$ points at $\mu=0.9$, whereas CEM is more sensitive, decreasing from
$69.9\%$ to $56.8\%$. This suggests that structural conditioning can provide
useful support when concept supervision becomes particularly sparse.

\begin{table}[t]
    \centering
    \caption{
        \textbf{Robustness to noisy concept annotations.}
        Task/concept accuracy (\%, mean $\pm$ std.) under increasing
        probabilities $\eta$ of corrupting available concept annotations.
        Results are averaged over three random seeds.
    }
    \label{tab:noisy-concept-annotations}
    \setlength{\tabcolsep}{3pt}
    \resizebox{\linewidth}{!}{%
    \begin{tabular}{llcccc}
        \toprule
        Dataset & Model & Clean ref. & $\eta=0.3$ & $\eta=0.6$ & $\eta=0.9$ \\
        \midrule
        \multirow{2}{*}{\textsc{Alarm}}
        & C$^2$BM
        & $71.4\pm0.9 / 89.9\pm0.5$
        & $71.7\pm0.5 / 87.9\pm0.4$
        & $72.1\pm0.9 / 46.7\pm2.4$
        & $51.7\pm9.6 / 14.9\pm2.2$ \\
        & CEM
        & $69.9\pm2.8 / 90.2\pm0.3$
        & $73.1\pm0.6 / 89.1\pm0.6$
        & $71.6\pm1.3 / 47.1\pm2.7$
        & $50.5\pm12.2 / 7.2\pm0.7$ \\
        \midrule
        \multirow{2}{*}{\textsc{Asia}}
        & C$^2$BM
        & $81.0\pm1.6 / 91.2\pm0.6$
        & $81.3\pm1.3 / 89.3\pm1.0$
        & $64.1\pm8.4 / 30.6\pm9.4$
        & $63.1\pm7.4 / 24.9\pm5.9$ \\
        & CEM
        & $81.3\pm1.4 / 91.0\pm0.7$
        & $81.7\pm1.4 / 88.6\pm1.0$
        & $81.3\pm1.2 / 16.1\pm1.8$
        & $72.5\pm7.2 / 15.5\pm1.8$ \\
        \bottomrule
    \end{tabular}%
    }
\end{table}

Table~\ref{tab:noisy-concept-annotations} shows that moderate annotation noise
is also well tolerated: at $\eta=0.3$, task accuracy remains within $3.2$
percentage points of the clean reference across all settings, while concept
accuracy decreases by only $1.1$--$2.4$ points. At $\eta=0.6$, concept
accuracy deteriorates substantially, although task accuracy remains close to
the clean reference for both \textsc{Alarm} models and \textsc{Asia}--CEM; \textsc{Asia}--C$^2$BM is
more sensitive in this regime. At $\eta=0.9$, both concept and task
performance degrade more consistently. Overall, noisy annotations are more
harmful than missing ones, as incorrect labels introduce contradictory
supervision rather than simply reducing the amount of available supervision.

\subsection{Robustness to graph perturbations}
\label{app:robust_graph_perturb}
In this section, we evaluate the sensitivity of the DAG aggregation procedure (Sec.~\ref{sec:graph_agg}) to noise in client-provided graphs, which can arise both from limited local data (estimation error) and from adversarial behavior (e.g., poisoning the communicated structure) \cite{sunDataPoisoning2022, fenoglioFBP2024}. To reduce variance due to a specific client partition, we run this experiment with a large federation and high per-round participation. Concretely, training follows the same temporal non-stationarity setup as in the main paper, but with $|\mathcal{K}_t|=100$ participating clients per round. The client pool is time-varying: during a warm-up phase the federation contains $|\mathcal{K}(t)|=100$ clients, and at the predefined join round (as in our standard protocol, depending on the dataset) an additional cohort joins, expanding the pool to $|\mathcal{K}(t)|=200$. Late-arriving clients may introduce new concepts and associated edges, requiring the aggregated shared structure to expand accordingly. 

Each participating client provides a local graph (binary adjacency matrix) over its observed variables. We inject noise by randomly deleting, adding, or reversing edges, while enforcing acyclicity. We vary two factors: \emph{(\% client alteration)}, the fraction of clients whose graphs are corrupted, and \emph{(graph alteration)}, the fraction of edges modified within each corrupted client graph. Aggregation quality is measured by \emph{DiffPairs}, i.e., the number of node pairs whose relative ordering differs between the aggregated shared DAG and the reference DAG (lower is better). Figure~\ref{fig:graph_robustness} reports results for the federated regimes (F-CMs and S-F-CMs) and includes localized training (Loc.) as a reference.

\paragraph{Empirical Validation (Graph Quality).}
Overall, the proposed server-side graph aggregation is robust to substantial client-side noise. In particular, for moderate perturbations (e.g., 30\% graph alteration), the aggregated structure remains essentially unchanged across the full range of \% client alteration: DiffPairs stays nearly constant for both F-CMs and S-F-CMs. This indicates that the aggregation effectively denoises inconsistent local proposals when the per-client corruption level is limited. In contrast, the Localized baseline exhibits a markedly different behavior: for a fixed graph alteration level, DiffPairs increases approximately linearly with \% client alteration, reflecting the lack of cross-client averaging and the direct propagation of local graph errors to the deployed structure.

At higher corruption levels, DiffPairs increases as expected. For 60\% graph alteration, degradation becomes visible once a large majority of clients are corrupted (typically around 60--70\% corrupted clients across datasets). For 90\% graph alteration, the effect appears earlier, becoming pronounced already around 50\% corrupted clients. Finally, the gap between F-CMs and S-F-CMs is primarily explained by \emph{structure evolution}: S-F-CMs does not update the shared graph when late-arriving clients introduce new concepts and edges, leading to an increasingly partial global view, whereas F-CMs continuously integrates new proposals and maintains coverage of the evolving federation.
\begin{figure}
   \centering   \includegraphics[width=\linewidth]{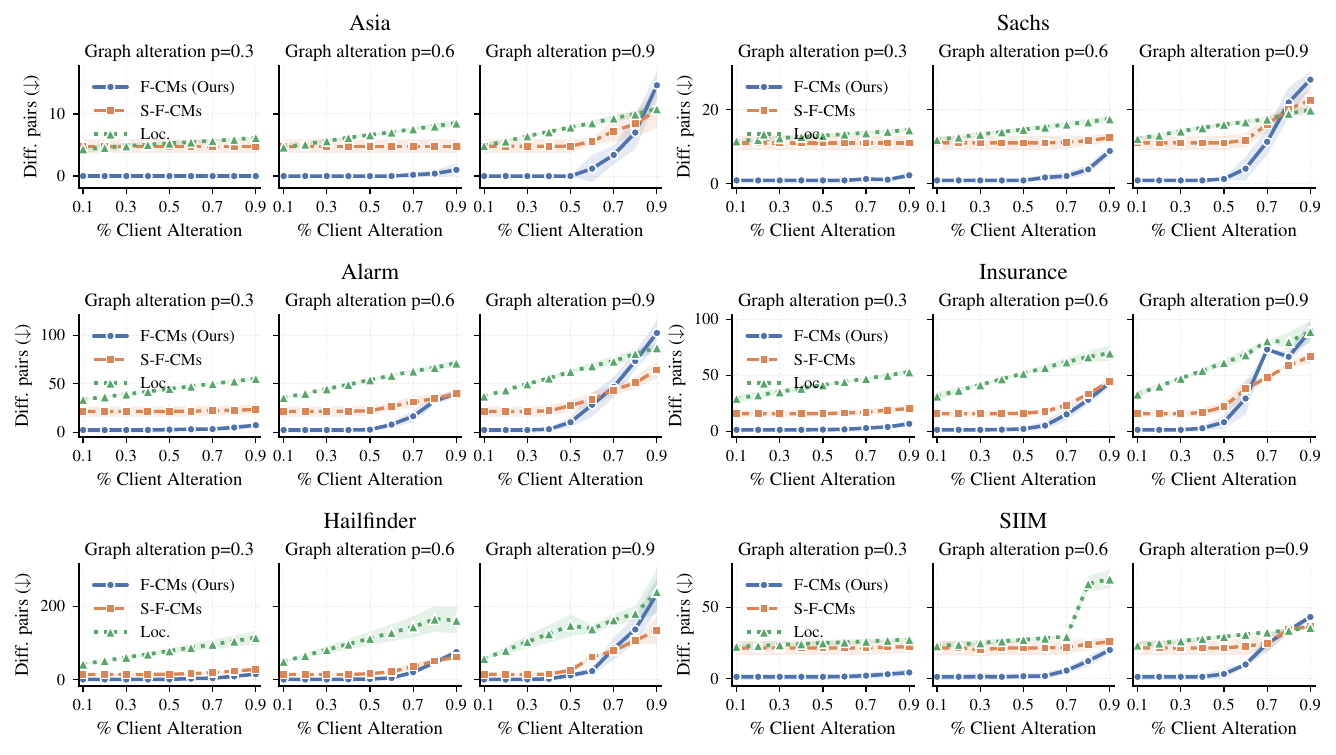}   \caption{\textbf{Robustness of DAG aggregation to client graph perturbations.} DiffPairs (lower is better) between the aggregated DAG and the ground-truth DAG as a function of \% client alteration and per-client graph alteration.}
\label{fig:graph_robustness}
\end{figure}

\paragraph{Empirical Validation (Downstream Accuracy).}
To directly assess whether structural errors in the aggregated graph translate into degraded downstream performance, we ran an additional experiment under the same perturbation protocol and measured the final task accuracy of C$^2$BM, used here as a representative graph-based instantiation. Table~\ref{tab:graph_robustness_task_acc} reports results for the setting with 60\% graph alteration as the fraction of corrupted clients increases.

The results show that downstream performance is substantially more robust than graph quality alone might suggest. On \textsc{SIIM-Pneumothorax}, F-CMs changes only from $68.80\pm2.51$ at 0\% corrupted clients to $67.06\pm2.90$ at 90\% corrupted clients, despite the visible degradation in DiffPairs at high corruption levels. On \textsc{Asia}, task accuracy remains essentially unchanged around 80\% across the full corruption range. S-F-CMs remain roughly stable, but at a substantially lower absolute accuracy level, consistent with the static-structure limitation discussed above. Overall, these results indicate that increases in DiffPairs do not imply comparable drops in downstream task accuracy: F-CMs remain robust to moderate structural misspecification, and graph degradation becomes noticeable earlier than any meaningful loss in predictive performance.
\begin{table}[t]
\centering
\caption{
\textbf{Downstream robustness to graph misspecification.}
Task accuracy (\%, mean $\pm$ std.) for C$^2$BM under the same perturbation protocol as Figure~\ref{fig:graph_robustness}, with \emph{graph alteration} fixed at 60\% and varying fractions of corrupted clients. Although graph quality degrades at higher corruption levels, downstream task accuracy remains largely stable, especially for F-CMs.
}
\scriptsize
\setlength{\tabcolsep}{4pt}
\begin{tabular}{llcccc}
\toprule
\textbf{Dataset} & \textbf{Method} & \textbf{0\%} & \textbf{30\%} & \textbf{60\%} & \textbf{90\%} \\
\midrule
\multirow{2}{*}{\textsc{Asia}}
 & F-CMs (C$^2$BM)   & 80.35$\pm$2.16 & 80.41$\pm$1.80 & 80.25$\pm$1.98 & 80.39$\pm$2.51 \\
 & S-F-CMs (C$^2$BM) & 59.73$\pm$20.19 & 59.99$\pm$20.15 & 59.99$\pm$20.15 & 60.83$\pm$21.46 \\
\midrule
\multirow{2}{*}{\textsc{SIIM-Pneumothorax}}
 & F-CMs (C$^2$BM)   & 68.80$\pm$2.51 & 68.03$\pm$3.34 & 68.62$\pm$2.62 & 67.06$\pm$2.90 \\
 & S-F-CMs (C$^2$BM) & 62.68$\pm$12.67 & 62.13$\pm$1.01 & 62.14$\pm$1.02 & 63.07$\pm$2.58 \\
\bottomrule
\end{tabular}
\label{tab:graph_robustness_task_acc}
\end{table}

\subsection{Ablation of graph aggregation strategies under graph perturbations}
\label{app:graph_agg_ablation}
To complement the robustness analysis in App.~\ref{app:robust_graph_perturb}, we further compare different server-side graph aggregation strategies under the same noisy FL setting. In particular, we contrast our proposed aggregation with standard alternatives based on \textsc{Union} \cite{graph_aggregation2017}, \textsc{Intersection} \cite{graph_aggregation2017}, and \textsc{Bayesian Network structure fusion} (BN fusion) \cite{puertaBNfusion2021}, and evaluate each of them under three levels of structural visibility: \emph{localized} training (\textsc{Loc.}), \emph{static federated} visibility (\textsc{S-*}), and \emph{dynamic federated} visibility, where the shared graph is continuously updated as new clients join and introduce additional concepts and dependencies. As in App.~\ref{app:robust_graph_perturb}, all experiments follow the same temporal non-stationarity protocol, with large federation size and high participation per round, and performance is measured with \emph{DiffPairs} (lower is better). We additionally assess whether reliability-aware client weighting can further improve robustness within our aggregation rule.

Figures~\ref{fig:graph_ablation_0303} and~\ref{fig:graph_ablation_0606} report the results for two perturbation regimes, corresponding to moderate corruption (\emph{graph alteration} $p=0.3$, \emph{\% client alteration} $=0.3$) and stronger corruption (\emph{graph alteration} $p=0.6$, \emph{\% client alteration} $=0.6$). The color coding highlights the visibility regime: blue for dynamic training, orange for static federated visibility, and green for localized visibility.

\paragraph{Comparison of aggregation strategies.}
The results highlight three main trends. First, our aggregation is consistently the most robust across both perturbation regimes. In all datasets and in both Figures~\ref{fig:graph_ablation_0303} and~\ref{fig:graph_ablation_0606}, \textsc{F-CMs} aggregation achieves the lowest DiffPairs, while \textsc{S-F-CMs} is consistently the strongest static alternative. Moreover, \textsc{S-F-CMs} always improves over \textsc{Loc.}, and the dynamic version further improves over the static one on all datasets, confirming that our aggregation can exploit the broader structural visibility provided by federation and by late-arriving clients.

Second, broader visibility is \emph{not} uniformly beneficial across aggregation rules. While moving from local to static federated visibility can help some baselines, the effect is not systematic, and in several cases local training remains preferable to federated aggregation with naive fusion rules. This is particularly visible for \textsc{Union}, and to a lesser extent for \textsc{Intersection}. In other words, access to more client-side graphs does not automatically improve the recovered shared structure: when the aggregation rule is poorly matched to the underlying heterogeneity, the additional structural proposals may introduce more inconsistencies than useful signal. This effect becomes even clearer when comparing static and dynamic visibility. For \textsc{Union} and \textsc{BN fusion}, dynamic visibility often increases DiffPairs, in some cases substantially, especially under stronger corruption (Figure~\ref{fig:graph_ablation_0606}). A plausible explanation is that these rules do not sufficiently denoise discordant client proposals. By contrast, our aggregation operates edge-wise, considers only clients that observe both nodes, and selects among $\{\mathcal{V}_i\!\to\!\mathcal{V}_j,\; \mathcal{V}_j\!\to\!\mathcal{V}_i,\; \varnothing\}$, allowing the server to reject weak or inconsistent evidence instead of forcing unreliable edges into the aggregate. As more clients are observed, including corrupted ones and late-arriving clients with partially different structural patterns, \textsc{Union}-style aggregation becomes overly permissive, accumulating spurious edges and incorrect orderings, while \textsc{BN fusion} appears similarly sensitive to conflicting local structures. \textsc{Intersection}, by contrast, is more conservative and therefore more stable, but its gains from dynamic visibility are limited.

Finally, these results clarify that the benefit of global visibility depends on the aggregation mechanism. More visibility is useful only when the fusion rule can transform heterogeneous local proposals into a coherent shared DAG. Our method does so reliably: unlike alternative graph fusion strategies, it benefits monotonically from richer structural visibility and degrades more gracefully as both the fraction of corrupted clients and the per-client graph alteration increase. Overall, the ablation reinforces the central claim of this work: in evolving federated concept-based learning, global structural visibility is valuable only when coupled with a robust, structure-aware aggregation rule.

\vspace{-5pt}
\begin{figure}[H]
   \centering
   \includegraphics[width=0.74\linewidth]{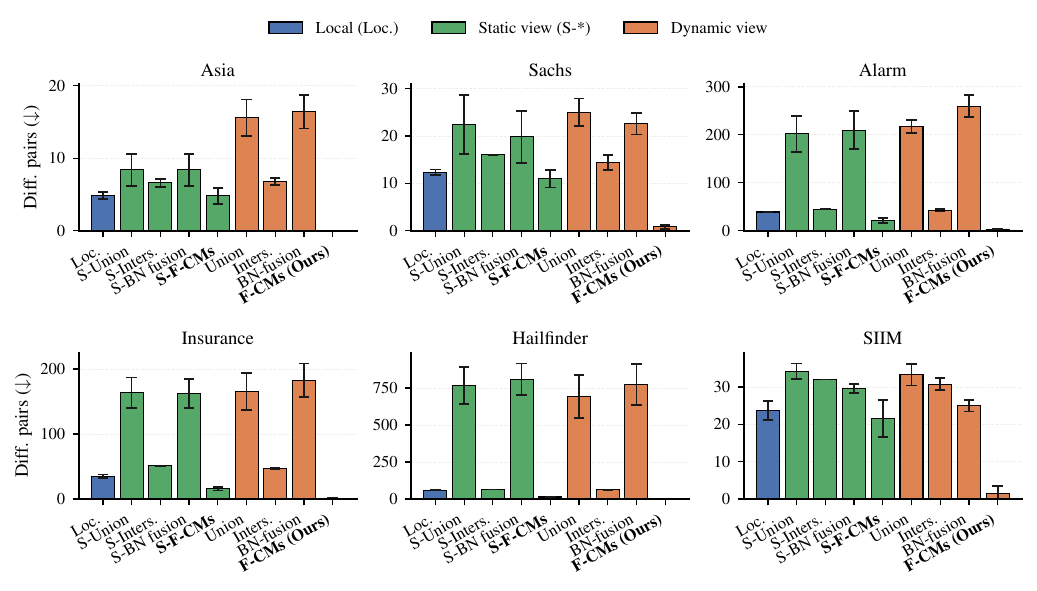}
   \vspace{-8pt}
   \caption{\textbf{Ablation of graph aggregation strategies under moderate perturbations.}
   DiffPairs (lower is better) for different aggregation rules under localized, static federated, and dynamic federated visibility, with graph alteration $p=0.3$ and 30\% corrupted clients.}
   \vspace{-10pt}
\label{fig:graph_ablation_0303}
\end{figure}

\vspace{-10pt}
\begin{figure}[H]
   \centering
   \includegraphics[width=0.74\linewidth]{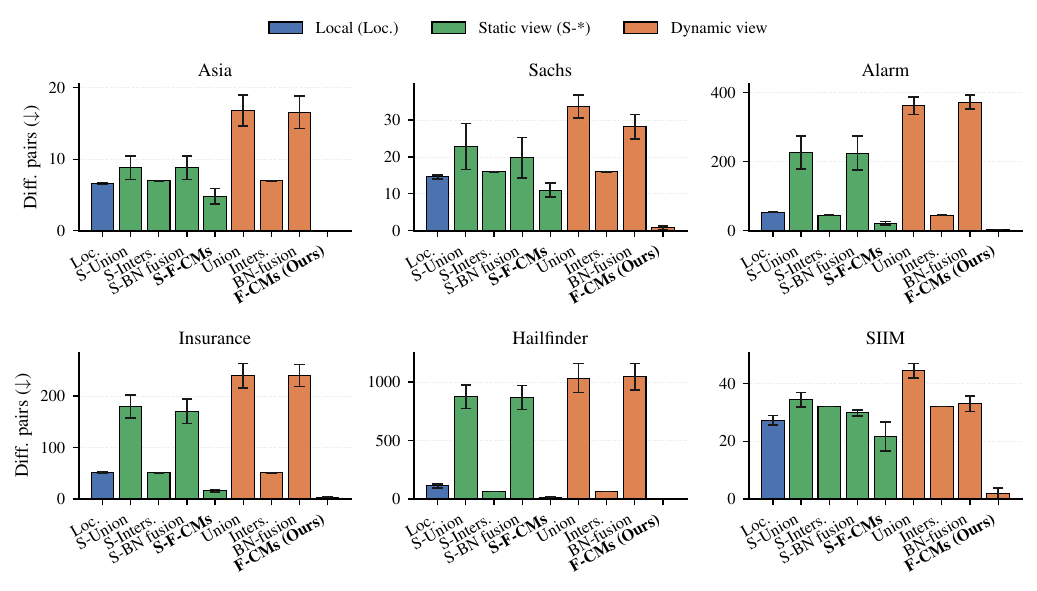}
   \vspace{-8pt}
   \caption{\textbf{Ablation of graph aggregation strategies under stronger perturbations.}
   DiffPairs (lower is better) for different aggregation rules under localized, static federated, and dynamic federated visibility, with graph alteration $p=0.6$ and 60\% corrupted clients.}
\label{fig:graph_ablation_0606}
\end{figure}

\paragraph{Reliability-aware client weighting.}
Our aggregation rule supports generic client weights $\alpha_k$ (Sec.~\ref{sec:graph_agg}).
In our experiments, following the standard FedAvg weighting scheme, we use
data-size-based weighting, i.e., clients are weighted proportionally to their
local dataset sizes. More generally, F-CMs can incorporate alternative
weighting strategies.  We therefore design a consensus-based weighting strategy
that estimates client reliability from the agreement among their reported
structures.

For each client $k$ and informative variable pair $\{V_i,V_j\}$ observed by
the client, we define the pairwise agreement score $a_{k,ij}$ as the fraction
of other clients observing the same pair that report the same relation as
client $k$. We consider a pair informative when it is observed by at least
two clients and at least one of them predicts an edge, thereby preventing
absent edges from dominating the agreement score. The consensus weight of
client $k$ is then defined as
\[
\alpha_k^{\mathrm{CB}}
=
\frac{1}{|\mathcal{I}_k|}
\sum_{\{V_i,V_j\}\in\mathcal{I}_k}
a_{k,ij},
\]
where $\mathcal{I}_k$ denotes the set of informative pairs available to
client $k$. Clients whose structural predictions systematically agree with
the federation therefore receive greater weight during graph aggregation.

\begin{table}[t]
    \centering
    \caption{
        \textbf{Consensus-based client weighting under graph perturbations.}
        DiffPairs (lower is better) for consensus-based and
        data-size-based graph aggregation under varying client
        alteration (CA) and graph alteration (GA).
    }
    \label{tab:consensus-weighting}
    \setlength{\tabcolsep}{3pt}
    \resizebox{0.7\linewidth}{!}{%
    \begin{tabular}{lc|ccc|ccc}
        \toprule
        & & \multicolumn{3}{c|}{Consensus-based aggregation}
          & \multicolumn{3}{c}{Data-size-based aggregation} \\
        Dataset & CA
        & GA$=0.3$ & GA$=0.6$ & GA$=0.9$
        & GA$=0.3$ & GA$=0.6$ & GA$=0.9$ \\
        \midrule
        \multirow{3}{*}{\textsc{Asia}}
        & $0.3$ & $0.0\pm0.0$ & $0.0\pm0.0$ & $0.0\pm0.0$ & $0.0\pm0.0$ & $0.0\pm0.0$ & $0.0\pm0.0$ \\
        & $0.6$ & $0.0\pm0.0$ & $0.0\pm0.0$ & $0.3 \pm 0.6$ & $0.0\pm0.0$ & $0.0\pm0.0$ & $1.0 \pm 1.7$ \\
        & $0.9$ & $0.0\pm0.0$ & $0.3\pm0.6$ & $14.3\pm1.2$
                  & $0.0\pm0.0$ & $1.0\pm1.0$ & $14.3\pm1.2$ \\
        \midrule
        \multirow{3}{*}{\textsc{Alarm}}
        & $0.3$ & $1.7\pm1.5$ & $1.7\pm1.5$ & $1.7\pm1.5$
                  & $1.7\pm1.5$ & $1.7\pm1.5$ & $1.7\pm1.5$ \\
        & $0.6$ & $2.0\pm1.0$ & $2.7\pm0.6$ & $7.0\pm2.7$
                  & $2.3\pm1.5$ & $7.7\pm2.3$ & $26.7\pm13.1$ \\
        & $0.9$ & $6.3\pm1.5$ & $39.7\pm3.5$ & $90.0\pm16.5$
                  & $7.0\pm1.7$ & $41.7\pm2.1$ & $93.7\pm11.0$ \\
        \bottomrule
    \end{tabular}%
    }
\end{table}

Table~\ref{tab:consensus-weighting} shows that both weighting strategies are
highly robust under mild structural perturbations, while consensus-based
weighting (CB) provides additional resilience as corruption increases. On \textsc{Asia},
the two strategies are nearly indistinguishable, with CB reducing DiffPairs
from $1.0$ to $0.3$ only at CA$=0.9$, GA$=0.6$. The difference is more
pronounced on \textsc{Alarm}: at CA$=0.6$, GA$=0.9$, CB reduces DiffPairs from
$26.7$ to $7.0$, while smaller improvements are observed across most other
corrupted settings. Under the most extreme regime (CA$=0.9$, GA$=0.9$),
both methods eventually degrade substantially. Overall, consensus-based
weighting performs comparably to data-size weighting in low-noise regimes
and can improve robustness when a larger fraction of client structures is
unreliable, by reducing the influence of systematically discordant reports.

\subsection{Multimodal extension}
\label{app:multimodal}

We extend F-CMs to multimodal federated settings, where different clients may observe different
input modalities while the semantic prediction targets, i.e., concepts and task labels, are shared
across modalities. Fig.~\ref{fig:multi} provides an overview. 

To experimentally validate this extension, we consider \textsc{CheXpert
Plus}~(App.~\ref{sec:chexpert_plus}), which augments \textsc{CheXpert} (App.~\ref{sec:chexpert})
with radiology reports. In what follows, we specify the modified problem setting, architecture, local optimization, aggregation and experimental
setup.

\paragraph{Multimodal problem setting.}
Let $\mathcal{A}$ denote the set of available modalities. In \textsc{CheXpert Plus}, we consider the image
modality $I$ (chest x-ray images) and the text modality $T$ (radiology reports). Each client $k$ observes data from a single modality
$a^{(k)} \in \mathcal{A}$ and holds a local dataset
\[
D^{(k)} =
\left\{
\left(x^{(k,a^{(k)})}_i, \mathbf{v}^{(k)}_i\right)
\right\}_{i=1}^{n^{(k)}},
\]
where $x^{(k,a^{(k)})}_i \in \mathcal{X}^{a^{(k)}}$ is the input observed under modality
$a^{(k)}$, and $\mathbf{v}^{(k)}_i$ contains the supervision available for example $i$. As in the unimodal
case, each client supervises a subset of interpretable variables $\mathcal{V}^{(k)} $, which may include concepts,
task variable, or both.
Importantly, concepts and task labels are treated as modality-independent interpretable targets. Each
concept $C_j$ and the task $Y$ are assumed to be predictable, in principle, from any modality
$a \in \mathcal{A}$. Thus, the same clinical concept, such as \emph{Cardiomegaly}, can be predicted
from a chest X-ray image or from a radiology report.

\paragraph{Multimodal F-CM shared model.}
The main architectural change is that the single input encoder is replaced by a collection of
modality-specific encoders. For each modality $a \in \mathcal{A}$, we define an encoder
\[
g_{\theta^a_t}^a : \mathcal{X}^a \rightarrow \mathbb{R}^{d_a},
\]
which maps inputs from modality $a$ to modality-specific embeddings. Since different modalities
produce embeddings with different statistical and geometric properties, we introduce a
modality-alignment module
\[
q_{\eta^a_t}^a : \mathbb{R}^{d_a} \rightarrow \mathbb{R}^{d},
\]
which maps modality-specific embeddings into a shared latent space. For an example observed by client $k$, the shared
representation is therefore
\[
z_i^{(k)} =
q_{\eta^{a^{(k)}}_t}^{a^{(k)}}
\left(
g_{\theta^{a^{(k)}}_t}^{a^{(k)}}
\left(x^{(k,a^{(k)})}_i\right)
\right)
\in \mathbb{R}^{d}.
\]
where $a^{(k)}$ denotes the modality observed by client $k$. Thus,
$g_{\theta^{a^{(k)}}_t}^{a^{(k)}}$ and $q_{\eta^{a^{(k)}}_t}^{a^{(k)}}$ are respectively the encoder
and alignment module associated with that modality.
In our implementation, the image encoder and the text encoder are instantiated as reported in App.~\ref{sec:chexpert_plus}, while each modality-specific alignment module is
parameterized as an MLP.

The concept and task modules are the same as in the unimodal case and are shared across
modalities. Therefore, at round $t$, the multimodal F-CM is characterized by the parameters
\[
\mathbf{w}'_t =
\left(
\{\theta^a_t\}_{a \in \mathcal{A}},
\{\eta^a_t\}_{a \in \mathcal{A}},
\phi_t,
\psi_t
\right).
\]

\paragraph{Graph aggregation and architecture adaptation.}
The semantic graph remains modality-invariant. Client graphs encode dependencies among
semantic variables, not among raw input modalities. Therefore, graph aggregation proceeds exactly
as in the unimodal case. Dynamic architecture adaptation is also
preserved. The only additional mechanism introduced by
the multimodal extension is expansion along the modality axis: if a new modality appears at round
$t$, the server initializes a new modality-specific encoder $g^a_{\theta^a_t}$ and alignment module
$q^a_{\eta^a_t}$, while reusing the existing semantic modules whenever possible.

\paragraph{Client-side optimization.}
At each round $t$, the server broadcasts the current global model to the selected clients. Each
client $k$ processes its local inputs through the modality-specific encoder and alignment module
associated with its modality $a^{(k)}$. Client-side optimization then follows the same module-specific
principle as in the unimodal setting: the client updates only the trainable modules for which it provides supervision. Concept and task modules
without local supervision remain frozen.

\paragraph{Server-side aggregation.}
Aggregation is performed module-wise. For each modality $a$, the server aggregates the encoder and alignment
module only over clients that observe modality $a$, using the same weighted averaging rule as in
the unimodal setting. If no selected client observes modality $a$ at round $t$, the corresponding
alignment module is left unchanged.

Concept and task modules are aggregated as in unimodal F-CMs. In particular, each concept module
is aggregated only over clients that provide supervision for the corresponding concept, while the
task module is aggregated only over clients that provide task supervision.

Notably, this section is intended as a first proof-of-concept extension of F-CMs to multimodal federated settings. Further dedicated multimodal alignment objectives could be explored, for instance by adding an explicit regularization term that encourages the latent representations of different modalities to be closer in the shared embedding space.

\begin{figure}[H]
   \centering
   \includegraphics[width=0.65\linewidth]{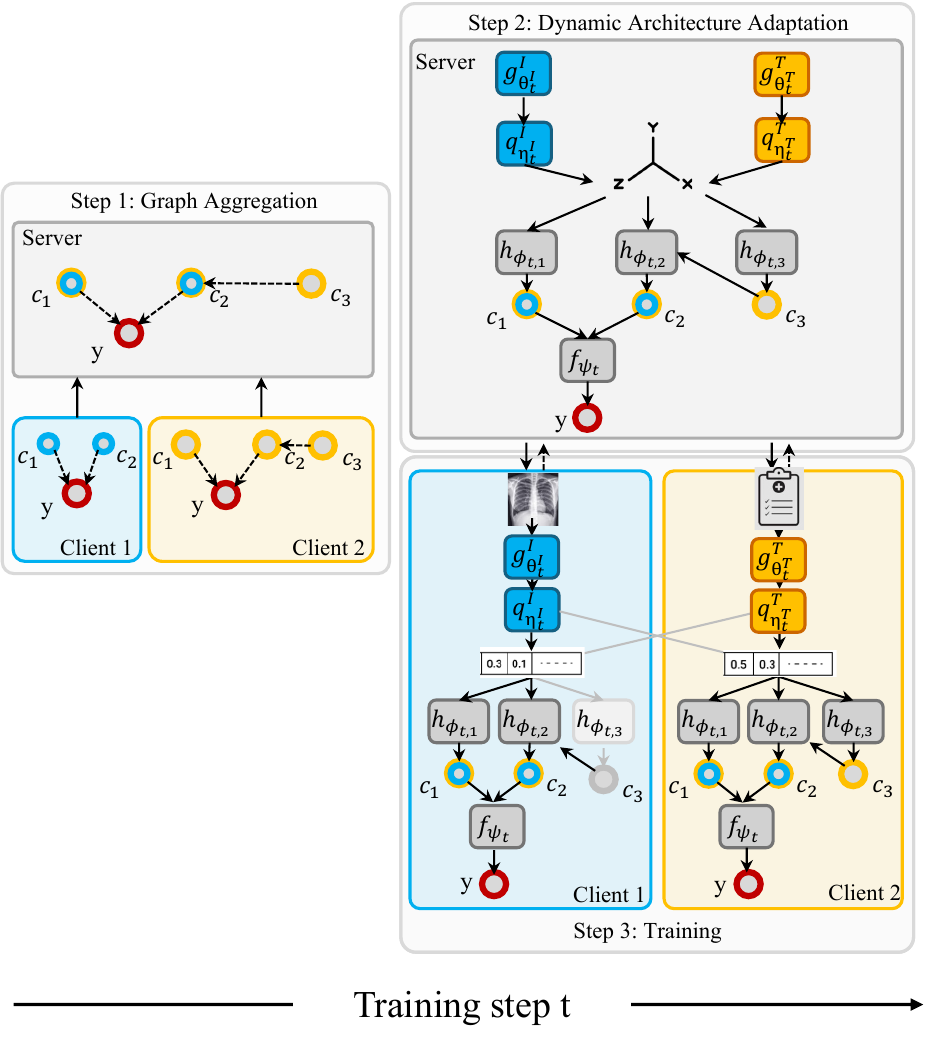}
   \caption{\textbf{Multimodal F-CMs overview}. At each round $t$, multimodal F-CMs follow the same steps as standard F-CMs: graph aggregation, dynamic architecture adaptation, and training. To handle heterogeneous input modalities, the shared CM includes, for each modality, a modality-specific encoder and a modality-alignment module, shown as the blue and yellow modules in the figure. The modality-specific encoder produces modality-specific embeddings, while the modality-alignment module maps them into a common latent space. Each client then processes its local inputs using only the encoder and alignment module associated with its modality, obtaining a latent representation that is then used to predict concepts and the task label.}
\label{fig:multi}
\end{figure}

\paragraph{Experimental evaluation.}
We empirically evaluate the multimodal extension on \textsc{CheXpert Plus} by measuring task accuracy and concept
coverage under the same training regimes and experimental setup used in the unimodal evaluation. Table~\ref{tab:chexpert_plus}
reports results for the most relevant baselines, namely the centralized, S-F-CMs, and 
F-CMs instantiations.

The results follow the same trend observed in the unimodal setting.
S-F-CMs aggregate client updates but keep the semantic architecture fixed, resulting in limited
coverage when later clients introduce additional concept supervision. In contrast, F-CMs dynamically
adapt the model structure and recover full coverage of task-relevant concepts across all
architectures.

\begin{table}[H]
\centering
\caption{
\textbf{Task accuracy} (\%) and \textbf{concept coverage} (\%) with respect to the ground-truth concepts
causally relevant for the task.
}
\scriptsize
\renewcommand{\arraystretch}{1.2}
\setlength{\tabcolsep}{4pt}
\begin{tabular}{lcccccc}
\toprule
\textbf{} &
\multicolumn{2}{c}{\textbf{Cent.}} &
\multicolumn{2}{c}{\textbf{S-F-CMs}} &
\multicolumn{2}{c}{\textbf{F-CMs}} \\
\cmidrule(lr){2-3} \cmidrule(lr){4-5} \cmidrule(lr){6-7}
\textbf{} &
\textbf{Task Acc.} & \textbf{Concept Cov.} &
\textbf{Task Acc.} & \textbf{Concept Cov.} &
\textbf{Task Acc.} & \textbf{Concept Cov.}
\\
\midrule
OpaqNN  &
72.9 $\pm$ 0.5 & 100 &
58.4 $\pm$ 2.6 & 52.0 $\pm$ 10.2
&
70.3 $\pm$ 2.8 & 100
\\
C$^2$BM &
73.3 $\pm$  0.6 & 100 &
58.4 $\pm$ 2.5 & 60.0 $\pm$  8.2
&
66.1 $\pm$ 2.7 & 100
\\
CBM     &
71.3 $\pm$ 0.1 & 100 & 55.4 $\pm$ 3.0 & 52.0  $\pm$ 10.2 &
63.3 $\pm$  1.4 & 100
\\
CEM     &
73.2 $\pm$ 0.3 & 100 &
55.7 $\pm$ 3.1 & 52.0  $\pm$ 10.2
&
64.8  $\pm$ 2.0 & 100
\\
CGM     &
73.3  $\pm$  0.4 & 100 &
60.0 $\pm$ 2.9 & 60.0 $\pm$ 8.2
&
62.3 $\pm$ 1.9 & 100
\\
\bottomrule
\end{tabular}
\label{tab:chexpert_plus}
\end{table}

\section{Algorithmic details and analysis}

\subsection{Algorithm}\label{app:algo}
Algorithm~\ref{alg:fcms} summarizes the end-to-end F-CMs training loop under statistical heterogeneity and temporal non-stationarity. At each round, the server samples available clients, aggregates their supervised variables (and optional local structure) into an updated shared variable set and DAG, and then adapts the shared CM architecture modularly. Clients perform module-specific local updates by freezing unsupervised components, and the server aggregates updates module-wise, averaging each module only over clients that contributed to it.

\begin{algorithm}
\caption{\textsc{F-CMs}: Federated Concept-based Models.}
\small
\label{alg:fcms}
\KwIn{
    Initial shared parameters $\mathbf{w}_0=(\theta_0,\phi_0,\psi_0)$, variable set $\mathcal{V}_0$, shared DAG $\mathcal{G}_0$,
    number of rounds $T$
}
\KwOut{Final shared model parameters $\mathbf{w}_T$ and shared structure $(\mathcal{V}_T,\mathcal{G}_T)$}

Initialize cache $\mathcal{H}_0$ with the latest available client graphs/adjacency matrices\;

\For{$t = 0, 1, \dots, T-1$}{

    $\mathcal{K}_t \leftarrow$ \textsc{ClientSampling}($\mathcal{K}(t)$)\;

    \tcp{Client-side: local graph proposals}
    \For{\textbf{each client} $k \in \mathcal{K}_t$ \textbf{in parallel}}{
        \If{client $k$ is new or has updated its local structure}{
            Send weighted adjacency matrix $A^{(k)}$ to server\;
        }
        Send supervised variables $\mathcal{V}^{(k)}$ to server\;
    }

    \tcp{Server-side: shared structure update}
    Update cache $\mathcal{H}_{t+1}$ with newly received $A^{(k)}$\;

    $\mathcal{V}_{t+1} \leftarrow \mathcal{V}_{t} \cup \bigcup_{k \in \mathcal{K}_{t}} \mathcal{V}^{(k)}$\;

    $\mathcal{G}_{t+1} \leftarrow 
    \textsc{AggregateDAG}\big(\mathcal{H}_{t+1}, \mathcal{V}_{t+1}, \{\alpha_k\}\big)$\;

    $\tilde{\mathbf{w}}_t \leftarrow 
    \textsc{AdaptModel}(\mathbf{w}_t, \mathcal{V}_{t+1}, \mathcal{G}_{t+1})$\;

    Broadcast $(\tilde{\mathbf{w}}_t,\mathcal{G}_{t+1},\mathcal{V}_{t+1})$ to all clients in $\mathcal{K}_t$\;

    \tcp{Client-side: module-specific training}
    \For{\textbf{each client} $k \in \mathcal{K}_t$ \textbf{in parallel}}{

        $\tilde{\mathbf{w}}^{(k)}_t \leftarrow 
        \textsc{FreezeUnsup}\big(\tilde{\mathbf{w}}_t, \mathcal{G}_{t+1}, \mathcal{V}^{(k)}\big)$\;

        $\Delta\tilde{\mathbf{w}}^{(k)}_t \leftarrow 
        \textsc{LocalUpdate}\big(\tilde{\mathbf{w}}^{(k)}_t, \mathcal{G}_{t+1}, \mathcal{V}^{(k)}, \mathcal{D}^{(k)}\big)$\;
    
        Send $\Delta\tilde{\mathbf{w}}^{(k)}_t$ to server\;
    }

    \tcp{Server-side: module-wise aggregation}
    $\mathbf{w}_{t+1} \leftarrow 
    \textsc{ModuleWiseFedAvg}\big(
    \tilde{\mathbf{w}}_t,
    \{\Delta\tilde{\mathbf{w}}^{(k)}_t\}_{k\in\mathcal{K}_t},
    \{\mathcal{V}^{(k)}\}_{k\in\mathcal{K}_t}
    \big)$\;
}
\end{algorithm}

\subsection{Theoretical analysis}\label{app:convergence}
We consider a static setting with a fixed set of clients \(\mathcal K\) and fixed
supervision patterns, i.e., the variables supervised by each client remain fixed throughout
training. The model architecture is therefore fixed and contains one module for each
variable supervised by at least one client in \(\mathcal K\). Each client
\(k\in\mathcal K\) has a local objective \(F^{(k)}\).

\paragraph{Partial supervision.} Since clients may supervise only a subset of variables, they update only the block of parameters
associated with the modules that predict the corresponding variables. We encode these locally updated
coordinates through a diagonal mask \(P^{(k)}\), whose entries are one for parameters
updated by client \(k\) and zero otherwise. For example, if the model contains modules for predicting
\((C_1,C_2,Y)\) and client \(k\) has supervision only for \(C_1\) and \(Y\), then
\(P^{(k)}\) selects the parameters of the modules for \(C_1\) and \(Y\) and masks out the parameters of the
\(C_2\)-module. In our analysis, these supervision patterns, and hence the masks
\(P^{(k)}\), remain fixed throughout training.

\paragraph{Module-wise server aggregation.} The server aggregates parameters module-wise. For a module \(j\), let
\(\mathcal K^j\subseteq\mathcal K\) be the clients that update that module. The server
averages the parameters of module \(j\) only across clients in \(\mathcal K^j\), using
weights
\[
    \beta_j^{(k)}
    =
    \frac{n^{(k)}}{\sum_{r\in\mathcal K^j} n^{(r)}} ,
    \qquad k\in\mathcal K^j .
\]
To model this, we first define for each client \(k\), a diagonal module-wise aggregation operator \(R^{(k)}\) that stores the aggregation weights for
client \(k\): on the coordinates of module \(j\), its diagonal entries are
\(\beta_j^{(k)}\) if client \(k\) updates module \(j\), and zero otherwise. Continuing the example above, if client \(k\) supervises only \(C_1\) and \(Y\), then
\(R^{(k)}\) places the weights \(\beta_{C_1}^{(k)}\) and \(\beta_Y^{(k)}\) on the
corresponding parameter blocks, i.e., parameters of the modules for predicting \(C_1\) and
\(Y\), and zero weight on the \(C_2\)-parameter block. Thus, \(P^{(k)}\) tells which modules
client \(k\) updates locally, while \(R^{(k)}\) tells how much those local updates count in
the server aggregation.


\paragraph{Local optimization.} For each client $k$, let
\[
g^{(k)}(\mathbf{w};\xi^{(k)})
\]
denote a stochastic gradient estimator of the local objective $F^{(k)}$ at parameter vector
$\mathbf{w}$, computed on a random sample or minibatch $\xi^{(k)}$.

At communication round $t$, the server broadcasts $\mathbf{w}_t$ and each client initializes
\[
\mathbf{w}^{(k)}_{t,0} = \mathbf{w}_t .
\]
Each client then performs $E$ local stochastic-gradient steps. At local step $e$, client $k$ draws a
sample or minibatch $\xi^{(k)}_{t,e}$ and computes
\[
g^{(k)}_{t,e}
:=
g^{(k)}(\mathbf{w}^{(k)}_{t,e};\xi^{(k)}_{t,e}) .
\]
The masked local update is
\[
\mathbf{w}^{(k)}_{t,e+1}
=
\mathbf{w}^{(k)}_{t,e}
-
\eta P^{(k)}g^{(k)}_{t,e},
\qquad e=0,\ldots,E-1 .
\]

The server aggregates the final local parameters module-wise. To write this update compactly, we unroll the local steps. For any module $j$ updated by client $k$, we have\[\mathbf{w}^{(k),j}_{t,E}=\mathbf{w}^{j}_t-\eta\sum_{e=0}^{E-1}g^{(k),j}_{t,e}.\]Therefore, the module-wise aggregation rule gives\[\mathbf{w}^{j}_{t+1}=\sum_{k\in \mathcal{K}^j}\beta_j^{(k)}\mathbf{w}^{(k),j}_{t,E}=\mathbf{w}^{j}_t-\eta\sum_{e=0}^{E-1}\sum_{k\in \mathcal{K}^j}\beta_j^{(k)}g^{(k),j}_{t,e}.\]In vector form, this can be written as\[\mathbf{w}_{t+1}=\mathbf{w}_t-\eta\sum_{e=0}^{E-1}\sum_{k\in\mathcal{K}}R^{(k)}g^{(k)}_{t,e}.\]

\paragraph{Assumptions.}
We make the following assumptions.

\begin{enumerate}[label=(S\arabic*),leftmargin=1.8em]

    \item \textbf{L-smoothness}
    Each local objective \(F^{(k)}\) is differentiable and \(L\)-smooth, i.e.,
    \[
        \|\nabla F^{(k)}(\mathbf{u})-\nabla F^{(k)}(\mathbf{v})\|
        \le
        L\|\mathbf{u}-\mathbf{v}\|,
        \qquad
        \forall \mathbf{u},\mathbf{v},\ \forall k\in\mathcal K .
    \]
    We also assume that there exists a differentiable, \(L\)-smooth, and lower-bounded
    effective global objective \(F\), with \(F(w)\ge F_{\inf}\).

\item \textbf{Unbiased stochastic gradients with bounded variance.}
   For each client \(k\), the stochastic gradient estimator is unbiased on the locally
    updated coordinates:
    \[
        \mathbb E_{\xi^{(k)}}\!\left[
        P^{(k)}g^{(k)}(\mathbf{w};\xi^{(k)})
        \right]
        =
        P^{(k)}\nabla F^{(k)}(\mathbf{w}).
    \]
    Moreover, its variance on those coordinates is uniformly bounded: there exists
    \(\sigma^2\ge 0\) such that
    \[
        \mathbb E_{\xi^{(k)}}\!\left[
        \left\|
        P^{(k)}
        \bigl(
        g^{(k)}(\mathbf{w};\xi^{(k)})-\nabla F^{(k)}(\mathbf{w})
        \bigr)
        \right\|^2
        \right]
        \le
        \sigma^2
    \]
    for every client \(k\) and every \(\mathbf{w}\). Conditioned on the current local iterates, the
    stochastic gradients are independent across clients and local steps.

\item \textbf{Bounded module-wise alignment mismatch.}
    We assume that the module-wise aggregated federation direction is close to the gradient
    of the global objective. Specifically, there exists \(\zeta^2\ge 0\) such
    that, for every \(\mathbf{w}\),
    \[
        \left\|
        \sum_{k\in\mathcal K}
        R^{(k)}\nabla F^{(k)}(\mathbf{w})
        -
        \nabla F(\mathbf{w})
        \right\|^2
        \le
        \zeta^2 .
    \]
    The case \(\zeta=0\) corresponds to exact alignment between the module-wise aggregate
    direction and the gradient of \(F\).

\item \textbf{Bounded second moment of masked stochastic gradients.}
    There exists \(G^2\ge 0\) such that
    \[
        \mathbb E\!\left[
        \|P^{(k)}g^{(k)}(\mathbf{w};\xi^{(k)})\|^2
        \right]
        \le
        G^2
    \]
    for every client \(k\) and every \(\mathbf{w}\). This assumption controls the drift caused by
    multiple local steps. Since \(R^{(k)}\) is supported on the same coordinates as
    \(P^{(k)}\) and has entries bounded by one, it also controls the corresponding
    module-wise aggregated updates.

\end{enumerate}

Notice moreover that, by construction, for every module \(j\), there is at least one
client that supervises it, i.e., \(\mathcal K^j\neq\emptyset\), and the aggregation weights
satisfy \(\sum_{k\in\mathcal K^j}\beta_j^{(k)}=1\). 

We next bound the discrepancy between the actual aggregated local-update direction and the ideal global direction.

\begin{lemma}[Deviation of the aggregated local direction]
\label{lem:static_delta}
Define
\[
\bar g_t
\coloneqq
\frac{1}{E}\sum_{e=0}^{E-1}\sum_{k\in\mathcal K} R^{(k)} g_{t,e}^{(k)}
\]
and
\[
\delta_t \coloneqq \bar g_t - \nabla F(\mathbf{w}_t).
\]
Under Assumptions (S1), (S2), (S3), and (S4), for every round $t$,
\[
    \mathbb E\|\delta_t\|^2
    \le
    C
    \left(
    \frac{\sigma^2}{E}
    +
    L^2\eta^2(E-1)^2G^2
    +
    \zeta^2
    \right),
\]
where \(C>0\) depends only on the fixed client set.
\end{lemma}

\begin{proof}
We decompose $\delta_t$ into three terms:
\[
\delta_t = A_t + B_t + C_t,
\]
where
\begin{align*}
A_t
&\coloneqq
\frac{1}{E}\sum_{e=0}^{E-1}\sum_{k\in\mathcal K} R^{(k)}
\big(g_{t,e}^{(k)}-\nabla F^{(k)}(\mathbf{w}_{t,e}^{(k)})\big),\\
B_t
&\coloneqq
\frac{1}{E}\sum_{e=0}^{E-1}\sum_{k\in\mathcal K} R^{(k)}
\big(\nabla F^{(k)}(\mathbf{w}_{t,e}^{(k)})-\nabla F^{(k)}(\mathbf{w}_t)\big),\\
C_t
&\coloneqq
\sum_{k\in\mathcal K} R^{(k)} \nabla F^{(k)}(\mathbf{w}_t)
-
\nabla F(\mathbf{w}_t).
\end{align*}
Applying Jensen's inequality, we have
\[
\mathbb E\!\left[\|\delta_t\|^2\right]
\le
3\mathbb E\!\left[\|A_t\|^2\right]
+
3\mathbb E\!\left[\|B_t\|^2\right]
+
3\mathbb E\!\left[\|C_t\|^2\right].
\]

We bound the three terms separately.

\paragraph{\textbf{Bound on $A_t$}.}
Let us define
\[
\varepsilon_{t,e}^{(k)}
\coloneqq
g_{t,e}^{(k)}-\nabla F^{(k)}(\mathbf w_{t,e}^{(k)}).
\]
For each round \(t\) and local step \(e\), let \(\mathcal F_{t,e}\) denote the
sigma-field containing all information available before the minibatches at step \(e\) are
sampled. In particular, given \(\mathcal F_{t,e}\), the local iterate
\(\mathbf w_{t,e}^{(k)}\) is fixed, since it has been computed from 
the minibatches used in the previous local steps. Thus, for client \(k\), the fresh randomness
in \(g^{(k)}_{t,e}\) comes from the minibatch \(\xi_{t,e}^{(k)}\). Therefore, by Assumption~(S2), we have
\[
    \mathbb E\!\left[
    P^{(k)}g^{(k)}_{t,e}
    \mid
    \mathcal F_{t,e}
    \right]
    =
    P^{(k)}\nabla F^{(k)}(\mathbf{w}^{(k)}_{t,e}).
\]
Equivalently,
\[
    \mathbb E\!\left[
    P^{(k)}\varepsilon^{(k)}_{t,e}
    \mid
    \mathcal F_{t,e}
    \right]
    =
    0.
\]

Since \(R^{(k)}P^{(k)}=R^{(k)}\) and \(R^{(k)}\) is fixed , this also implies
\begin{equation}
\label{eq:mte}
    \mathbb E\!\left[
    R^{(k)}\varepsilon^{(k)}_{t,e}
    \mid
    \mathcal F_{t,e}
    \right]
    =
    R^{(k)}
    \mathbb E\!\left[
    P^{(k)}\varepsilon^{(k)}_{t,e}
    \mid
    \mathcal F_{t,e}
    \right]
    =
    0.
\end{equation}

Moreover, by Assumption~(S2),
\[
    \mathbb E\!\left[
    \|P^{(k)}\varepsilon^{(k)}_{t,e}\|^2
    \mid
    \mathcal F_{t,e}
    \right]
    \le
    \sigma^2.
\]


For convenience, define
\[
M_{t,e}
\coloneqq
\sum_{k\in\mathcal K} R^{(k)} \varepsilon_{t,e}^{(k)}.
\]
Then
\[
A_t = \frac{1}{E}\sum_{e=0}^{E-1} M_{t,e}.
\]

Hence
\[
\|A_t\|^2
=
\frac{1}{E^2}
\left\|
\sum_{e=0}^{E-1} M_{t,e}
\right\|^2
=
\frac{1}{E^2}
\sum_{e=0}^{E-1}\|M_{t,e}\|^2
+
\frac{2}{E^2}\sum_{0\le e<e'\le E-1}\langle M_{t,e},M_{t,e'}\rangle .
\]
Taking expectations gives
\[
\mathbb E[\|A_t\|^2]
=
\frac{1}{E^2}
\sum_{e=0}^{E-1}\mathbb E[\|M_{t,e}\|^2]
+
\frac{2}{E^2}\sum_{0\le e<e'\le E-1}\mathbb E[\langle M_{t,e},M_{t,e'}\rangle].
\]

We now show that all cross-step terms vanish. Fix $e<e'$. The quantity \(M_{t,e}\)
is already determined before the minibatches at step \(e'\) are sampled, while, by Eq. \ref{eq:mte},
\[
    \mathbb E[M_{t,e'}\mid \mathcal F_{t,e'}]=0.
\]
Given the properties of conditional expectation, we have
\begin{align*}
\mathbb E[\langle M_{t,e},M_{t,e'}\rangle]
&=
\mathbb E\!\left[
\mathbb E\!\left[\langle M_{t,e},M_{t,e'}\rangle \mid \mathcal F_{t,e'}\right]
\right] \\
&=
\mathbb E\!\left[
\left\langle
M_{t,e},
\mathbb E[M_{t,e'}\mid \mathcal F_{t,e'}]
\right\rangle
\right] \\
&=0.
\end{align*}
Therefore,
\[
\mathbb E[\|A_t\|^2]
=
\frac{1}{E^2}
\sum_{e=0}^{E-1}\mathbb E[\|M_{t,e}\|^2].
\]

We next expand the square within a fixed local step $e$:
\[
\|M_{t,e}\|^2
=
\left\|
\sum_{k\in\mathcal K}  R^{(k)} \varepsilon_{t,e}^{(k)}
\right\|^2
=
\sum_{k\in\mathcal K}\|R^{(k)}\varepsilon_{t,e}^{(k)}\|^2
+
\sum_{k\neq r} 
\left\langle
R^{(k)}\varepsilon_{t,e}^{(k)},
R^{(r)}\varepsilon_{t,e}^{(r)}
\right\rangle .
\]
Taking conditional expectation with respect to \(\mathcal F_{t,e}\), the cross-client
terms vanish. Indeed, for \(k\neq r\), let
\[
    X_k := R^{(k)}\varepsilon^{(k)}_{t,e},
    \qquad
    X_r := R^{(r)}\varepsilon^{(r)}_{t,e}.
\]
Conditioned on \(\mathcal F_{t,e}\), the minibatch noises of different clients are
independent. Moreover, as shown above,
\[
    \mathbb E[X_k\mid\mathcal F_{t,e}]
    =
    \mathbb E[X_r\mid\mathcal F_{t,e}]
    =
    0.
\]
Therefore,
\[
\begin{aligned}
\mathbb E\!\left[
\left\langle
R^{(k)}\varepsilon^{(k)}_{t,e},
R^{(r)}\varepsilon^{(r)}_{t,e}
\right\rangle
\mid
\mathcal F_{t,e}
\right]
&=
\left\langle
\mathbb E[X_k\mid\mathcal F_{t,e}],
\mathbb E[X_r\mid\mathcal F_{t,e}]
\right\rangle \\
&=0.
\end{aligned}
\]

Hence,
\[
\mathbb{E}
\left[
\|M_{t,e}\|^2
\mid
\mathcal{F}_{t,e}
\right]
=
\sum_{k\in\mathcal{K}}
\mathbb{E}
\left[
\|R^{(k)}\varepsilon_{t,e}^{(k)}\|^2
\mid
\mathcal{F}_{t,e}
\right].
\]

Since \(R^{(k)}P^{(k)}=R^{(k)}\) and the diagonal entries of \(R^{(k)}\) lie in
\([0,1]\), we have
\[
    \|R^{(k)}\varepsilon^{(k)}_{t,e}\|^2
    =
    \|R^{(k)}P^{(k)}\varepsilon^{(k)}_{t,e}\|^2
    \le
    \|P^{(k)}\varepsilon^{(k)}_{t,e}\|^2 .
\]
Using Assumption~(S2), we obtain
\[
\begin{aligned}
    \mathbb E\!\left[
    \|M_{t,e}\|^2
    \mid
    \mathcal F_{t,e}
    \right]
    &=
    \sum_{k\in\mathcal K}
    \mathbb E\!\left[
    \|R^{(k)}\varepsilon^{(k)}_{t,e}\|^2
    \mid
    \mathcal F_{t,e}
    \right] \\
    &\le
    \sum_{k\in\mathcal K}
    \mathbb E\!\left[
    \|P^{(k)}\varepsilon^{(k)}_{t,e}\|^2
    \mid
    \mathcal F_{t,e}
    \right] \\
    &\le
    |\mathcal K|\sigma^2 .
\end{aligned}
\]
Since the client set is fixed, \(|\mathcal K|\) is a constant
independent of \(t,T,\eta\), and \(E\). Defining \(C_R:=|\mathcal K|\), and taking
expectations on both sides, we get
\[
    \mathbb E\|M_{t,e}\|^2
    \le
    C_R\sigma^2 .
\]
Substituting this into the previous bound gives
\[
    \mathbb E\|A_t\|^2
    \le
    \frac{1}{E^2}
    \sum_{e=0}^{E-1}
    C_R\sigma^2
    =
    C_R\frac{\sigma^2}{E}.
\]
\paragraph{\textbf{Bound on $B_t$}.}
By Assumption (S1) and since each $R^{(k)}$ is diagonal with entries in $[0,1]$,
\[
\|B_t\|
\le
\frac{L}{E}\sum_{e=0}^{E-1}\sum_{k\in\mathcal K} 
\|\mathbf{w}_{t,e}^{(k)}-\mathbf{w}_t\|.
\]
Applying Jensen's inequality gives
\[
\|B_t\|^2
\le
\frac{L^2|\mathcal{K}|}{E}\sum_{e=0}^{E-1}\sum_{k\in\mathcal K}
\|\mathbf{w}_{t,e}^{(k)}-\mathbf{w}_t\|^2.
\]
Now,
\[
\mathbf{w}_{t,e}^{(k)}-\mathbf{w}_t
=
-\eta \sum_{s=0}^{e-1} P^{(k)} g_{t,s}^{(k)},
\]
so, again by Jensen,
\[
\|\mathbf{w}_{t,e}^{(k)}-\mathbf{w}_t\|^2
\le
\eta^2 e \sum_{s=0}^{e-1}\|P^{(k)} g_{t,s}^{(k)}\|^2.
\]
Taking expectations and using Assumption (S4),
\[
\mathbb E\!\left[\|\mathbf{w}_{t,e}^{(k)}-\mathbf{w}_t\|^2\right]
\le
\eta^2 e^2 G^2.
\]
Therefore,
\begin{align*}
&\mathbb E\!\left[\|B_t\|^2\right]
\le
\frac{L^2|\mathcal{K}|}{E}\sum_{e=0}^{E-1}\sum_{k\in\mathcal K}  \eta^2 e^2 G^2 
= \frac{L^2|\mathcal{K}|^2}{E} \eta^2 G^2 \sum_{e=0}^{E-1}e^2
\\ & \le \frac{L^2|\mathcal{K}|^2}{E} \eta^2 G^2 \sum_{e=0}^{E-1}(E-1)^2=
L^2|\mathcal{K}|^2\eta^2G^2(E-1)^2.
\end{align*}
Since $\mathcal{K}$ is fixed in the static regime, $|\mathcal{K}|^2$ is a constant. Renaming
constants, we obtain
\[
\mathbb{E}\!\left[\|B_t\|^2\right]
\leq
C_B L^2\eta^2G^2(E-1)^2,
\]
for a constant $C_B>0$ independent of $t,T,\eta$, and $E$.

\paragraph{\textbf{Bound on $C_t$}.}
By Assumption (S3), we have
\[
\|C_t\|^2
=
\left\|
\sum_{k\in\mathcal{K}}
R^{(k)}\nabla F^{(k)}(\mathbf{w}_t)
-
\nabla F(\mathbf{w}_t)
\right\|^2
\leq
\zeta^2 .
\]
Therefore,
\[
\mathbb{E}\!\left[\|C_t\|^2\right]
\leq
\zeta^2 .
\]
Combining the three bounds proves the claim.
\end{proof}

\begin{theorem}[Convergence in the static case]
\label{thm:static_convergence}
Suppose Assumptions~(S1)--(S4) hold, and let
\(\{\mathbf w_t\}_{t=0}^{T-1}\) be the sequence generated by the static training
procedure above. Let
\[
    \bar\eta := \eta E
\]
and assume \(\bar\eta \le 1/(8L)\). Then there exist constants
\(C_1,C_2,C_3>0\), independent of \(T,\eta\), and \(E\), such that
\[
\frac1T
\sum_{t=0}^{T-1}
\mathbb E
\left[
\|\nabla F(\mathbf w_t)\|^2
\right]
\le
\frac{2(F(\mathbf w_0)-F_{\inf})}{\bar\eta T}
+
C_1 L\bar\eta\frac{\sigma^2}{E}
+
C_2 L^2\eta^2(E-1)^2G^2
+
C_3\zeta^2,
\]
where \(F_{\inf}\) is a lower bound of \(F\).

Equivalently, since \(\bar\eta=\eta E\),
\[
\frac1T
\sum_{t=0}^{T-1}
\mathbb E
\left[
\|\nabla F(\mathbf w_t)\|^2
\right]
\le
\frac{2(F(\mathbf w_0)-F_{\inf})}{\eta E T}
+
C_1 L\eta\sigma^2
+
C_2 L^2\eta^2(E-1)^2G^2
+
C_3\zeta^2.
\]
In particular, if \(E\) is fixed and \(\eta=\Theta((E\sqrt T)^{-1})\), then
\[
\frac1T
\sum_{t=0}^{T-1}
\mathbb E
\left[
\|\nabla F(\mathbf w_t)\|^2
\right]
=
O(T^{-1/2}+\zeta^2).
\]
In the exact-alignment case \(\zeta=0\), the standard nonconvex
\(O(T^{-1/2})\) stationarity rate is recovered.
\end{theorem}

\begin{proof}
By the module-wise aggregation rule introduced above, the global update can be written as
\[
    \mathbf w_{t+1}
    =
    \mathbf w_t
    -
    \eta
    \sum_{e=0}^{E-1}
    \sum_{k\in\mathcal K}
    R^{(k)}g^{(k)}_{t,e}.
\]
Define, as in Lemma~\ref{lem:static_delta},
\[
    \bar g_t
    :=
    \frac1E
    \sum_{e=0}^{E-1}
    \sum_{k\in\mathcal K}
    R^{(k)}g^{(k)}_{t,e}.
\]
Then
\[
    \mathbf w_{t+1}
    =
    \mathbf w_t-\bar\eta \bar g_t,
    \qquad
    \bar\eta=\eta E.
\]

Using the decomposition from Lemma~\ref{lem:static_delta}, we can write
\[
    \bar g_t
    =
    \nabla F(\mathbf w_t)
    +
    A_t+B_t+C_t.
\]
Therefore,
\[
    \mathbf w_{t+1}
    =
    \mathbf w_t
    -
    \bar\eta
    \bigl(
    \nabla F(\mathbf w_t)+A_t+B_t+C_t
    \bigr).
\]

Let
\[
    G_t := \nabla F(\mathbf w_t).
\]
By \(L\)-smoothness of \(F\), the descent lemma gives
\[
\begin{aligned}
F(\mathbf w_{t+1})
&\le
F(\mathbf w_t)
+
\left\langle
G_t,\mathbf w_{t+1}-\mathbf w_t
\right\rangle
+
\frac{L}{2}
\|\mathbf w_{t+1}-\mathbf w_t\|^2 \\
&=
F(\mathbf w_t)
-
\bar\eta
\left\langle
G_t,
G_t+A_t+B_t+C_t
\right\rangle
+
\frac{L\bar\eta^2}{2}
\|G_t+A_t+B_t+C_t\|^2 .
\end{aligned}
\]

We now take expectations. First, we have
\[
    \mathbb E\left[\langle G_t,A_t\rangle\right]=0,
\]
because \(G_t=\nabla F(\mathbf w_t)\) is fixed at the beginning of round \(t\), while
\(A_t\) is a sum of conditionally zero-mean minibatch-noise terms, as shown in
Lemma~\ref{lem:static_delta}. Hence,
\[
\begin{aligned}
\mathbb E[F(\mathbf w_{t+1})]
&\le
\mathbb E[F(\mathbf w_t)]
-
\bar\eta
\mathbb E\|G_t\|^2
-
\bar\eta
\mathbb E
\left[
\left\langle
G_t,B_t+C_t
\right\rangle
\right] \\
&\quad
+
\frac{L\bar\eta^2}{2}
\mathbb E
\|G_t+A_t+B_t+C_t\|^2 .
\end{aligned}
\]

We bound the remaining inner product using Cauchy--Schwarz and Young's inequality:
\[
    -
    \left\langle
    G_t,B_t+C_t
    \right\rangle
    \le
    \frac14\|G_t\|^2
    +
    \|B_t+C_t\|^2.
\]
Moreover,
\[
    \|B_t+C_t\|^2
    \le
    2\|B_t\|^2+2\|C_t\|^2.
\]
and
\[
    \|G_t+A_t+B_t+C_t\|^2
    \le
    4\left(
    \|G_t\|^2+\|A_t\|^2+\|B_t\|^2+\|C_t\|^2
    \right).
\]
Substituting these inequalities gives
\[
\begin{aligned}
\mathbb E[F(\mathbf w_{t+1})]
&\le
\mathbb E[F(\mathbf w_t)]
-
\bar\eta
\mathbb E\|G_t\|^2
+
\frac{\bar\eta}{4}
\mathbb E\|G_t\|^2 \\
&\quad
+
2\bar\eta
\mathbb E\|B_t\|^2
+
2\bar\eta
\mathbb E\|C_t\|^2 \\
&\quad
+
2L\bar\eta^2
\mathbb E\|G_t\|^2
+
2L\bar\eta^2
\mathbb E\|A_t\|^2
+
2L\bar\eta^2
\mathbb E\|B_t\|^2
+
2L\bar\eta^2
\mathbb E\|C_t\|^2 .
\end{aligned}
\]

Since \(\bar\eta\le 1/(8L)\), we have \(2L\bar\eta\le 1/4\). Therefore,
\[
-\bar\eta+\frac{\bar\eta}{4}+2L\bar\eta^2
\le
-\frac{\bar\eta}{2}.
\]
Thus,
\[
\begin{aligned}
\mathbb E[F(\mathbf w_{t+1})]
&\le
\mathbb E[F(\mathbf w_t)]
-
\frac{\bar\eta}{2}
\mathbb E\|G_t\|^2 \\
&\quad
+
2L\bar\eta^2
\mathbb E\|A_t\|^2
+
(2\bar\eta+2L\bar\eta^2)
\mathbb E\|B_t\|^2
+
(2\bar\eta+2L\bar\eta^2)
\mathbb E\|C_t\|^2 .
\end{aligned}
\]

From Lemma~\ref{lem:static_delta} and its proof, the three components satisfy
\[
    \mathbb E\|A_t\|^2
    \le
    C_R\frac{\sigma^2}{E},
\]
\[
    \mathbb E\|B_t\|^2
    \le
    C_B L^2\eta^2(E-1)^2G^2,
\]
\[
    \mathbb E\|C_t\|^2
    \le
    \zeta^2.
\]
Since \(\bar\eta\le 1/(8L)\), the factors
\(2\bar\eta+2L\bar\eta^2\) are bounded by a constant multiple of \(\bar\eta\). Hence, for
suitable constants \(C'_1,C'_2,C'_3>0\),
\[
\mathbb E[F(\mathbf w_{t+1})]
\le
\mathbb E[F(\mathbf w_t)]
-
\frac{\bar\eta}{2}
\mathbb E\|\nabla F(\mathbf w_t)\|^2
+
C'_1 L\bar\eta^2\frac{\sigma^2}{E}
+
C'_2 \bar\eta L^2\eta^2(E-1)^2G^2
+
C'_3 \bar\eta\zeta^2 .
\]

Rearranging terms gives
\[
\frac{\bar\eta}{2}
\mathbb E\|\nabla F(\mathbf w_t)\|^2
\le
\mathbb E[F(\mathbf w_t)]
-
\mathbb E[F(\mathbf w_{t+1})]
+
C'_1 L\bar\eta^2\frac{\sigma^2}{E}
+
C'_2 \bar\eta L^2\eta^2(E-1)^2G^2
+
C'_3 \bar\eta\zeta^2 .
\]

Summing over \(t=0,\ldots,T-1\), we obtain
\[
\begin{aligned}
\frac{\bar\eta}{2}
\sum_{t=0}^{T-1}
\mathbb E\|\nabla F(\mathbf w_t)\|^2
&\le
\sum_{t=0}^{T-1}
\left(
\mathbb E[F(\mathbf w_t)]
-
\mathbb E[F(\mathbf w_{t+1})]
\right) \\
&\quad
+
T C'_1 L\bar\eta^2\frac{\sigma^2}{E}
+
T C'_2 \bar\eta L^2\eta^2(E-1)^2G^2
+
T C'_3 \bar\eta\zeta^2 .
\end{aligned}
\]
The first sum is telescopic. Since \(F(\mathbf w_0)\) is deterministic and
\(F(\mathbf w_T)\ge F_{\inf}\), we get
\[
\frac{\bar\eta}{2}
\sum_{t=0}^{T-1}
\mathbb E\|\nabla F(\mathbf w_t)\|^2
\le
F(\mathbf w_0)-F_{\inf}
+
T C'_1 L\bar\eta^2\frac{\sigma^2}{E}
+
T C'_2 \bar\eta L^2\eta^2(E-1)^2G^2
+
T C'_3 \bar\eta\zeta^2 .
\]
Dividing by \(\bar\eta T/2\) yields
\[
\frac1T
\sum_{t=0}^{T-1}
\mathbb E\|\nabla F(\mathbf w_t)\|^2
\le
\frac{2(F(\mathbf w_0)-F_{\inf})}{\bar\eta T}
+
2C'_1 L\bar\eta\frac{\sigma^2}{E}
+
2C'_2 L^2\eta^2(E-1)^2G^2
+
2C'_3\zeta^2.
\]
Renaming constants proves the first claim.

Using \(\bar\eta=\eta E\), this can equivalently be written as
\[
\frac1T
\sum_{t=0}^{T-1}
\mathbb E\|\nabla F(\mathbf w_t)\|^2
\le
\frac{2(F(\mathbf w_0)-F_{\inf})}{\eta E T}
+
C_1 L\eta\sigma^2
+
C_2 L^2\eta^2(E-1)^2G^2
+
C_3\zeta^2.
\]
If \(E\) is fixed and \(\eta=\Theta((E\sqrt T)^{-1})\), then the first two terms are
\(O(T^{-1/2})\), the third term is \(O(T^{-1})\), and the fourth term remains
proportional to \(\zeta^2\). Therefore,
\[
\frac1T
\sum_{t=0}^{T-1}
\mathbb E\|\nabla F(\mathbf w_t)\|^2
=
O(T^{-1/2}+\zeta^2).
\]
In particular, when \(\zeta=0\), the standard \(O(T^{-1/2})\) nonconvex stationarity rate
is recovered.
\end{proof}


\paragraph{Empirical assessment of the alignment mismatch.}
To characterize the effect of module-wise aggregation, we compare
its update direction with a sample-weighted average of the local
gradients, evaluated at the same broadcast model $w_t$. For this comparison, we define the sample-weighted pooled objective
\[
F_{\mathrm{pool}}(w)
:=
\sum_{k \in \mathcal{K}} p^{(k)} F^{(k)}(w),
\qquad
p^{(k)}
:=
\frac{n^{(k)}}{\sum_{r \in K} n^{(r)}}.
\]
At the beginning of each communication round, every
client computes its gradient at the broadcast model $w_t$ for this diagnostic. We then construct
the module-wise aggregated direction
\[
d_t
=
\sum_{k \in \mathcal{K}} R^{(k)} \nabla F^{(k)}(w_t)
\]
and measure its empirical mismatch with the sample-weighted
average of the local gradients, given by
$\nabla F_{\mathrm{pool}}(w_t)$:
\[
\widehat{\zeta}_t^2
=
\left\|
d_t - \nabla F_{\mathrm{pool}}(w_t)
\right\|^2.
\]
We evaluate $\widehat{\zeta}_t^2$ over the first 30 communication rounds,
using three random seeds and one local epoch per round. We additionally
measure the cosine similarity between $d_t$ and
$\nabla F_{\mathrm{pool}}(w_t)$ to distinguish directional disagreement
from differences in gradient magnitude.

\begin{table}[t]
    \centering
    \caption{
        \textbf{Empirical alignment mismatch.}
        Empirical mismatch $\widehat{\zeta}_t^2$ between the module-wise
        aggregated direction and the gradient of the pooled objective over
        the first 30 communication rounds. The round attaining the minimum
        and maximum mismatch is reported in parentheses. Results are
        aggregated over three random seeds.
    }
    \label{tab:alignment-mismatch}
    \setlength{\tabcolsep}{5pt}
    \begin{tabular}{lcccc}
        \toprule
        Dataset
        & Mean $\widehat{\zeta}_t^2$
        & Min. $\widehat{\zeta}_t^2$ (round)
        & Max. $\widehat{\zeta}_t^2$ (round)
        & Mean cosine sim. \\
        \midrule
        \textsc{Asia}
        & $0.0194$
        & $0.0186$ (30)
        & $0.0203$ (1)
        & $0.949 \pm 0.013$ \\
        \textsc{Alarm}
        & $0.0288$
        & $0.0287$ (30)
        & $0.0290$ (1)
        & $0.883 \pm 0.007$ \\
        \bottomrule
    \end{tabular}
\end{table}

Table~\ref{tab:alignment-mismatch} shows that the mismatch remains stable
along the evaluated optimization trajectories, ranging from $0.0186$ to
$0.0203$ on \textsc{Asia} and from $0.0287$ to $0.0290$ on \textsc{Alarm}. In both cases,
the maximum occurs at the first round and the minimum at round 30, indicating
a mild decrease over the evaluated rounds. Moreover, mean cosine similarities of
$0.949$ and $0.883$, respectively, show that the module-wise aggregate
remains strongly positively aligned with the pooled gradient.

\paragraph{Extension to dynamic environments with piecewise-stationary changes.}

We consider a piecewise-stationary dynamic process that can be
decomposed into a sequence of static regimes, each corresponding
to the interval between two consecutive change points. Within each regime $s$, we retain the participation model of
Theorem~E.2, with a fixed client set $\mathcal K_s$. Changes in the client set, supervision pattern,
concept graph, or architecture occur only at regime boundaries.

We assume that the sequence of client and structural changes
eventually stabilizes. In particular, given finite sets of
task-relevant concepts and admissible dependencies, we assume
that the aggregated structure eventually stabilizes and that
the total number of regimes remains bounded independently of
the training horizon. At a change point, the active client set, the available
supervision, the concept graph, the induced objective, and the
model architecture may change. The current model is then transferred to the parameter
space of the new optimization problem and used to initialize
the next regime. More formally, let
\[
0=\tau_0<\tau_1<\cdots<\tau_S=T
\]
denote the regime boundaries, with changes occurring at
$\tau_1,\ldots,\tau_{S-1}$. Within each interval
\[
I_s:=[\tau_s,\tau_{s+1}),
\]
all quantities defining the optimization problem remain fixed.
In particular, the model has a fixed parameter space
$\mathcal W_s$ and a differentiable objective
\[
F_s:\mathcal W_s\to\mathbb R.
\]

Let
\[
T_s:=\tau_{s+1}-\tau_s
\]
denote the duration of regime $s$. We assume that the
regime-specific counterparts of Assumptions (S1)--(S4) hold
within each regime, with constants uniformly bounded across
regimes. For notational simplicity, let
$L,\sigma,G,C_1,C_2,C_3$ denote common finite upper bounds
on the corresponding regime-specific quantities. We further assume that every trainable module in regime $s$ is
supervised by at least one client in $\mathcal K_s$; retained modules
with no current supervision are treated as fixed and excluded from
the optimization variables of that regime.

At a transition time $\tau_s$, the parameter space, objective,
and supervision pattern may differ from those of the preceding
regime. We denote by \(w_{\tau_s}^{-}\) the parameters immediately before the transfer and by
\[
w_{\tau_s}^{+}\in\mathcal W_s
\]
the post-transfer parameters used to initialize regime \(s\). For notational uniformity, we set $w_{\tau_0}^{+}:=w_0$ and write
$w_{\tau_s}:=w_{\tau_s}^{+}$ for all $s$.

For each regime, define the expected regime-start objective gap as
\[
\Gamma_s
:=
\mathbb E\!\left[
F_s(w_{\tau_s})-F_{s,\inf}
\right],
\]
where $F_{s,\inf}$ is a lower bound on $F_s$. For $s\geq1$,
$\Gamma_s$ measures the quality of the transferred initialization
under the newly active objective.

Applying the proof of Theorem~E.2 within each regime gives:
\[
\begin{aligned}
&\frac{1}{T_s}
\sum_{t=\tau_s}^{\tau_{s+1}-1}
\mathbb E\!\left[
\left\|\nabla F_s(w_t)\right\|^2
\right]
\\
&\quad\leq
\frac{2\Gamma_s}{\bar\eta T_s}
+
C_1L\bar\eta\frac{\sigma^2}{E}
+
C_2L^2\eta^2(E-1)^2G^2
+
C_3\zeta_s^2,
\end{aligned}
\]
where $\bar\eta:=\eta E\leq 1/(8L)$ and $\zeta_s^2$
denotes the module-wise alignment mismatch in regime $s$.

Multiplying the bound for regime $s$ by $T_s/T$ and summing
over all regimes yields
\[
\begin{aligned}
\mathcal G_T
&:=
\frac{1}{T}
\sum_{s=0}^{S-1}
\sum_{t=\tau_s}^{\tau_{s+1}-1}
\mathbb E\!\left[
\left\|\nabla F_s(w_t)\right\|^2
\right]
\\
&\leq
\frac{2}{\bar\eta T}
\sum_{s=0}^{S-1}\Gamma_s
+
C_1L\bar\eta\frac{\sigma^2}{E}
\\
&\quad+
C_2L^2\eta^2(E-1)^2G^2
+
\frac{C_3}{T}
\sum_{s=0}^{S-1}T_s\zeta_s^2.
\end{aligned}
\]

Here, $\mathcal G_T$ measures average stationarity over the
complete dynamic trajectory, evaluating each iterate with respect
to the objective active at that round.

By assumption, the number of regimes is bounded independently
of $T$, i.e.,
\[
S\leq S_{\max}
\]
for some constant $S_{\max}$. Assume further that
 $\Gamma_0=O(1)$ and that the post-transition initialization
gaps satisfy
\[
\Gamma_s\leq B,
\qquad s\geq1.
\]
It follows that
\[
\sum_{s=0}^{S-1}\Gamma_s
\leq
\Gamma_0+(S-1)B
\leq
\Gamma_0+(S_{\max}-1)B
=
O(1).
\]

With fixed $E$ and the same step-size choice as in
Theorem E.2,
\[
\eta=\Theta\!\left(\frac{1}{E\sqrt T}\right),
\qquad
\bar\eta=\Theta(T^{-1/2}),
\]
the contribution of the regime-start objective gaps is
$O(T^{-1/2})$. If the module-wise alignment mismatch is uniformly bounded
across regimes, i.e.,
\[
\zeta_s^2\leq\zeta^2
\qquad
\text{for all }s,
\]
then, similarly to before, we obtain 
\[
\mathcal G_T
=
O\!\left(T^{-1/2}+\zeta^2\right).
\]

Thus, under the stated assumptions, the piecewise-stationary
dynamic trajectory satisfies the same average-stationarity
guarantee as the static case, up to the module-wise alignment
mismatch. Under exact alignment in every regime, i.e.,
$\zeta_s=0$ for all $s$, we recover
\[
\mathcal G_T=O(T^{-1/2}).
\]

The assumptions above are natural in settings where new institutions, concepts, and dependencies are introduced during a bounded onboarding or adaptation phase, after which the federation stabilizes. In continuously evolving deployments, the same piecewise analysis suggests that the stationarity bound would additionally depend on the frequency and cumulative cost of regime changes. Extending the analysis to persistent non-stationarity, for example through explicit control of these quantities, is an interesting direction for future work.


\end{document}